\documentclass[nohyperref]{article}

\usepackage{mymath,natbib,enumitem}

\usepackage{microtype}
\usepackage{graphicx}
\usepackage{subfigure}
\usepackage{booktabs} 

\usepackage{hyperref}

\usepackage[accepted]{icml2022}

\usepackage{amsmath}
\usepackage{amssymb}
\usepackage{mathtools}
\usepackage{amsthm}

\usepackage[capitalize,noabbrev]{cleveref}

\theoremstyle{plain}
\newtheorem{theorem}{Theorem}[section]
\newtheorem{proposition}[theorem]{Proposition}
\newtheorem{lemma}[theorem]{Lemma}

\theoremstyle{definition}

\newtheorem{assumption}[theorem]{Assumption}
\theoremstyle{remark}

\usepackage[textsize=tiny]{todonotes}

\icmltitlerunning{Last Iterate Risk Bounds of SGD with Decaying Stepsize for Overparameterized Linear Regression}

\begin{document}

\twocolumn[
\icmltitle{Last Iterate Risk Bounds of SGD with Decaying Stepsize for Overparameterized Linear Regression}



\icmlsetsymbol{equal}{*}

\begin{icmlauthorlist}
\icmlauthor{Jingfeng Wu}{equal,jhu}
\icmlauthor{Difan Zou}{equal,ucla}
\icmlauthor{Vladimir Braverman}{jhu}
\icmlauthor{Quanquan Gu}{ucla}
\icmlauthor{Sham M. Kakade}{harvard}
\end{icmlauthorlist}

\icmlaffiliation{jhu}{Department of Computer Science, Johns Hopkins University, Baltimore, MD 21218, USA}
\icmlaffiliation{ucla}{Department of Computer Science, University of California, Los Angeles, CA 90095, USA}
\icmlaffiliation{harvard}{Department of Computer Science and Department of Statistics, Harvard University, Cambridge, MA 02138, USA}

\icmlcorrespondingauthor{Vladimir Braverman}{vova@cs.jhu.edu}
\icmlcorrespondingauthor{Quanquan Gu}{qgu@cs.ucla.edu}
\icmlcorrespondingauthor{Sham M. Kakade}{sham@seas.harvard.edu}

\icmlkeywords{SGD, Overparameterization, Risk bound}

\vskip 0.3in
]



\printAffiliationsAndNotice{\icmlEqualContribution} 

\begin{abstract}
Stochastic gradient descent (SGD) has been shown to generalize well in many deep learning applications. In practice, one often runs SGD with a \emph{geometrically decaying stepsize}, i.e., a constant initial stepsize followed by multiple geometric stepsize decay, and uses the \emph{last iterate} as the output. This kind of SGD is known to be nearly minimax optimal for classical finite-dimensional linear regression problems \citep{ge2019step}. However, a sharp analysis for the last iterate of SGD in the overparameterized setting is still open. In this paper, we provide a \emph{problem-dependent} analysis on the last iterate risk bounds of SGD with decaying stepsize, for (overparameterized) linear regression problems. In particular, for last iterate SGD with (tail) geometrically decaying stepsize, we prove nearly matching upper and lower bounds on the excess risk. Moreover, we provide an excess risk lower bound for last iterate SGD with polynomially decaying stepsize and demonstrate the advantage of geometrically decaying stepsize in an \emph{instance-wise} manner, which complements the minimax rate comparison made in prior works.
\end{abstract}

\allowdisplaybreaks




\section{Introduction}\label{sec:intro}

It is widely observed in practice that modern neural networks trained by \emph{stochastic gradient descent} (SGD) often generalize well \citep{zhang2021understanding}. 
In all the successful applications, two ingredients are crucial: 
(1) an overparameterized model, where the number of parameter excesses the number of training examples \citep{belkin2020two};
and (2) SGD with the \emph{last iterate} as output and with a \emph{decaying stepsize}, e.g., an initially large stepsize, followed by geometrically decaying stepsizes after every certain number of iterations \citep{he2015deep}.
Theoretically, however, it remains largely open to understand the generalization of the \emph{last iterate of SGD} (with a \emph{decaying stepsize}) for learning overparamerized models, even for the arguably simplest setting such as \emph{overparamerized linear regression}.


For linear regression in the classical regime, \citet{ge2019step} showed that last iterate SGD (with geometrically decaying stepsize) can achieve nearly minimax optimal excess risk up to logarithmic factors.
However, the results by \citet{ge2019step} cannot be carried over to the overparameterized setting since their excess risk bounds are \emph{dimension-dependent}, which become vacuous when the problem dimension excesses the sample size. There is a fundamental barrier to extend the statistical minimax rate to the overparameterized setting, as the minimax result concerns the \emph{worst instance} in the problem class, while apparently SGD cannot generalize for certain overparameterized linear regression problem (e.g., when the data distribution has an identity covariance and the model parameter is uniformly distributed). 

For SGD with iterate averaging, a recent work by \citet{zou2021benign} proved a tight problem-dependent excess risk bound for overparameterized linear regression, which can diminish in the overparameterized setting, provided a sufficiently fast decaying spectrum of the data covaraiance matrix.
While \citet{zou2021benign} sharply characterized the generalization of SGD with iterate averaging in the overparameterized setting, their analysis is tailored to the averaged iterate of SGD and is not directly applicable to the last iterate of SGD.


In this paper, in order to explain its success in learning overparameterized models,
we provide a tight analysis for the last iterate of SGD that adapts to both of the least-square problem instance and the algorithm configuration.

\noindent\textbf{Contributions.} Our first main result is a sharp problem-dependent excess risk bound for the \emph{last iterate SGD with tail geometrically decaying stepsize} (see \eqref{eq:geometry-tail-decay-lr}, also Algorithm \ref{alg:SGD}) for linear regression.
The derived bound does not depend on the ambient {dimension} and, instead, depends on the spectrum of the data covariance matrix. 
In particular, the excess risk bound vanishes as long as the data covariance matrix has a fast-decay eigenspectrum, despite of a large ambient {dimension} in the overparameterized setting.
Furthermore, an excess risk lower bound is proved, which shows the upper bound is tight up to absolute constant in terms of variance error, and is nearly tight in terms of bias error.
This result recovers the existing minimax bound in the classical regime  \citep{polyak1992acceleration,bach2013non} ignoring logarithmic factors, and is comparable to the bounds for  SGD with iterate averaging  in the overparameterized setting \citep{zou2021benign}.

Our second main result is a comparison between SGD with (1) tail geometrically stepsize-decaying scheme and (2) {tail polynomially stepsize-decaying scheme}, in an \emph{instance-wise} manner. 
Our result shows that the variance error of SGD with tail polynomially decaying stepsize is \textit{instance-wise no better} than that of SGD with tail geometrically decaying stepsize, given the same optimization trajectory length (i.e., summation of stepsizes). In contrast, the comparison between these two stepsize schemes made in \citet{ge2019step} only concerns the worst-case result: the worst-case excess risk bound achieved by geometrically decaying stepsize is strictly better than the worst case bound achieved by polynomially decaying stepsize. Thus, their analysis does not rule out the possibility that for some problem instances, polynomially decaying stepsize can generalize better than the geometrically decaying one.


\begin{algorithm}[t]
    \caption{\textsc{Last Iterate SGD with Tail Geometric Decaying Stepsize}}\label{alg:SGD}
    \begin{algorithmic}[1]
        \REQUIRE Initial weight $\wB$, initial stepsize $\gamma$, total sample size $N$, first phase length $s$, decaying phase length $K$
        \FOR{$t = 1, \dots, N$}
            \IF{$t > s$ and $(t - s ) \bmod{ K} = 0$}
                \STATE $\gamma \leftarrow \gamma / 2$
            \ENDIF
            \STATE $\wB \leftarrow \wB + \gamma (y - \abracket{\wB, \xB})\xB$, with a fresh data $(\xB, y)$
        \ENDFOR
        \RETURN $\wB$
    \end{algorithmic}
\end{algorithm}

Our analysis follows and extends the operator method for analyzing SGD in linear regression \citep{dieuleveut2017harder,jain2017markov,jain2017parallelizing,neu2018iterate,ge2019step,zou2021benign}. 
Specifically, we develop a novel, multi-phase analysis for the bias error of the last SGD iterate, which sharpens existing results (see Section \ref{sec:proof-sketch}).
We believe our proof technique is of broader interest and can be applied to analyze other variants of SGD such as SGD with momentum.

\noindent\textbf{Notation.}
We reserve lower-case letters for scalars, lower-case boldface letters for vectors, upper-case boldface letters for matrices, and upper-case calligraphic letters for linear operators on symmetric matrices. For two positive-value functions $f(x)$ and $g(x)$ we write  $f(x)\lesssim g(x)$ or $f(x)\gtrsim g(x)$
if $f(x) \le cg(x)$ or $f(x) \ge cg(x)$ for some absolute constant 
$c>0$ respectively.
For two vectors $\uB$ and $\vB$ in a Hilbert space, their inner product is denoted by $\abracket{\uB, \vB}$ or equivalently, $\uB^\top \vB$.
For a matrix $\AB$, its spectral norm is denoted by $\norm{\AB}_2$.
For two matrices $\AB$ and $\BB$ of appropriate dimension, their inner product is defined as $\langle \AB, \BB \rangle := \tr(\AB^\top \BB)$.
For a positive semi-definite (PSD) matrix $\AB$ and a vector $\vB$ of appropriate dimension, we write $\norm{\vB}_{\AB}^2 := \vB^\top \AB \vB$.
The Kronecker/tensor product is denoted by $\otimes$.
Finally, $\log (\cdot )$ refers to logarithm base $2$.


\begin{figure}[t]
    \centering
    \includegraphics[width=0.7\linewidth]{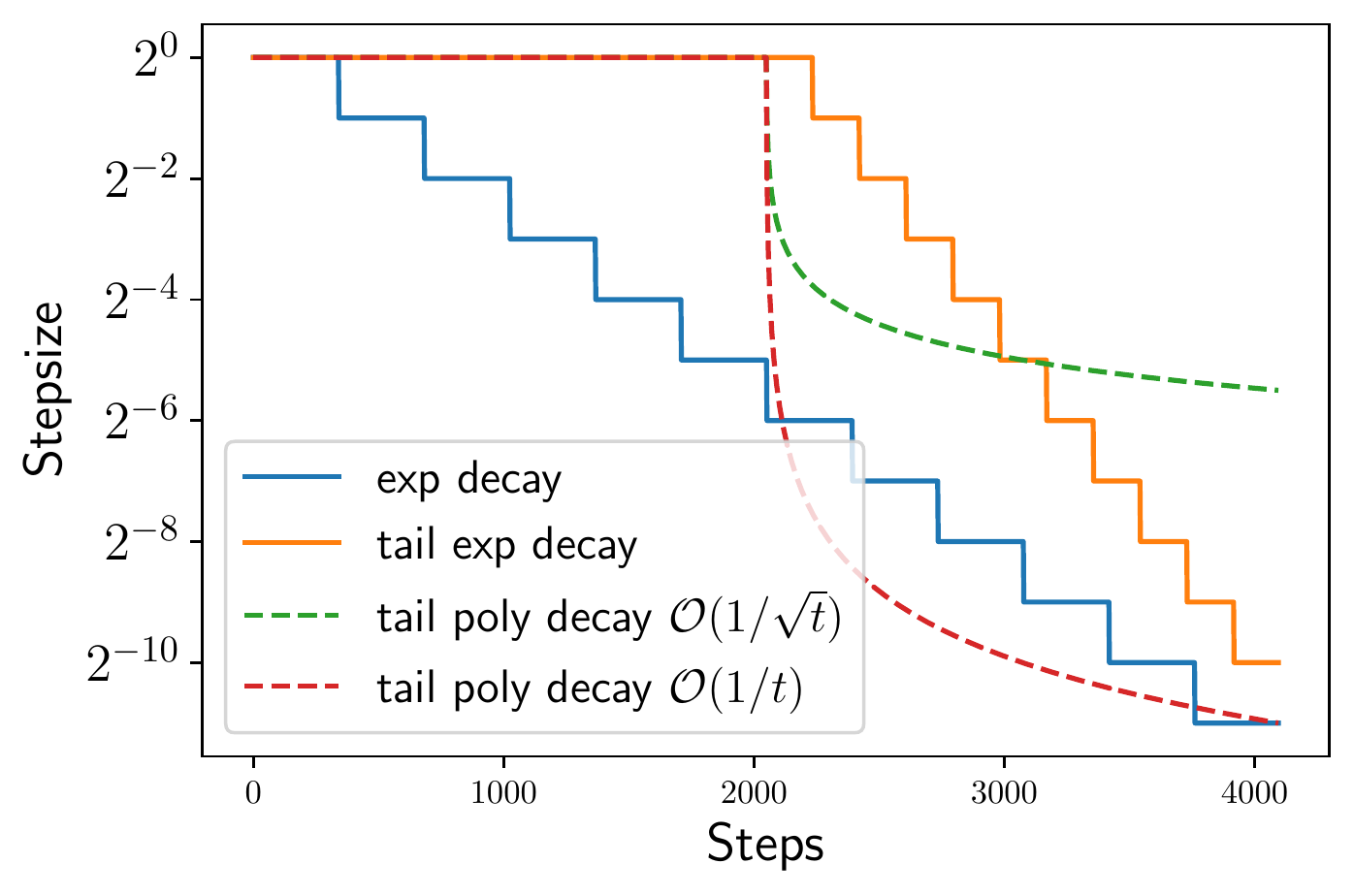}
    \vspace{-0.6cm}
    \caption{\small 
      Four stepsize decaying examples given by \eqref{eq:geometry-tail-decay-lr} and \eqref{eq:poly-tail-decay-lr}. $\gamma_0 = 1$ and $N = 4096$.
      \textsc{exp decay}: \eqref{eq:geometry-tail-decay-lr} with $s=0$ and $K = \lceil N / \log N \rceil$;
      \textsc{tail exp decay}: \eqref{eq:geometry-tail-decay-lr} with $s= N / 2$ and $K = \lceil (N-s) / \log (N-s)\rceil$;
      \textsc{tail poly decay $\Ocal(1/\sqrt{t})$}: \eqref{eq:poly-tail-decay-lr} with $s = N /2$ and $a = 0.5$;
      \textsc{tail poly decay $\Ocal(1/t)$}: \eqref{eq:poly-tail-decay-lr} with $s = N /2$ and $a = 1$.}
    \label{fig:stepsize-decay}
    \vspace{-0.6cm}
\end{figure}

\section{Related Work}\label{sec:related-works}

\noindent\textbf{Problem-Dependent Bounds for Linear Regression.}
We first discuss a set of dimension-free and problem-dependent bounds for linear regression that are similar to what we show in this paper.
\citet{bartlett2020benign} proved risk bounds of ordinary least square (OLS) for overparameterized linear regression in terms of the full eigenspectrum of the data covariance matrix, and showed that benign overfitting can occur even when OLS memorizes the training data. \citet{tsigler2020benign} extended the benign overfitting result of OLS to ridge regression, and proved diminishing risk bounds for a larger class of least square problems. \citet{zou2021benign} proved problem-dependent risk bounds for constant-stepsize SGD with iterate averaging (and tail-averaging) and compared the algorithmic regularization afforded by SGD with OLS and ridge regression.
Our excess risk upper bound (Theorem \ref{thm:tail-decay-upper-bound}) for last iterate SGD is comparable to theirs for SGD with averaging (Theorems 2.1 and 5.1 in \citet{zou2021benign}).
Due to this similarity, the benefits of SGD with tail-averaging over ridge regression, as discussed in \citet{zou2021benefits}, naturally extends to the, more practical, last iterate SGD studied in this paper.

\noindent\textbf{Nonparamatric Bounds for SGD.}
We then discuss other SGD risk bounds for infinite-dimensional/nonparamatric linear regression \citep{DieuleveutB15,lin2017optimal,mucke2019beating,berthier2020tight,varre2021last}.
\citet{DieuleveutB15} only discussed linear regression with data covariance whose spectrum decays polynomially, in contrast our results apply to general data covariance.
The works by \citet{berthier2020tight,varre2021last} only dealt with bias error (i.e., they assume no additive label noise), but we provide both variance and bias error bounds.
Compared to \citet{lin2017optimal,mucke2019beating,berthier2020tight,varre2021last}, our results rely on a different set of assumptions:
they assume a stronger condition on the optimal model parameter ($\wB^*$), which requires $\norm{\HB^{-\alpha} \wB^*}_2$ to be finite for some constant $\alpha > 0$ where $\HB$ is the data covariance; though we do not require this, our assumption on the fourth moment operator is stronger (see Assumption \ref{assump:fourth-moment}).

\noindent\textbf{Last Iterate SGD with Decaying Stepsize in the Classical Regime.}
In the finite-dimensional setting, there is a rich literature considering the last iterate SGD with decaying stepsize. For example, polynomially decaying stepsizes are studied in \citep{dekel2012optimal,rakhlin2011making,lacoste2012simpler,bubeck2014theory}, and geometrically decaying stepsizes are considered in \citep{davis2019stochastic,ghadimi2012optimal,hazan2014beyond,aybat2019universally,kulunchakov2019generic,ge2019step}; besides, a recent work by \citet{pan2021eigencurve} explored eigenvalue-dependent stepsizes.
However, the bounds derived in the aforementioned papers are all \emph{dimension-dependent} and therefore cannot be applied to the overparameterized setting.
In this regard, our work can be viewed as a dimension-free, problem-dependent extension of \citet{ge2019step}'s results that are limited to finite-dimensional, worst-case scenarios.

Finally, we would like to refer the reader to Table \ref{table:comparison} in Section \ref{sec:upper-bound} for a detailed comparison between our results and several existing ones \citep{ge2019step,bach2013non,zou2021benign}.

\vspace{-0.25cm}
\section{Problem Setup and Preliminaries}
\vspace{-0.1cm}
We now formally set up the problem.
\vspace{-0.1cm}

\noindent\textbf{High-Dimensional Linear Regression.}
Let $\xB$ be a feature vector in a Hilbert space that can be $d$-dimensional or countably infinite dimensional, and $y \in \Rbb$ be the response.
\emph{Linear regression} concerns the following objective:
\begin{equation}\label{eq:linear-regression}
    \min_{\wB} L(\wB), \ \text{where}\ L(\wB) := \half \Ebb \bracket{y - \abracket{\wB, \xB}}^2,
\end{equation}
where $\wB$ is a weight vector to be learned, 
and the expectation is over an unknown joint distribution $\Dcal$ of $(\xB, y)$\footnote{Unless otherwise noted, all expectations in this paper are taken with respect to the joint distribution of $(\xB, y)$.}. 

\noindent\textbf{SGD.}
We consider solving \eqref{eq:linear-regression} using \emph{stochastic gradient descent} (SGD). The weight vector is initialized at $\wB_0$ in the Hilbert space; then at the $t$-th iteration, a fresh data $(\xB_t, y_t)$ is drawn independently from the distribution, and the weight vector is updated according to 
\begin{equation}\label{eq:sgd-iterates}
    \wB_t = \wB_{t-1} + \gamma_t ( y_t - \abracket{\wB_{t-1}, \xB_t} ) \xB_t,
    \ t= 1,2,\dots, N,
\end{equation}
where $\gamma_t > 0$ is the stepsize at step $t$.
Our main focus in this paper is \emph{last iterate SGD with decaying stepsize}, which uses a sequence of properly decaying stepsize $(\gamma_t)_{t=1}^{N}$, and outputs the last iterate $\wB_{N}$.
For example, one can use \emph{tail geometrically decaying stepsize} (see also Algorithm \ref{alg:SGD}): 
\begin{equation}\label{eq:geometry-tail-decay-lr}
    \gamma_t = 
    \begin{cases}
    \gamma_0, & 0 \le t \le s; \\
    {\gamma_0}/{2^\ell}, & s < t \le N, \ell = \left\lfloor (t - s) / K \right\rfloor,
    \end{cases}
\end{equation}
where the stepsize is kept as a constant in the first $s$ steps, and is then divided by a factor of $2$ every $K$ steps. Figure \ref{fig:stepsize-decay} shows two examples of such stepsize decay schemes.
We note that \eqref{eq:geometry-tail-decay-lr} captures the widely used stepsize decaying scheduler in deep learning \citep{he2015deep}: the stepsize is epoch-wise a constant, and decays geometrically after every certain number of epochs. 

Another widely studied variant of SGD is \emph{constant-stepsize} SGD with \emph{averaging}. More specifically, it updates the iterate according to \eqref{eq:sgd-iterates} with a constant stepsize, i.e., $\gamma_t = \gamma$, and its final output is an averaging of all iterates ($\frac{1}{N} \sum_{t=0}^{N-1} \wB_{t}$) or only the tail iterates ($\frac{1}{N-s}\sum_{t=s}^{N-1}\wB_{t}$). 
Compared with last iterate SGD, SGD with averaging is less practical but more theoretically favorable.
For a few examples, the risk bounds of SGD with averaging have been studied in both the classical underparameterized regime \citep{bach2013non,dieuleveut2017harder,jain2017markov,jain2017parallelizing,neu2018iterate} and the overparameterized setting \citep{DieuleveutB15,zou2021benign}.

Next we review a set of assumptions 
for our analysis.
\begin{assumption}[Regularity conditions]\label{assump:second-moment}
Denote \(
\HB := \Ebb[\xB \xB^\top],
\)
and assume that $\HB$ is (entry-wisely) finite
and $\tr(\HB)$ is finite.
For convenience, we further assume that $\HB$ is strictly positive definite and that $L(\wB)$ admits a unique global optimum, denoted by $\wB^* := \arg\min_{\wB} L(\wB)$.
\end{assumption}
The condition $\HB \succ 0$ is only made for simple presentation; if $\HB$ has zero eigenvalues, one can choose 
$\wB^* = \arg\min\{ \norm{\wB}_2: \wB \in \arg\min L(\wB) \}$, and our results still hold.
This argument also holds in a reproducing kernel Hilbert space \citep{scholkopf2002learning}.

\begin{assumption}[Fourth moment conditions]\label{assump:fourth-moment}
Assume that the fourth moment of $\xB$ is finite and:
\begin{enumerate}[label=\Alph*]
    \item There is a constant $\alpha > 0$, such that for every PSD matrix $\AB$, we have
\begin{equation*}
\Ebb [\xB \xB^\top \AB \xB \xB^\top] \preceq \alpha \cdot \tr (\HB \AB) \cdot \HB.
\end{equation*}
Clearly, it must hold that $\alpha \ge 1$. \label{item:fourth-moement-upper}
\item \label{item:fourth-moement-lower} There is a constant $\beta > 0$, such that for every PSD matrix $\AB$, we have
\[
\Ebb [\xB \xB^\top \AB \xB \xB^\top] - \HB \AB \HB   \succeq  \beta \cdot \tr (\HB \AB) \cdot \HB.
\]
\end{enumerate}
\end{assumption}
To give an example, if $\HB^{-\half} \xB$ satisfies Gaussian distribution, then Assumption \ref{assump:fourth-moment} holds with $\alpha = 3$ and $\beta = 1$.
More generally, Assumption \ref{assump:fourth-moment}\ref{item:fourth-moement-upper} holds for data distributions with a \emph{bounded kurtosis along every direction} \citep{dieuleveut2017harder}, i.e., there is a constant $\kappa > 0$ such that
\begin{equation} \label{eq:bounded-kurtosis}\tag{ \ref{assump:fourth-moment}\ref{item:fourth-moement-upper}'}
    \text{for every } \vB,\ \Ebb [\abracket{\vB, \xB}^4] \le \kappa \abracket{\vB, \HB \vB}^2.
\end{equation}
One can verify that condition \eqref{eq:bounded-kurtosis} (hence Assumption \ref{assump:fourth-moment}\ref{item:fourth-moement-upper}) is weaker than assuming a sub-Gaussian tail for the distribution of $\HB^{-\half}\xB$ (see Lemma A.1 in \citet{zou2021benign}), where the latter condition is typically made in regression analysis \citep{HsuKZ14,bartlett2020benign,tsigler2020benign}.
On the other hand, Assumption \ref{assump:fourth-moment}\ref{item:fourth-moement-upper}  is stronger than the condition $\Ebb [\xB \xB^\top \xB \xB^\top] \preceq R^2 \HB$ for some constant $R^2 > 0$, as often assumed in many SGD analysis \citep{bach2013non,dieuleveut2017harder,jain2017markov,jain2017parallelizing,neu2018iterate,ge2019step}. 
We refer the reader to Appendix \ref{append:sec:fourth-moment-assump-examples} for more examples for Assumption \ref{assump:fourth-moment}.



\begin{assumption}[Noise condition]\label{assump:noise}
Denote 
\[
\SigmaB := \Ebb[(y - \abracket{\wB^*, \xB})^2 \xB \xB^\top],\quad 
\sigma^2 := \big\| \HB^{-\half} \SigmaB \HB^{-\half} \big\|_2, 
\]
and assume $\SigmaB$ and $\sigma^2$ are finite.
\end{assumption}
Here $\SigmaB$ is the covariance matrix of the gradient noise at the optimal $\wB^*$, and $\sigma^2$ characterizes the noise level in that $\SigmaB \preceq \sigma^2 \HB$. 
Assumption \ref{assump:noise} allows the additive noise to be mis-specified \citep{dieuleveut2017harder,jain2017parallelizing}; and in particular, Assumption \ref{assump:noise} is directly implied by the following Assumption \ref{assump:well-specified-noise} for a well-specified linear regression model.

\begin{taggedassumption}{\ref{assump:noise}'}[Well-specified noise]\label{assump:well-specified-noise}
Assume the response is generated by
\[
y = \abracket{\wB^*, \xB} + \epsilon,\quad
\epsilon \sim \Ncal (0, \sigma^2),
\]
where $\epsilon$ is independent with $\xB$.
\end{taggedassumption}

\noindent\textbf{Additional Notation.}
Denote the eigen decomposition of the Hessian by
\(\HB = \sum_{i} \lambda_i \vB_i \vB_i^\top\), where $(\lambda_i)_{i\ge 1}$ are eigenvalues in a non-increasing order and $(\vB_i)_{i\ge 1}$ are the corresponding eigenvectors.
We denote $\HB_{k^*:k^\dagger} := \sum_{k^* < i \le k^\dagger} \lambda_i \vB_i \vB_i^\top$, where $0\le k^* \le k^\dagger$ are two integers, and we allow $k^\dagger = \infty$. For example,
\[
\HB_{0:k} = \sum_{1\le i \le k} \lambda_i \vB_i \vB_i^\top,\quad 
\HB_{k:\infty} = \sum_{i > k} \lambda_i \vB_i \vB_i^\top.
\]
Similarly, we denote $\IB_{k^*:k^\dagger} := \sum_{k^* < i \le k^\dagger} \vB_i \vB_i^\top$.

\section{Main Results}\label{sec:exponential_decay}

In this section, we present our main results.

\subsection{An Upper Bound}\label{sec:upper-bound}
We begin with an excess risk upper bound for last iterate SGD with tail geometrically decaying stepsize.

\begin{theorem}[An upper bound]\label{thm:tail-decay-upper-bound}
Consider last iterate SGD with stepsize scheme \eqref{eq:geometry-tail-decay-lr}.
Suppose Assumptions \ref{assump:second-moment},  \ref{assump:fourth-moment}\ref{item:fourth-moement-upper} and \ref{assump:noise} hold.
Let $K := \lceil (N-s) / \log (N-s) \rceil$.
Suppose $\gamma_0 < 1/(3\alpha\tr(\HB)\log (s+K))$. Then we have 
\[
\Ebb [ L(\wB_N) - L(\wB^*) ] \le \biasErr + \varErr,
\]
where 
\begin{align*}
    & \biasErr
    \lesssim \frac{\big\| (\IB-\gamma_0 \HB)^{s+K} (\wB_0 - \wB^*) \big\|^2_{\IB_{0:k^*}} }{\gamma_0 K} \ +  \\
    &  \big\| (\IB-\gamma_0 \HB)^{s+K} (\wB_0 - \wB^*) \big\|^2_{\HB_{k^*:\infty}}  +  \log (s+K)\ \cdot \\
    & \ \bigg( \frac{\|\wB_0 - \wB^* \|^2_{\IB_{0:k^\dagger}}}{\gamma_0 (s+K)}  + \|\wB_0 - \wB^* \|^2_{\HB_{k^\dagger:\infty}} \bigg)  \cdot \frac{\DIM}{K},
\end{align*}
and 
\begin{align*}
    \varErr \le \frac{8\sigma^2}{1-\alpha \gamma_0 \tr(\HB)}\cdot \frac{\DIM}{K}.
\end{align*}
Here $k^*$, $k^\dagger$ are arbitrary indexes, and the \emph{effective dimension} is defined by
\[
\DIM := k^* + \gamma_0 K \sum_{k^* < i \le k^\dagger}\lambda_i + \gamma_0^2 K (s+K) \sum_{i>k^\dagger}\lambda_i^2.
\]
Moreover, the bound is minimized for $k^*:= \max\{ k: \lambda_k \ge 1/(\gamma_0 K) \}$ and $k^\dagger := \max\{k: \lambda_k \ge 1/(\gamma_0 (s+K)) \}$.
\end{theorem}

The excess risk bound in Theorem \ref{thm:tail-decay-upper-bound} consists a \emph{bias error} term that stems from the incorrect initialization $\wB_0 - \wB^* \ne 0$, and a \emph{variance error} term that stems from the presence of additive noise $y - \abracket{\wB^*, \xB} \ne 0$. 
Our bound is \emph{dimension-free} and \emph{problem-dependent}:
instead of depending on the ambient dimension $d$, it depends on the \emph{effective dimension} $\DIM$, which is jointly determined by the problem and the algorithm.
In particular, when the eigenspectrum of the data covariance decays fast, the effective dimension $\DIM$ could be much smaller than the ambient dimension $d$ (and sample size $N$) to enable generalization in the overparameterized scenarios. 

\setcellgapes{2pt}
\begin{table*}[t]
\caption{ 
A comparison between our result and several existing results. See Section \ref{sec:upper-bound} for more details.
}
\label{table:comparison}
\centering\makegapedcells
 \begin{tabular}{ c | c c c c } 
 \toprule
   & {\small \citet{bach2013non}} & {\small\citet{ge2019step}} & {\small \citet{zou2021benign}}  & Ours \\ 
 \midrule
 output  & averaged iterate & last iterate & averaged iterate & last iterate \\ 
 \midrule
  \makecell{initial \\stepsize}  & $\gamma \lesssim 1$ & $\gamma \lesssim 1$  & $\gamma \lesssim 1$ & $\gamma \lesssim 1/ \log (N)$ \\
 \midrule
 \makecell{effective\\ number of \\ steps ($\STEPS$)} & $N$  & $ \dfrac{ N }{\log (N)}$ & $N$ & $ \dfrac{N}{\log (N)} $ \\
 \midrule
 \makecell{effective \\ dimension \\ ($\DIM$) }  & $d$ & $d$ & $k^* + \gamma^2 N^2 \sum_{i>k^*}\lambda_i^2 $ & $k^* + \gamma^2 \STEPS^2 \sum_{i>k^*}\lambda_i^2 $  \\
 \midrule 
 \makecell{effective \\ noise \\ ($\NOISE$)} & $\sigma^2$  & $\sigma^2$  & $\begin{gathered}
     \sigma^2 +  \\  \dfrac{\| \wB^* \|^2_{\IB_{0:k^*}}}{\gamma \STEPS}  + \|\wB^* \|^2_{\HB_{k^*:\infty}} 
 \end{gathered}$ & \( \begin{gathered}
     \sigma^2 + \log (N)  \cdot \\ \bigg( \dfrac{\| \wB^* \|^2_{\IB_{0:k^*}}}{\gamma \STEPS}  + \|\wB^* \|^2_{\HB_{k^*:\infty}} \bigg)
 \end{gathered} \)  \\
 \midrule
 \makecell{effective \\ bias  error \\ ($\EffBias$)}  & \( \dfrac{ \| \wB^* \|^2_{2}}{\gamma \STEPS} \) & \( \dfrac{d \| \wB^* \|^2_{2}}{\gamma \STEPS}  \) & $ \dfrac{\| \wB^* \|^2_{\HB^{-1}_{0:k^*}}}{\gamma^2 \STEPS^2}  + \|\wB^* \|^2_{\HB_{k^*:\infty}}$ & \( \begin{gathered}
     \dfrac{\| (\IB-\gamma\HB)^\STEPS \wB^* \|^2_{\IB_{0:k^*}}}{\gamma \STEPS} \\ + \|(\IB-\gamma\HB)^\STEPS \wB^* \|^2_{\HB_{k^*:\infty}} 
 \end{gathered}\) \\
 \midrule
 \makecell{ unified \\ risk bound} & \multicolumn{4}{c}{$\EffBias + \NOISE \cdot \dfrac{\DIM}{\STEPS}$} \\
\bottomrule
 \end{tabular}
\end{table*}

For example, let us consider $s = N /2$ and $K = {N} / {(2 \log (N/2) )}$, which corresponds to SGD that starts decaying stepsize after seeing half of the samples.
Then the excess risk bound vanishes provided that $\DIM = o(K)$
, or in other words, 
\[ k^* = o\bigg( \frac{N}{\log (N)} \bigg), \sum_{k^* < i \le k^\dagger} \lambda_i = o(1), \sum_{i > k^\dagger} \lambda_i^2 = o\Big(\frac{1}{N}\Big). \]
Theorem \ref{thm:tail-decay-upper-bound} allows the last iterate of SGD to generalize even in the overparameterized regime ($d > N$). 
Several concrete examples are presented in Corollary \ref{thm:examples}.

\begin{mycorollary}[Example data distributions]\label{thm:examples}
Under the same conditions as Theorem \ref{thm:tail-decay-upper-bound}, suppose that $s = N /2$, $K  = {N} / {(2 \log (N/2) )}$, $\gamma_0 = 1/(4 \alpha\tr(\HB)\log (N)))$, and $\|\wB_0 - \wB^* \|_2 $ is finite. Recall the eigenspectrum of $\HB$ is $(\lambda_k)_{k \ge 1}$.
\begin{enumerate}
    \item If $\lambda_k = k^{-(1+r)}$ for some constant $r > 0$, then the excess risk is $\Ocal\big(   N^{\frac{-r}{1+r}}\cdot \log^{\frac{r-1}{1+r}} (N) \big)$.
    \item If $\lambda_k = k^{-1} \log^{-r}(k+1)$ for some constant $r > 1$, then the excess risk is $\Ocal\big( \log^{-r} (N) \big)$.
    \item If $\lambda_k = 2^{-k}$, then the excess risk is $\Ocal\big( N^{-1}  \log^2 (N) \big)$.
\end{enumerate}
\end{mycorollary}
These examples are from Corollary 2.3 in \citet{zou2021benign} for SGD with iterate-averaging (one can verify that their Corollary 2.3 also holds for constant-stepsize SGD with tail-averaging with $s = N/2$). 
Comparing our Corollary \ref{thm:examples} with Corollary 2.3 in \citet{zou2021benign}, we can see that the excess risk bounds of last iterate SGD is inferior to that of SGD with averaging by at most polylogarithmic factors.

\noindent\textbf{Reduction to the Classical Regime.}
It is worthy noting that Theorem \ref{thm:tail-decay-upper-bound} nearly recovers the minimax optimal bounds \citep{polyak1992acceleration,bach2013non} in the classical regime when $d = o(N)$. 
In particular, let us set $k^* = k^\dagger = d$, $s=0$ and $K = N / \log (N)$, then Theorem \ref{thm:tail-decay-upper-bound} implies:
\begin{align*}
    & \biasErr \lesssim \frac{\norm{\wB_0 - \wB^*}_2^2 \log (N)}{\gamma_0 N} \Big( 1 + \frac{d \log^2 (N) }{N} \Big), \\ 
    & \varErr \lesssim \frac{\sigma^2 d \log (N)}{N}.
\end{align*}
Now choose $\gamma_0 = 1/(4\alpha\tr(\HB)\log\log(N))$ and recall $d = o(N)$, then both the bias and variance errors match the statistical minimax rates \citep{polyak1992acceleration,bach2013non} up to some logarithmic factors.




\noindent\textbf{Comparison with Existing Bounds.}
Table \ref{table:comparison} presents a detailed comparison between our result and several existing results, including \citet{ge2019step} for last iterate SGD and \citet{bach2013non,zou2021benign} for SGD with iterate averaging.
To unify notations, we use $\gamma$, $N$, $d$, $\wB^*$, $\sigma^2$ to denote the (initial) stepsize, the total number of steps, the ambient dimension, the optimal model parameter and the noise level, respectively. We also use \textit{effective number of steps} as the number of equivalently steps when using constant stepsize (or can be understood as the total optimization length). The \textit{effective dimension} can be understood as the number of useful dimensions (discovered by the algorithm) that contribute to the problem.
We also assume all algorithms are initialized from zero ($\wB_0 = 0$), without lose of generality.
To be consistent with the algorithmic setting of \citet{ge2019step}, we restrict our result to geometric decaying stepsize scheduler ($s=0$), which decreases the effective number of steps in our result.
Table \ref{table:comparison} shows that our result generalizes that in \citet{ge2019step} for last iterate SGD to high dimensional setting, and is comparable to that in \citet{zou2021benign} for SGD with iterate averaging ignoring some logarithmic factors.

\subsection{A Lower Bound}
We complement the above upper bound with a lower bound.
\begin{theorem}[A lower bound]\label{thm:tail-decay-lower-bound}
Consider last iterate SGD with stepsize scheme \eqref{eq:geometry-tail-decay-lr}.
Suppose Assumptions \ref{assump:second-moment}, \ref{assump:fourth-moment}\ref{item:fourth-moement-lower} and \ref{assump:well-specified-noise} hold.
Let $K = (N-s) / \log (N-s)$.
Suppose $K\ge 10$ and $\gamma_0 < 1 / \lambda_1$. Then we have 
\[
\Ebb [ L(\wB_N) - L(\wB^*) ] = \half \biasErr + \half \varErr,
\]
where 
\begin{align*}
    & \biasErr \ge
    \big\| (\IB-\gamma_0 \HB)^{s+2K} (\wB_0 - \wB^*) \big\|^2_{\HB} \ + \\
    &\qquad \qquad \frac{\beta}{1200} \cdot  \|\wB_0 - \wB^* \|^2_{\HB_{k^\dagger:\infty}}  \cdot \frac{\DIM}{K},
\end{align*}
and 
\begin{equation*}
    \varErr \ge \frac{\sigma^2}{400}\cdot \frac{\DIM}{K}.
\end{equation*}
Here $k^*:= \max\{k: \lambda_k \ge 1/(\gamma_0 K)  \}$, $k^\dagger := \max\{ k: \lambda_k \ge 1/(\gamma_0 (s+K))  \}$, and the \emph{effective dimension} is defined by
\[
\DIM := k^* + \gamma_0 K \sum_{k^* < i \le k^\dagger}\lambda_i + \gamma_0^2 K (s+K) \sum_{i>k^\dagger}\lambda_i^2.
\]
\end{theorem}
Theorem \ref{thm:tail-decay-lower-bound} provides a problem-dependent lower bound for last iterate SGD in the well-specified linear regression model. 
It shows that our variance error bound is tight up to constant; however, for our bias error bound, there is a gap ($1/(\gamma_0 K)$ vs. $(\IB - \gamma_0 \HB)^{K}\HB_{0:k^*}$) between the upper and lower bounds in the first term, and is missing a factor of $\|\wB_0 - \wB^* \|^2_{\HB_{0:k^\dagger}}$ and a $\log (s+K)$ factor in the second term.
These gaps are due to some technical difficulties to obtain an accurate bias bound on the last iterate of SGD. We leave it as a future work to close these gaps.

\subsection{Comparison with Polynomially Decaying Stepsize}\label{sec:poly-decay}
In terms of the statistical minimax rate, it is proved by \citet{ge2019step} that the last iterate of SGD performs better with geometrically decaying stepsize than with polynomially decaying stepsize. Nonetheless, their comparison is in terms of the \emph{worst-case} performance, and \citet{ge2019step} did not rule out the possibility that there could exist some linear regression problems such that SGD generalizes better with polynomially decaying stepsize.
Thanks to our sharp problem-dependent bounds on SGD with (tail) geometrically decaying stepsize, we are able to compare its performance with that of SGD with (tail) polynomially decaying stepsize, in an \emph{instance-wise} manner.
The \emph{(tail) polynomially decaying stepsize} is formally defined by
\begin{equation}\label{eq:poly-tail-decay-lr}
    \gamma_t = 
    \begin{cases}
    \gamma_0, & 0 \le t \le s; \\
    {\gamma_0} / {(t-s)^a}, & s < t \le N,
    \end{cases}
\end{equation}
for some $a \in [0,1]$.
We then present a problem-dependent excess risk lower bound for the last iterate of SGD with stepsize scheme \eqref{eq:poly-tail-decay-lr}. 
Due to the space limit, the following theorem focuses on $a\in[0, 1)$; the full version for $a \in [0,1]$ is stated as Theorem \ref{thm:tail_decay_poly_full_ver} in Appendix \ref{append-sec:poly}.

\begin{theorem}[A lower bound for poly-decaying stepsizes]\label{thm:tail_decay_poly}
Consider last iterate SGD with stepsize scheme \eqref{eq:poly-tail-decay-lr}.
Suppose Assumptions \ref{assump:second-moment}, \ref{assump:fourth-moment}\ref{item:fourth-moement-lower} and \ref{assump:well-specified-noise} hold. Suppose $\gamma_0< 1/(4\lambda_1)$, $s\gamma_0\ge \sum_{t>s}\gamma_t$, and $a\in[0, 1)$. Then we have
\[
\Ebb [ L(\wB_N) - L(\wB^*) ] = \half \biasErr + \half \varErr,
\]
where
\begin{align*}
    \biasErr &\gtrsim \big\| (\IB-\gamma_0 \HB)^{s+\frac{2N^{1-a}}{1-a}}\cdot (\wB_0-\wB^*)\big\|_{\HB}^2 \\
    &\qquad +  \beta\cdot \|\wB_0 - \wB^* \|^2_{\HB_{k^\dagger:\infty}} \cdot\frac{\DIM}{N},
    \end{align*}
    and
    \begin{align*}
      \varErr &\gtrsim \sigma^2\cdot \frac{\DIM}{N}.  
    \end{align*}
    Here $k^* := \max\{k:\gamma_0 \lambda_k\ge (1-a)/( 2(N-s)^{1-a} )\}$, $k^\dagger := \max\{k:\gamma_0 \lambda_k\ge 1/(2s)\}$, and the \emph{effective dimension} is defined by
    \begin{align*}
    \DIM & := \sum_{i\le k^*} \max\{ N^{1-a}\gamma_0\lambda_i ,\ a\log(N) \} \\
    &\quad + \gamma_0 N \sum_{k^*< i\le k^\dagger} \lambda_i +  \gamma_0^2 s N\sum_{i> k^\dagger} \lambda_i^2.
    \end{align*}
    

\end{theorem}

Comparing Theorem \ref{thm:tail_decay_poly} for tail polynomially decaying stepsize with Theorem \ref{thm:tail-decay-upper-bound} for tail geometrically decaying stepsize, the main difference is in the definition of the effective dimension $\DIM$. This is due to the different algorithmic regularization effects afforded by the different stepsize decaying schemes. With this difference in hand, our next theorem provides an instant-wise \emph{risk inflation} \citep{dhillon2013risk} comparison between (the last iterate of SGD with) these two stepsize decaying schemes.


\begin{theorem}[An instance-wise risk comparison]\label{thm:comparison}
Suppose Assumptions \ref{assump:second-moment}, \ref{assump:fourth-moment} and \ref{assump:well-specified-noise} all hold. Suppose $\gamma_0 < 1/(3\alpha\tr(\HB)\log (s+K))$. Let $N$ be the sample size, and set $s=N/2$. 
Let $\wB_N^{\mathrm{exp}}$ and $\wB_N^{\mathrm{poly}}$ be the last iterate of SGD with stepsize scheme \eqref{eq:geometry-tail-decay-lr} and \eqref{eq:poly-tail-decay-lr}, respectively.
Then there is a constant $C > 0$ such that 
\begin{align*}
& \Ebb [ L(\wB_N^{\mathrm{exp}}) - L(\wB^*) ] \le \notag\\
& \qquad C\cdot\big( 1 + \log(N)\cdot R(N)\big)\cdot \Ebb [ L(\wB_N^{\mathrm{poly}}) - L(\wB^*) ]
\end{align*}
for every problem-algorithm instance $(\HB, \wB^*, \gamma_0)$.
Here
\[ R(N) := \frac{\|\wB-\wB^*\|_{\IB_{0:k^\dagger}}^2/(\gamma_0 N)+\|\wB-\wB^*\|^2_{\HB_{k^\dagger:\infty}} }{\sigma^2}\]
for $k^\dagger:=\max\{k:\lambda_k\ge 1/(\gamma_0 N)\}$.
\end{theorem}

The choice of $s=N/2$ in Theorem \ref{thm:comparison} is for ensuring that the two SGD variants have the same optimization trajectory length, i.e., $\sum_{i=1}^N \gamma_i = \Theta(N\gamma_0)$. This rules out the trivial optimization difference in the bias error between the two SGD variants, so Theorem \ref{thm:comparison} reveals only the statistical difference between the two stepsize schemes.

Let us assume $\log (N)\cdot R(N)  \le 1$ for now. Then Theorem \ref{thm:comparison} reads that, \emph{for every problem instance}, with the same initial stepsize, the excess risk of SGD with tail geometrically decaying stepsize is \emph{no worse than} that of SGD with tail polynomially decaying stepsize, \emph{upto constant factors}. 
This suggests that for the last iterate of SGD, a tail geometrically decaying stepsize is \emph{always} as good as a tail polynomially decaying stepsize in terms of generalization.



We now discuss the quantity $\log(N)\cdot R(N)$ in Theorem \ref{thm:comparison}.
First of all, this quantity is rooted from the $\log (s+K)$ factor in the bias error upper bound in Theorem~\ref{thm:tail-decay-upper-bound}.
Therefore, the $\log(N)\cdot R(N)$ factor in Theorem~\ref{thm:comparison} might be an artifact that  can be removed given a tighter bias analysis (we conjecture that Theorem~\ref{thm:tail-decay-upper-bound} is not tight with the $\log (s+K)$ factor).
Moreover, we argue that $\log(N)\cdot R(N)$ itself is small in many scenarios so that the comparison in Theorem \ref{thm:comparison} is still meaningful. 
To see this, note that
\begin{align*}
R(N) \le {\|\wB-\wB^*\|_2^2} / (\gamma_0 N \sigma^2)
\end{align*}
by the definition of $k^\dagger$.
Thus, we have $\log(N)\cdot R(N) = \bigO{1}$ so long as $\|\wB_0-\wB^*\|_2^2 = \bigO{ \sigma^2 \gamma_0 N/\log(N))}$.

Figure \ref{fig:excess-risk} provides further empirical verification to our comparison of the two stepsize schemes for the last iterate of SGD. 
We see from Figure \ref{fig:excess-risk} that the last iterate of SGD generalizes significantly better with tail geometrically decaying stepsize than with tail polynomially decaying stepsize. 

\begin{figure*}[ht!]
     \centering
      \subfigure[\small{$\lambda_i=i^{-1}, \wB^{*}{[i]}=1$}]{\includegraphics[width=0.3\textwidth]{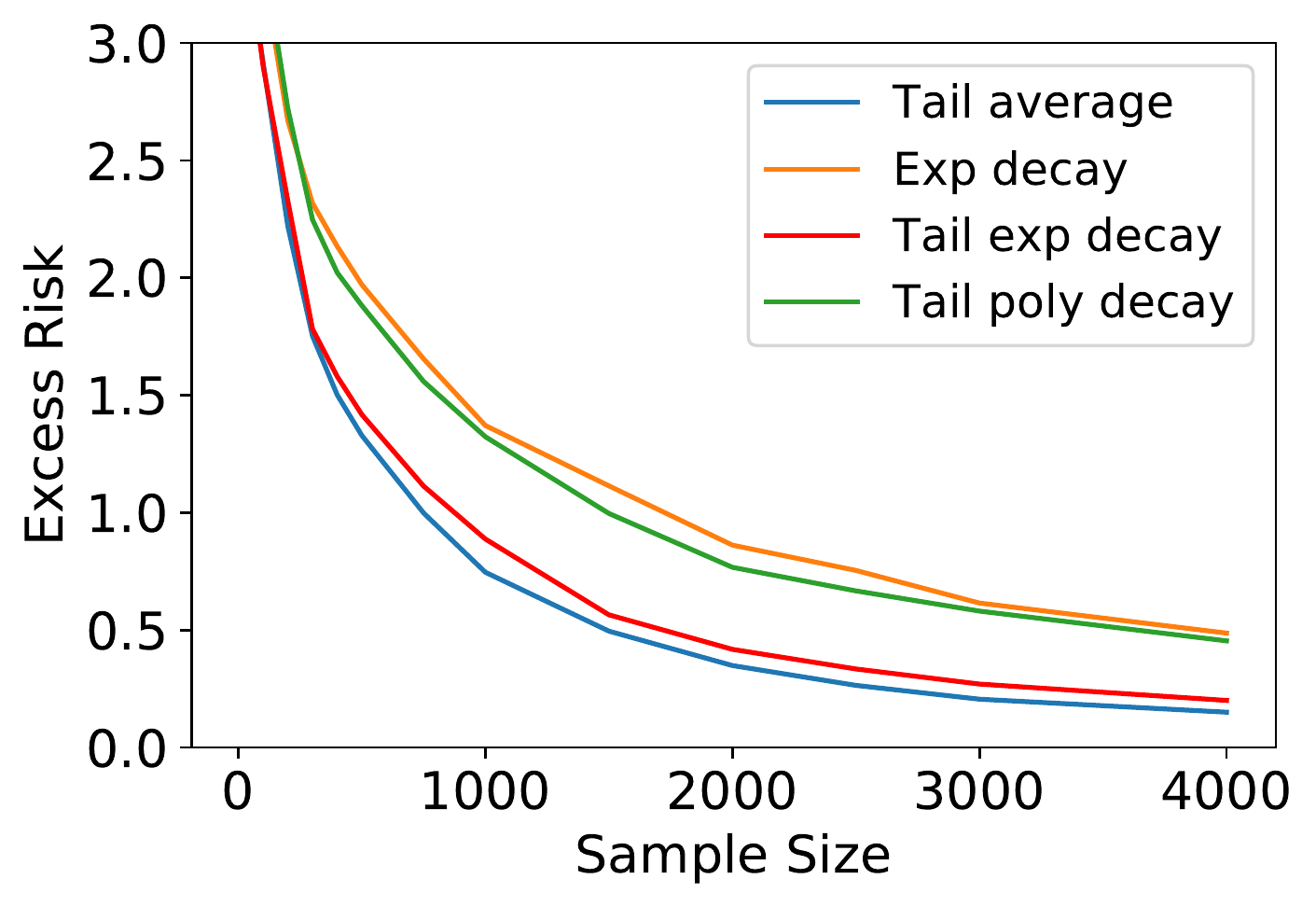}}
      \subfigure[\small{$\lambda_i=i^{-1}, \wB^{*}{[i]}=i^{-1}$}]{\includegraphics[width=0.3\textwidth]{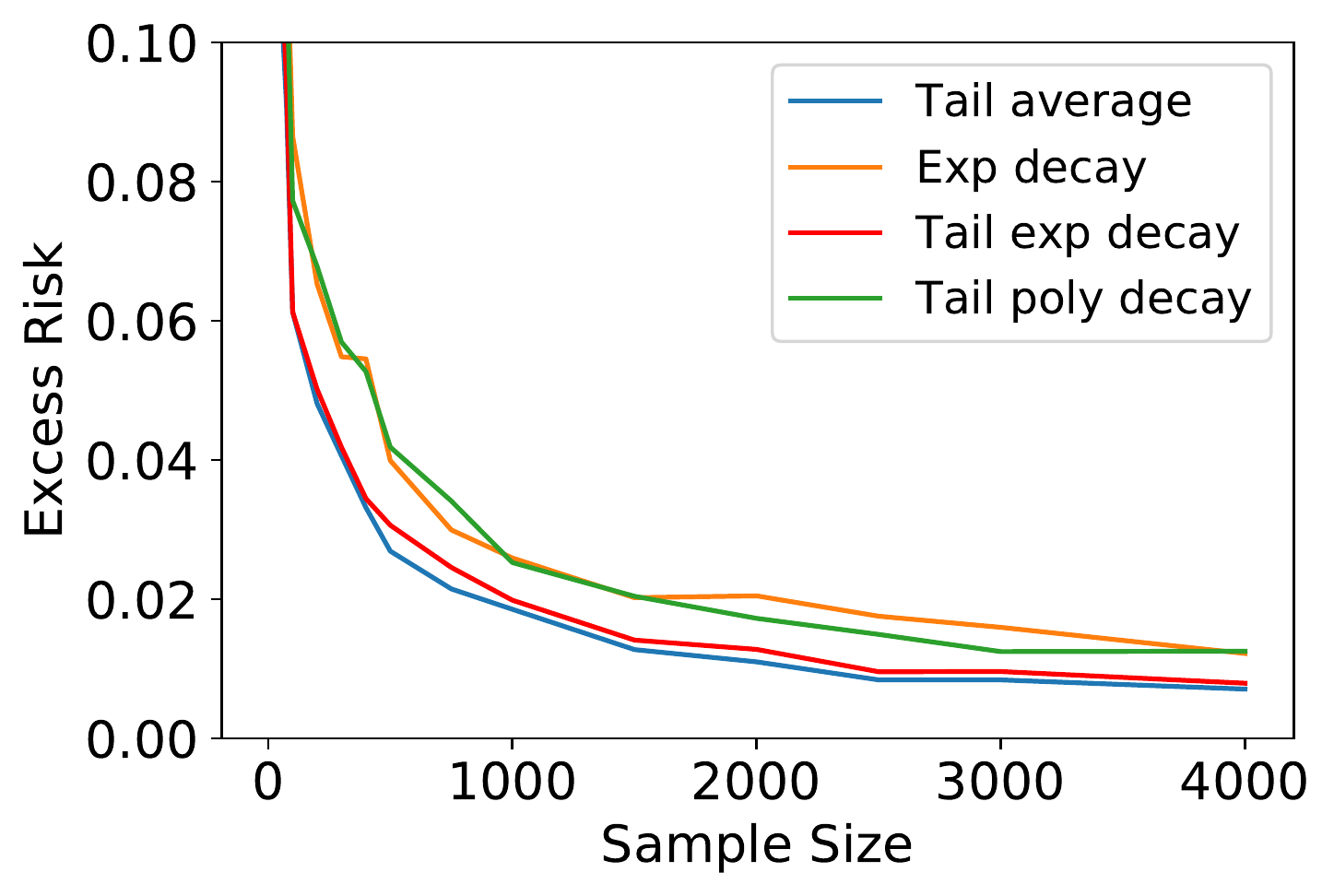}}
      \subfigure[\small{$\lambda_i=i^{-1}, \wB^{*}{[i]}=i^{-2}$}]{\includegraphics[width=0.3\textwidth]{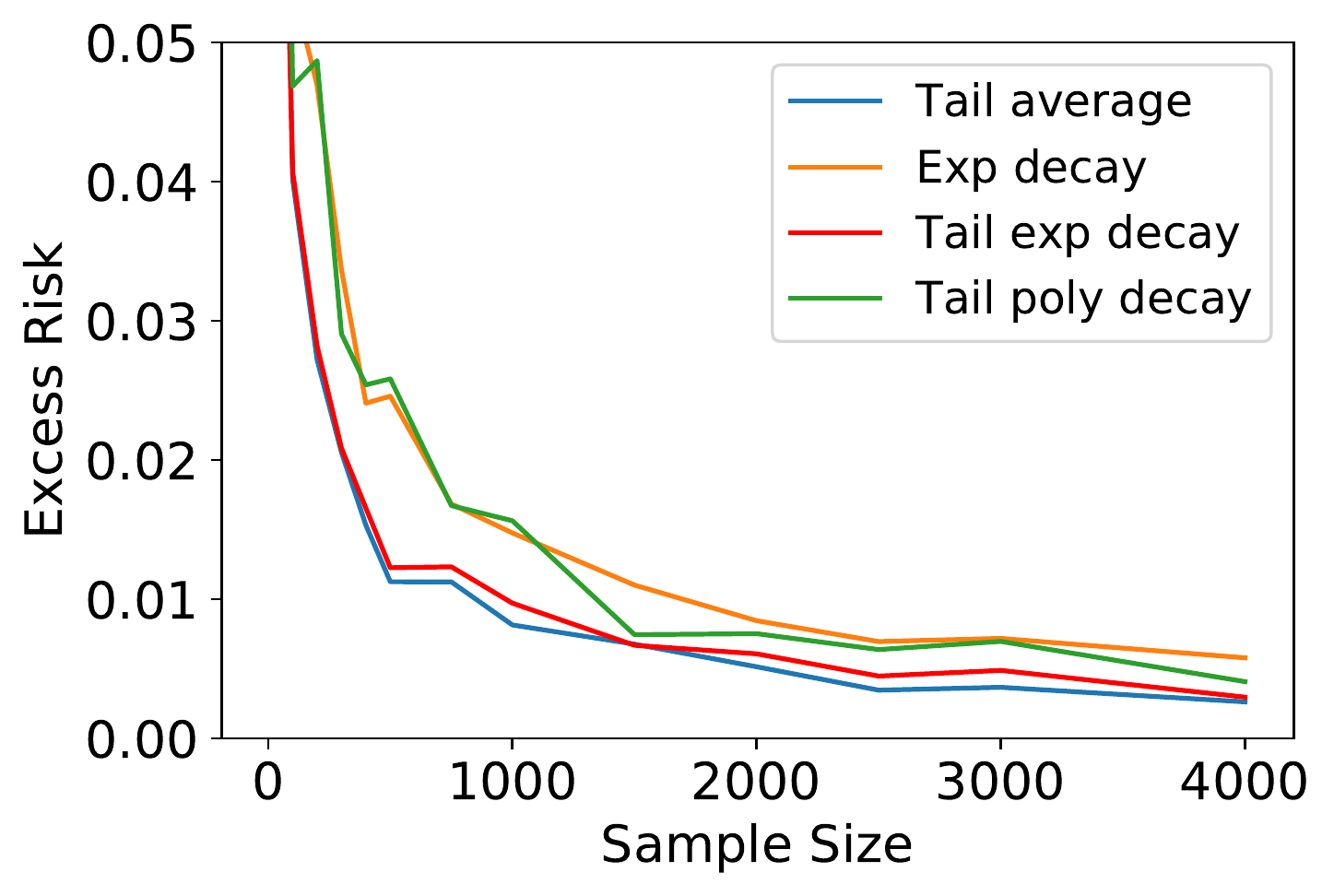}}\\
      \vspace{-0.1cm}
      \subfigure[\small{$\lambda_i=i^{-2}, \wB^{*}{[i]}=1$}]{\includegraphics[width=0.3\textwidth]{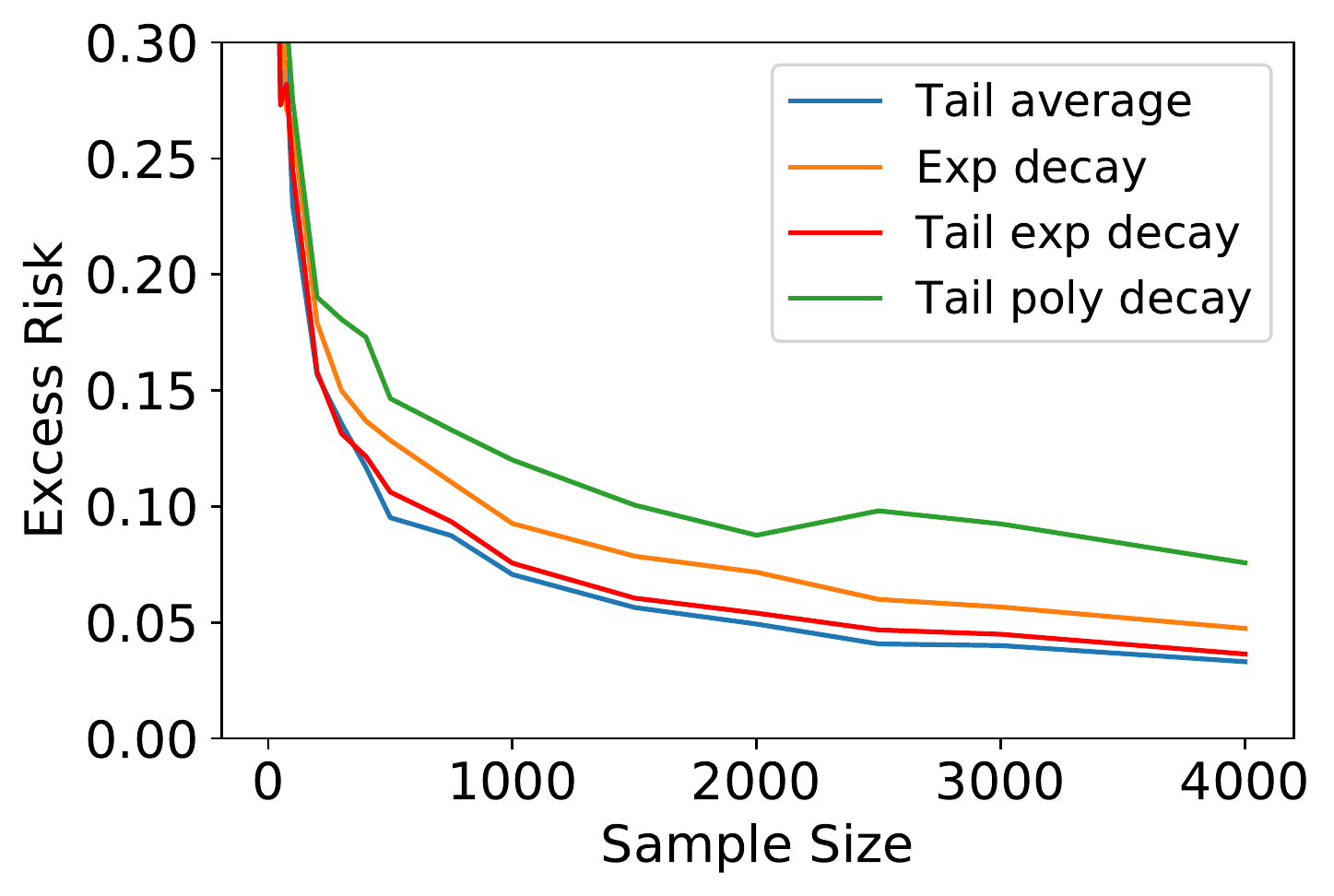}}
      \subfigure[\small{$\lambda_i=i^{-2}, \wB^{*}{[i]}=i^{-1}$}]{\includegraphics[width=0.3\textwidth]{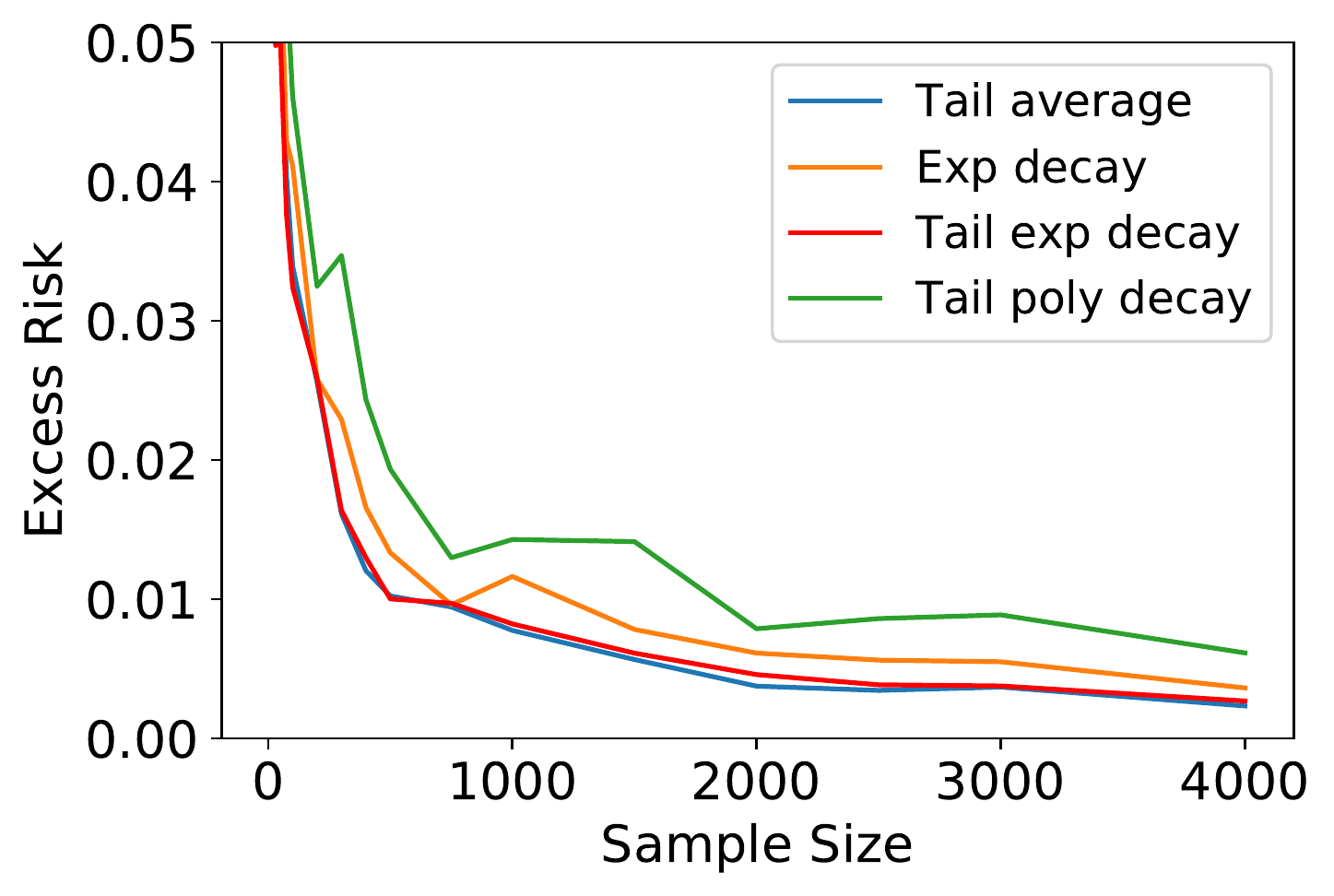}}
      \subfigure[\small{$\lambda_i=i^{-2}, \wB^{*}{[i]}=i^{-2}$}]{\includegraphics[width=0.3\textwidth]{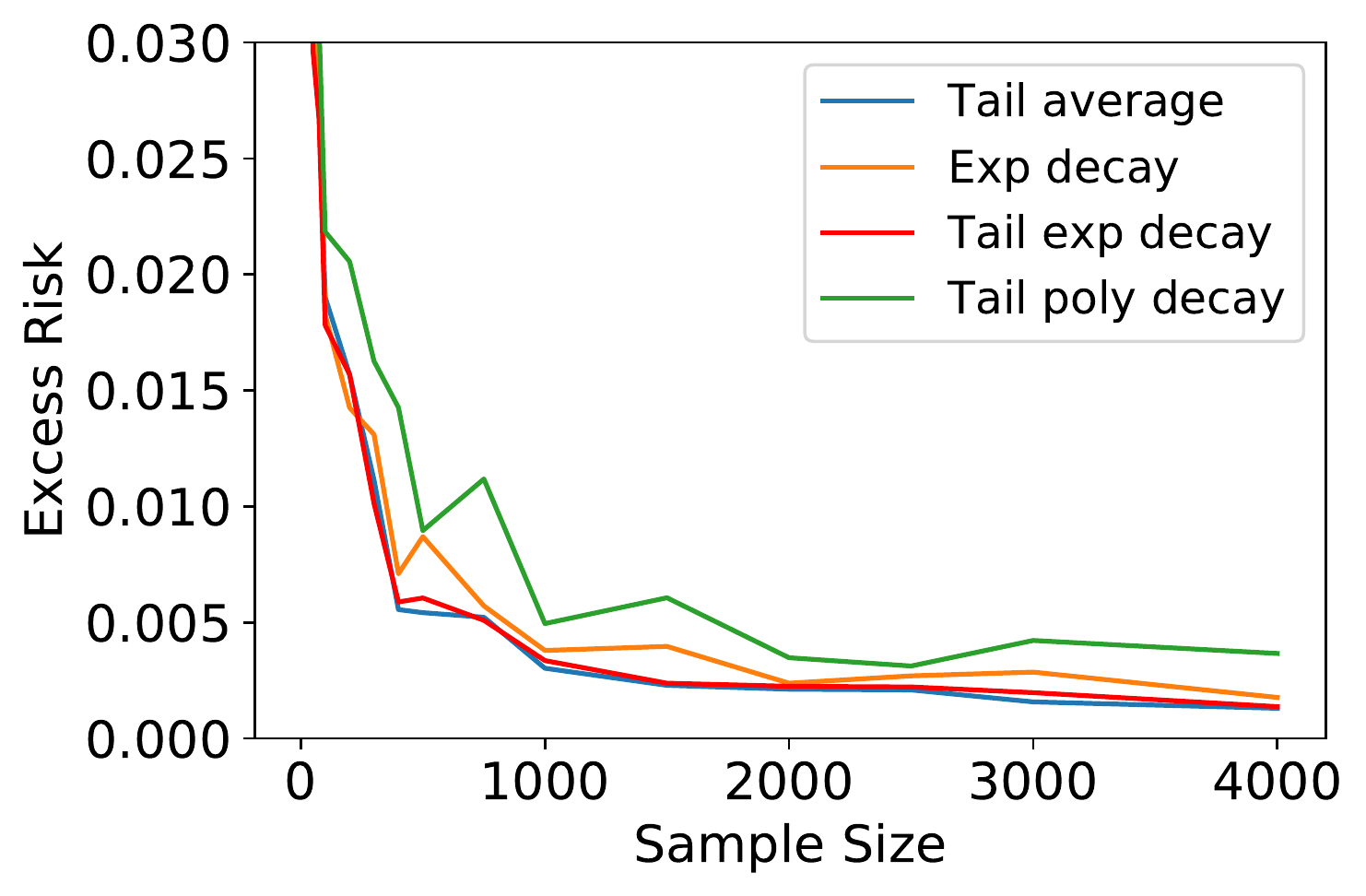}}
      \vspace{-0.3cm}
    \caption{\small
    Excess risk comparison between SGD variants. The problem dimension is $d=256$ and the linear regression model is well-specified with noise variance $\sigma^2=1$.
    \textsc{Tail average}: constant-stepsize SGD with tail averaging ($s = N/2$);
    \textsc{Exp decay}: SGD with geometrically decaying stepsize ($s=0$ and $K = \lceil N / \log (N) \rceil$);
    \textsc{Tail exp decay}: SGD with tail geometrically decaying stepsize ($s=N/2$ and $K = \lceil N / (2\log (N/2) )\rceil$);
    \textsc{Tail poly decay}: SGD with tail polynomially decaying stepsize ($s=N/2$ and $a=1$).
    We consider $6$ combinations of $2$ different covariance matrices and $3$ different true model parameters.
    For each algorithm and each sample size, we do a grid search and report the best excess risk achieved by \( \gamma_0 \in \{ 10^{-4}, 2\times 10^{-4}, 5\times 10^{-4}, 7\times 10^{-4}, 10^{-3}, 2\times 10^{-3}, 5 \times 10^{-3}, 0.01, 0.02, 0.03, 0.05, 0.075, 0.1, 0.2, 0.3, 0.5, 0.8, 1.0 \} \).
    The plots are averaged over $20$ independent runs.
    }
    \vspace{-0.3cm}
    \label{fig:excess-risk}
\end{figure*}

\section{Overview of the Proof Techniques}\label{sec:proof-sketch}
We now sketch the proof of Theorem \ref{thm:tail-decay-upper-bound} and highlight the key proof techniques. A complete proof is deferred to Appendix \ref{append-sec:upper-bound}.
For simplicity, let us denote $L := \log (N-s)$ and $K := (N-s) / L$, and assume they are integers.
\noindent\textbf{Bias-Variance Decomposition.}
We follow the well-known operator viewpoint for analyzing SGD iterates \citep{bach2013non,dieuleveut2017harder,jain2017markov,jain2017parallelizing,neu2018iterate,ge2019step,zou2021benign}.
In particular, the excess risk can be decomposed into \emph{bias error} and \emph{variance error} (see Lemma \ref{lemma:bias-var-decomp} in the appendix):
\begin{equation*}
    \Ebb [ L(\wB_N) - L(\wB^*) ] \le \langle \HB, \BB_N \rangle  + 
\langle \HB, \CB_N \rangle,
\end{equation*}
where $\BB_N$ and $\CB_N$ refer to the last \emph{bias iterate} and the last \emph{variance iterate} in the matrix space, respectively.
More precisely, they are recursively defined by\footnote{One can think of the bias iterates as SGD iterates on the data without additive label noise, and the variance iterates as SGD iterates with initialization $\wB^*$.}
\begin{align}
& \begin{cases}
\BB_t = (\Ical - \gamma_t \Tcal_t)\circ \BB_{t-1},\quad t\ge 1; \\
\BB_0 = (\wB_0 - \wB^*)(\wB_0 - \wB^*)^\top,
\end{cases} \label{eq:B-defi} \\
& \begin{cases}
\CB_t = (\Ical - \gamma_t \Tcal_t)\circ \CB_{t-1} + \gamma_t^2 \SigmaB, \quad t\ge 1; \\
\CB_0 = 0.
\end{cases}\label{eq:C-defi}
\end{align}
Here $\Ical := \IB \otimes \IB$, $\Mcal := \Ebb[\xB \otimes \xB \otimes \xB \otimes \xB]$ and $\Tcal_t := \HB \otimes \IB + \IB \otimes \HB - \gamma_t \Mcal$ are operators on symmetric matrices (see Appendix \ref{append:sec:preliminary} for their precise definitions). One can verify that for symmetric matrix $\AB$,
\[(\Ical - \gamma_t \Tcal_t) \circ \AB = \Ebb [ (\IB - \gamma_t \xB \xB^\top)\AB (\IB - \gamma_t \xB \xB^\top)]. \]
We next bound $\langle \HB, \BB_N \rangle$ and $
\langle \HB, \CB_N \rangle$ separately.

\subsection{Bias Upper Bound}
We first bound the bias error.
Recall that the stepsize scheme \eqref{eq:geometry-tail-decay-lr} splits the total $N$ iterations into $L = \log (N- s)$ fixed-stepsize phases: in the first phase, SGD is initialized from $\BB_0$, and runs with constant stepsize $\gamma_0$ for $s+ K$ steps; and in the $\ell$-th phase for $2 \le \ell \le L$, SGD is initialized from $\BB_{s+K(\ell-1)}$, and runs with stepsize $\gamma_0 / 2^{\ell-1}$ for $K$ steps.

\noindent\textbf{Main Challenges and Proof Techniques.}
The key difficulty here is to obtain a sharp characterization of \emph{each bias iterate} (i.e., $\BB_t$), instead of \emph{their summation} $\sum_{t=1}^N\BB_t$.
Therefore, existing techniques for SGD with averaging \citep{jain2017parallelizing,zou2021benign} are not sufficient.
In particular, \citet{zou2021benign} only gave a constant upper bound on the bias iterate (see Eq. (D.3) in their Lemma D.4, which is already sufficient for their purpose).
To obtain a tight and vanishing bound on each bias iterate, we need to carefully utilize the $(\Ical - \gamma \widetilde{T})^i$ decaying factor in the bias expansion (see \eqref{eq:B-single-phase-expanded}). 
Our proof is motivated by this idea and handle the decaying factor with an inequality $(1-\gamma x)^t\le 1/(\gamma x)$ (see \eqref{eq:HB-single-phase-expanded}). Based on this we get a (relatively loose) vanishing bound on each $\BB_t$. 
We further sharpen this upper bound with a multi-phase strategy: (1) splitting the entire bias iterates into multiple phases; (2) deriving an upper bound for each phase; and (3) carefully combining them to get the final result.
Details are explained  below.


\noindent\textbf{One Phase Analysis.}
We first investigate the decreasing effect of the bias error within one phase. 
For simplicity, with a slight abuse of notation, we use $\gamma$, $n$ and $\BB_t$ to denote the constant stepsize, the number of steps and the $t$-th bias iterate ($0\le t \le n$) within one phase, respectively. 
Assume that $\gamma < 1/(3\alpha\tr(\HB)\log n)$.
Clearly $\HB \otimes \HB \succeq 0$, therefore 
\begin{align}
\widetilde{\Tcal} & := \HB \otimes \IB + \IB \otimes \HB - \gamma \HB \otimes \HB  \notag \\
& = \Tcal + \gamma \Mcal - \gamma \HB \otimes \HB \le  \Tcal + \gamma \Mcal. \label{eq:tildeT-smaller-than-T}
\end{align} 
Plug \eqref{eq:tildeT-smaller-than-T} into \eqref{eq:B-defi}, and apply Assumption \ref{assump:fourth-moment}\ref{item:fourth-moement-upper}, we obtain:
\begin{align}
    \BB_t 
    \preceq (\Ical - \gamma \widetilde{\Tcal})\circ \BB_{t-1} + \alpha \gamma^2 \abracket{\HB, \BB_{t-1}} \HB,\quad t\ge 1. \notag 
\end{align}
Solving this recursion yields
\begin{equation}
    \BB_t 
    \preceq (\Ical - \gamma \widetilde{\Tcal})^t \circ \BB_{0} + \alpha \gamma^2\sum_{i=0}^{t-1} (\Ical - \gamma \widetilde{\Tcal})^{t-1-i} \circ \HB \cdot \abracket{\HB, \BB_i}. \label{eq:B-single-phase-expanded}
\end{equation}
In \eqref{eq:B-single-phase-expanded}, we apply
\[
(\Ical - \gamma \widetilde{\Tcal})^{t-1-i} \circ \HB
= (\IB - \gamma \HB)^{2(t-1-i)} \HB \preceq \frac{ \IB}{\gamma (t-i)}
\]
and take the inner product with $\HB$, so we have
\begin{equation}\label{eq:HB-single-phase-expanded}
    \abracket{\HB, \BB_t} \le \big\langle{(\Ical - \gamma \widetilde{\Tcal})^t \circ\HB, \BB_0}\big\rangle + \alpha\gamma \tr(\HB) \underbrace{\sum_{i=0}^{t-1} \frac{\abracket{\HB, \BB_i} }{t-i}}_{(\diamond)}.
\end{equation}
By recursively calling \eqref{eq:HB-single-phase-expanded}, one can observe that 
term $(\diamond)$
is \emph{self-governed} (this trick first appears in \citet{varre2021last} to our knowledge), which leads to the following upper bound (see Lemma \ref{lemma:HB-upper-bound} in the appendix):
\[
(\diamond)
\lesssim \Big\langle \sum_{i=0}^{t-1} \frac{ (\Ical - \gamma \widetilde{\Tcal})^i\circ \HB }{t-i},\ \BB_0 \Big\rangle.
\]
Substituting the above bound into \eqref{eq:HB-single-phase-expanded} leads to
\begin{align}
     \abracket{\HB, \BB_t} 
     & \lesssim \Big\langle (\Ical - \gamma \widetilde{\Tcal})^t \circ \HB + \sum_{i=0}^{t-1} \frac{ (\Ical - \gamma \widetilde{\Tcal})^i\circ \HB }{t-i}, \BB_0 \Big\rangle \notag \\
     & \lesssim  \Big\langle{\frac{1}{\gamma t}\IB_{0:k^*} + \HB_{k^*:\infty},\  \BB_0  }\Big\rangle,\label{eq:HB-single-phase-crude-bound}
\end{align}
where the second inequality holds by bounding the summation $\sum_{0\le i < t} (\cdot)$ separately for $\sum_{0 \le i < t/2} (\cdot)$ and $\sum_{t/2 \le i < t} (\cdot)$ (see Lemma \ref{lemma:HB-upper-bound} in the appendix for more details).
Here $k^*$ can be arbitrary.
From \eqref{eq:HB-single-phase-crude-bound}, we can see a decreasing effect of the bias error within one phase.


\noindent\textbf{Combining Multiple Phases.}
Next we discuss how to combine the decreasing effect of multiple phases. In this part, we use $\BB^{(\ell)}$ to denote the bias iterate output by the $\ell$-th phase (a.k.a., the input of the $(\ell+1)$-th phase).

In the first phase, a bound on $\BB^{(1)}$ is obtained by setting $k^* = k^\dagger$ and $\gamma = \gamma_0$ in \eqref{eq:HB-single-phase-crude-bound}, and substituting \eqref{eq:HB-single-phase-crude-bound} into \eqref{eq:B-single-phase-expanded} with $t = s+K$ (see Lemma \ref{lemma:B-ell-recursion} in the appendix):
\begin{equation}\label{eq:B-1-bound}
\begin{aligned}
    & \BB^{(1)}
     \lesssim (\Ical - \gamma \widetilde{\Tcal})^{s+K} \circ \BB_{0} +\gamma_0^2 (s+K) \log (s+K)  \cdot \\
    &   \Big\langle{\frac{\IB_{0:k^\dagger}}{\gamma_0 (s+K)} + \HB_{k^\dagger:\infty},  \BB_0  }\Big\rangle \cdot \Big( \frac{\IB_{0:k^\dagger}}{\gamma_0 (s+K)} + \HB_{k^\dagger:\infty} \Big).
\end{aligned}
\end{equation}
In the second phase, setting $ \gamma = \gamma_0 / 2 $ and $ t = K$ in \eqref{eq:HB-single-phase-crude-bound}, we obtain 
\begin{equation}\label{eq:HB-2-bound}
    \langle\HB, \BB^{(2)} \rangle \lesssim \Big\langle{\frac{1}{\gamma t}\IB_{0:k^*} + \HB_{k^*:\infty},\  \BB^{(1)}  }\Big\rangle.
\end{equation}
Plugging \eqref{eq:B-1-bound} into \eqref{eq:HB-2-bound} shows that $\BB^{(2)}$ already achieves the desired bias bound in Theorem \ref{thm:tail-decay-upper-bound}.
The remaining effort is to combine the effect from the third to the $L$-th phase, 
which leads to (see Lemma \ref{lemma:HB-ell-reduction} in the appendix): 
\begin{align*}
    \langle \HB, \BB^{(L)} \rangle \le e \cdot \langle \HB, \BB^{(2)} \rangle.
\end{align*}
This completes the proof for the bias error. 

\subsection{Variance Upper Bound}
\noindent\textbf{Main Challenges and Proof Techniques.}
Note that we are considering the variance error of the last SGD iterate, thus we cannot utilize the effect of iterate averaging to decrease the variance error \citep{bach2013non,jain2017markov,jain2017parallelizing,zou2021benign}.
Instead, to achieve a vanishing variance bound on the last iterate, we need to consider the effect of stepsize decaying. 
More details are provided below.


We first observe a uniform but crude upper bound on the variance iterates (see Lemma \ref{lemma:C-crude-bound} in the appendix, and also Lemma 5 in \citet{ge2019step}):
\begin{equation}\label{eq:C-crude-bound}
    \CB_t \preceq \frac{\gamma_0 \sigma^2}{1 - \alpha \gamma_0 \tr(\HB)} \IB, \quad t= 1,2,\dots, N.
\end{equation}
Then we will plug this crude bound on $\CB_t$ into \eqref{eq:C-defi} to further improve the upper bound of $\CB_t$ (see Theorem \ref{thm:HC-decay-lr-upper-bound} and its proof in the appendix):
\[
    \CB_N \preceq \frac{\sigma^2}{1-\gamma_0 R^2} \underbrace{\sum_{t=1}^N \gamma_t^2 \prod_{i=t+1}^N  (\IB-\gamma_i \HB)^2 \HB}_{(*)}.
\]
The remaining effort is to control term $(*)$.
Intuitively, though $(*)$ is a summation of $N$ terms, $(*)$ could vanish as $N$ increases thanks to the appropriate decaying stepsize scheme \eqref{eq:geometry-tail-decay-lr}:
for large $t$, the $t$-th term in the summation is small as $\gamma_t$ is small;
as for small $t$ where $\gamma_t$ is large, the $t$-th term in the summation is also small since the product $\prod_{i=t+1}^N  (\IB-\gamma_i \HB)^2 \HB$ is small (note the subsequent $\gamma_i$'s are at least $\gamma_t / 2$ according to \eqref{eq:geometry-tail-decay-lr}).
More precisely, our analysis (see Lemmas \ref{lemma:scalar-func-upper-bound} and \ref{lemma:scalar-func-lower-bound} in the appendix) shows that, ignoring constant factors,
\[(*) \eqsim  \frac{1}{K} \HB^\inv_{0:k^*} + \gamma_0 \IB_{k^*:k^\dagger} + \gamma_0^2 (s+K) \HB_{k^\dagger:\infty} \]
for the optimally chosen $k^*$ and $k^\dagger$ in Theorems \ref{thm:tail-decay-upper-bound} and \ref{thm:tail-decay-lower-bound}.
In this way, we can establish a tight upper bound on $\CB_N$.
Finally, taking inner product with $\HB$ yields the variance upper bound (see Theorem \ref{thm:HC-decay-lr-upper-bound} in the appendix).

\section{Concluding Remarks}
In this work, we provide a problem dependent excess risk bound for the last iterate of SGD with decaying stepsize for linear regression. The derived bound is dimension-free and can be applied to the overparamerized setting where the problem dimension excesses the sample size.
A nearly-matching, problem-dependent lower bound is also proved. 
We further compare the excess risk bounds of last iterate SGD with tail geometric-decaying stepsize and that with tail polynomial-decaying stepsize, and show that the former outperforms the latter, instance-wisely.  
We believe the developed theoretical framework can also be used to find better stepsize schemes, or even the optimal one, which is left as a future work.

\section*{Acknowledgements}
We would like to thank the anonymous reviewers and area chairs for their helpful comments. 
JW and VB are supported by the Defense Advanced Research Projects Agency (DARPA) under Contract No. HR00112190130. DZ acknowledges the support from Bloomberg Data Science Ph.D. Fellowship. QG is partially supported by the National Science Foundation award IIS-1906169 and IIS-2008981. 
SK acknowledges funding from the Office of Naval Research under award N00014-22-1-2377 and the National Science Foundation Grant under award CCF-1703574.
The views and conclusions contained in this paper are those of the authors and should not be interpreted as representing any funding agencies.

\bibliographystyle{icml2022}
\bibliography{ref}

\onecolumn
\appendix

\section{More Examples for Assumption \ref{assump:fourth-moment}}\label{append:sec:fourth-moment-assump-examples}

\begin{proposition}[Examples for Assumption \ref{assump:fourth-moment}\ref{item:fourth-moement-upper}]
Assumption \ref{assump:fourth-moment}\ref{item:fourth-moement-upper} holds for data distributions with a \emph{bounded kurtosis along every direction} \citep{dieuleveut2017harder}, i.e., there is a constant $\alpha > 0$ such that
\begin{equation*} 
    \text{for every } \vB,\ \Ebb [\abracket{\vB, \xB}^4] \le \alpha \abracket{\vB, \HB \vB}^2.
\end{equation*}
In particular, the above is satisfied when $\zB := \HB^{-\frac{1}{2}}\xB$ has sub-Gaussian or sub-exponential tail.
\end{proposition}
\begin{proof}
For a PSD matrix $\AB$, with eigenvalues $(\mu_i)_{i \ge 1}$ and eigenvectors $(\vB_i)_{i\ge 1}$, and a vector $\uB$, we have 
\begin{align*}
    \uB^\top \Ebb [ \xB \xB^\top \AB \xB \xB^\top ] \uB 
    &= \Ebb [ \xB^\top \AB \xB \cdot \langle \xB, \uB \rangle^2 ] \\
    &= \sum_{i} \mu_i \cdot  \Ebb [ \langle \xB, \vB_i \rangle^2  \cdot \langle \xB, \uB \rangle^2  ] \\
    &\le \sum_{i} \mu_i \cdot  \sqrt{\Ebb [ \langle \xB, \vB_i \rangle^4 ] \cdot \Ebb [\langle \xB, \uB \rangle^4  ]}\qquad (\text{by Cauchy-Schwarz inequality}) \\
    &\le \alpha \cdot \sum_{i} \mu_i \cdot   \langle \vB_i, \HB \vB_i \rangle  \cdot \langle \uB, \HB\uB \rangle  \qquad (\text{by bounded kurtosis condition}) \\
    &= \alpha \cdot \langle \AB, \HB \rangle \cdot \langle \uB, \HB\uB \rangle.
\end{align*}
Since the above holds for every vector $\uB$, we conclude that
\[
\Ebb [ \xB \xB^\top \AB \xB \xB^\top ] \preceq \alpha \cdot \langle \AB, \HB \rangle \cdot \HB,
\]
which proves Assumption \ref{assump:fourth-moment}\ref{item:fourth-moement-upper}.
\end{proof}

\begin{proposition}[Examples for Assumption \ref{assump:fourth-moment}\ref{item:fourth-moement-lower}]
Denote $\zB := \HB^{-\frac{1}{2}} \xB =: (z_1,\dots,z_d)^\top$. 
Then Assumption \ref{assump:fourth-moment}\ref{item:fourth-moement-lower} holds if:
\begin{enumerate}
    \item the distribution of $\zB$ is spherically symmetric, with a stochastic representation $\zB = r\cdot \uB$ where $r$ and $\uB$ are independent, $r>0$ and $\uB$ obeys the uniform distribution on the unit sphere $\Scal^{d-1}$;
\item $\Ebb[r^2] = d$ and $\Ebb[r^4] \ge \beta \cdot d(d+2)$ for a constant $\beta \ge 0.5$.
\end{enumerate}
\end{proposition}
\begin{proof}
We refer the reader to \citet{fang2018symmetric} for the moments calculation of spherically symmetric distributions.


Note that $\Ebb[\zB] = \zeroB$ and $\Ebb[\zB\zB^\top] = \IB$.
Let $\eB_1$ be the first standard basis, then for \emph{every} unit vector $\aB = (a_1,\dots,a_d)^\top$, 
    \begin{align*}
          \Ebb[( \eB_1^\top \zB )^2\cdot ( \aB^\top \zB )^2 ]
        &= \Ebb[z_1^2 \cdot ( \aB^\top \zB )^2 ] \\
        &= a_1^2 \cdot \Ebb[ z_1^4 ] + \sum_{j =2}^d a_j^2 \cdot \Ebb[z_1^2 z_j^2] \\ 
        &  = a_1^2 \cdot \Ebb[r^4] \cdot \frac{3}{d(d+2)}  + \sum_{j =2}^d a_j^2 \cdot  \Ebb[r^4] \cdot \frac{1}{d(d+2)}  \\
         & \ge 3\beta \cdot a_1^2 + \beta \cdot \sum_{j =2}^d a_j^2 \ge a_1^2 + \beta \\
         & = (\eB_1^\top \aB)^2 + \beta.
    \end{align*} 
    By the spherical symmetricity the above condition is equivalent to:
    \begin{align*}
        \text{for every unit vectors $\aB$ and $\bB$}, \quad \Ebb[( \aB^\top \zB )^2\cdot ( \bB^\top \zB )^2 ] \ge  (\aB^\top \bB)^2 + \beta. 
    \end{align*}
    Then for every PSD matrix $\AB$, with eigen decomposition $ \AB = \sum_{i=1}^d \mu_i \aB_i \aB_i^\top$, and every unit vector $\bB$, it holds that
    \begin{align*}
     \bB^\top \Ebb[  \zB \zB^\top \AB \zB \zB^\top  ] \bB 
     &= \sum_{i=1}^d \mu_i \cdot \Ebb[ (\aB_i^\top \zB)^2\cdot (\bB^\top \zB)^2 ]\\
     &\ge \sum_{i=1}^d \mu_i \cdot ( (\aB_i^\top \bB )^2 + \beta ) \\
     &= \bB^\top \AB \bB + \beta \cdot \tr(\AB) ,
    \end{align*}
    which implies that $\Ebb[ \zB \zB^\top \AB \zB \zB^\top ] \succeq \AB + \beta \cdot \tr(\AB) \cdot \IB$ for every PSD matrix $\AB$.
    Finally, applying $\xB = \HB^\frac{1}{2} \zB$ we obtain
    \begin{align*}
    \Ebb[\xB\xB^\top\AB\xB\xB^\top] 
    &= \HB^{\half}  \Ebb[\zB \zB^\top \HB^{\half} \AB \HB^{\half} \zB \zB^\top] \HB^{\half}  \\
    &\succeq \HB^{\half} \big( \HB^{\half}\AB\HB^{\half} + \beta \tr(\HB^{\half}\AB\HB^{\half}) \IB\big) \HB^{\half} \\
    &= \HB\AB\HB + \beta \tr(\AB\HB)\HB,
    \end{align*}
    which proves Assumption \ref{assump:fourth-moment}\ref{item:fourth-moement-lower}.
\end{proof}

\section{Preliminaries}\label{append:sec:preliminary}
For two matrices $\AB \in \Rbb^{d \times d}$ and $\BB \in \Rbb^{d\times d}$, their tensor product is defined by 
\begin{align*}
    \AB \otimes \BB := 
    \begin{pmatrix}
    a_{11} \BB & \dots & \aB_{1n} \BB \\
    \vdots & \ddots & \vdots \\
    a_{n1} \BB & \dots & a_{nn}\BB
    \end{pmatrix} \in \Rbb^{d^2 \times d^2},
\end{align*}
where $a_{ij}$ is the entry of $\AB$ in the $i$-th row and $j$-th column.
We can also understand $\AB \otimes \BB$ as a \emph{linear matrix operator}, in which case we write 
\[
(\AB \otimes \BB) \circ \CB := (\AB \otimes \BB)\cdot  \mathtt{vector}(\CB),
\]
where $\CB \in \Rbb^{d \times d}$ is a matrix and $\mathtt{vector}(\CB) \in \Rbb^{d^2}$ converts $\CB$ into a vector in the canonical manner.

\paragraph{Operators.}
We first summarize the linear operators (on symmetric matrices) to be used in the proof:
\begin{gather*}
    \Ical := \IB \otimes \IB,\qquad
    \Mcal := \Ebb [ (\xB \xB^\top) \otimes (\xB \xB^\top) ],\qquad
    \widetilde{\Mcal} = \HB \otimes \HB, \\
    \Tcal_t := \HB \otimes \IB + \IB \otimes \HB - \gamma_t \Mcal, \qquad
    \widetilde{\Tcal}_t = \HB \otimes \IB + \IB \otimes \HB - \gamma_t \HB\otimes\HB.
\end{gather*}
With a slight abuse of notations, we write $\Tcal_t$ (resp. $\widetilde\Tcal_t$) as $\Tcal$ (resp. $\widetilde\Tcal$) when the corresponding stepsize $\gamma_t$ in its definition is written as $\gamma$.
We use the notation $\Ocal \circ \AB$ to denotes the operator
$\Ocal$ acting on a symmetric matrix $\AB$.
One can verify the following rules for these operators acting on a symmetric matrix $\AB$ \citep{zou2021benign}:
\begin{equation}\label{eq:operators-on-matrix}
\begin{gathered}
\Ical \circ \AB = \AB, \qquad \Mcal\circ \AB = \Ebb [ (\xB^\top \AB \xB) \xB \xB^\top ], \qquad \widetilde{\Mcal} \circ \AB = \HB \AB \HB, \\     
(\Ical - \gamma \Tcal) \circ \AB = \Ebb [ (\IB - \gamma \xB \xB^\top)\AB (\IB - \gamma \xB \xB^\top) ], \quad  (\Ical-\gamma\widetilde{\Tcal})\circ\AB = (\IB-\gamma\HB)\AB(\IB-\gamma\HB).
\end{gathered}
\end{equation}

For the linear operators we have the following technical lemma from \citet{zou2021benign}.
\begin{lemma}[Lemma B.1, \citet{zou2021benign}]\label{lemma:operators}
An operator $\Ocal$ defined on symmetric matrices is called PSD mapping, if $\AB \succeq 0$ implies $\Ocal \circ \AB \succeq 0$.
Then we have
\begin{enumerate}
    \item $\Mcal$ and $\widetilde\Mcal$ are both PSD mappings.
    \item $\Ical-\gamma\Tcal$ and $\Ical-\gamma\widetilde\Tcal$ are both PSD mappings.
    \item $\Mcal - \widetilde\Mcal$ and $\widetilde \Tcal - \Tcal$ are both PSD mappings.
\item  If $0 < \gamma < 1/\lambda_1$, then $\widetilde{\Tcal}^{-1}$ exists, and is a PSD mapping.
\item  If $0 < \gamma < 1/(\alpha\tr(\HB))$, then $\Tcal^{-1}\circ \AB$ exists for PSD matrix $\AB$, and $\Tcal^{-1}$ is a PSD mapping. 
\end{enumerate}
\end{lemma}
\begin{proof}
See proof of Lemma B.1 in \citet{zou2021benign}.
\end{proof}

We then prove the bias-variance decomposition.
\begin{lemma}[Bias-variance decomposition]\label{lemma:bias-var-decomp}
Suppose Assumptions \ref{assump:second-moment} and \ref{assump:noise} hold.
Then the excess risk could be decomposed as
\[
\Ebb [ L(\wB_N) - L(\wB^*) ] 
\le \big\langle \HB, \BB_N \big\rangle  + 
\big\langle \HB, \CB_N \big\rangle. 
\]
\end{lemma}
\begin{proof}
The proof has appeared in prior works \citep{jain2017parallelizing,ge2019step}. For completeness, we provide a simplified (but not fully rigorous) proof here.
Consider the centered SGD iterates
\(
\etaB_t := \wB_t - \wB^*,
\)
where $\wB_t$ is given by \eqref{eq:sgd-iterates}, then the centered iterates are updated by 
\begin{equation*}
\etaB_t = (\IB - \gamma_t \xB_t \xB_t^\top) \etaB_{t-1} + \gamma_t \xi_t \xB_t, \quad t=1,2,\dots, N,
\end{equation*}
where $\xi_t := y_t - \langle \wB^*- \xB_t\rangle$ is the additive noise.
With a slight abuse of probability spaces, one can view the centered SGD iterates as the sum of two random processes,
\begin{equation}\label{eq:iter-decomp}
    \etaB_t = \etaB_t^\bias + \etaB_t^\var,\quad t=1,2,\dots,N,
\end{equation}
where
\[
\begin{cases}
 \etaB^\bias_t = (\IB - \gamma_t \xB_t \xB_t^\top) \etaB_{t-1}^\bias; \\ \etaB^\bias_0 = \wB_0 - \wB^*,
\end{cases}
\quad 
\begin{cases}
\etaB_t^\var = (\IB - \gamma_t \xB_t \xB_t^\top) \etaB_{t-1}^\var + \gamma_t \xi_t \xB_t; \\ 
\etaB_0^\var = 0.
\end{cases}
\]
Then one can verify that $\Ebb[ \etaB^\var_t ] = 0$, and moreover, 
\[
\BB_t = \Ebb [\etaB_t^\bias \otimes \etaB_t^\bias], \quad
\CB_t = \Ebb [\etaB_t^\var \otimes \etaB_t^\var],
\]
where $\BB_t$ and $\CB_t$ are defined in \eqref{eq:B-defi} and \ref{eq:C-defi}.
Finally, the lemma is proved by
\begin{align*}
    \Ebb [ L(\wB_N) - L(\wB^*) ] &= \half \abracket{\HB, \Ebb[\etaB_N \otimes \etaB_N]} \\
    &= \half \abracket{\HB, \Ebb[(\etaB_N^\bias + \etaB^\var_N) \otimes (\etaB_N^\bias + \etaB^\var_N)]} \\
    &\le \big\langle \HB, \Ebb[\etaB_N^\bias \otimes \etaB_N^\bias] \big\rangle  + 
\big\langle \HB, \Ebb[\etaB_N^\var \otimes \etaB_N^\var] \big\rangle \\
&= \big\langle \HB, \BB_N \big\rangle  + 
\big\langle \HB, \CB_N \big\rangle,
\end{align*}
where the inequality is because: for two vectors $\uB$ and $\vB$, 
\(
(\uB + \vB) (\uB + \vB)^\top  \preceq 2 ( \uB \uB^\top + \vB \vB^\top).
\)
\end{proof}
\begin{lemma}[Bias-variance decomposition, lower bound]\label{lemma:bias-var-decomp-lowerbonud}
Suppose Assumptions \ref{assump:second-moment} and \ref{assump:well-specified-noise} hold.
Then the excess risk could be decomposed as
\[
\Ebb [ L(\wB_N) - L(\wB^*) ] = \half \big\langle \HB, \Ebb [\etaB_N \otimes \etaB_N] \big\rangle 
= \half \big\langle \HB, \BB_N \big\rangle  + 
\half \big\langle \HB, \CB_N \big\rangle. 
\]
\end{lemma}
\begin{proof}
The first equality is clear from definitions.
The second equality is due to \eqref{eq:iter-decomp} and the following fact (by Assumption \ref{assump:well-specified-noise}):
\(
\Ebb [\etaB^\var_t | \etaB^\bias_t ] = 0.
\)
\end{proof}

\section{Proof of Upper Bound}\label{append-sec:upper-bound}

\subsection{Variance Upper Bound}
In this part we replace Assumption \ref{assump:fourth-moment}\ref{item:fourth-moement-lower} with the following relaxed Assumption \ref{assump:fourth-moment-R}. It is clear that Assumption \ref{assump:fourth-moment}\ref{item:fourth-moement-lower} implies Assumption \ref{assump:fourth-moment-R} with $R^2 = \alpha \tr(\HB)$.

\begin{taggedassumption}{\ref{assump:fourth-moment}'}[Fourth moment condition, relaxed version]\label{assump:fourth-moment-R}
There exists a constant $R>0$ such that $\Ebb[\xB \xB^\top \xB \xB^\top] \preceq R^2 \HB$.
\end{taggedassumption}

The following lemma is from \citet{ge2019step}.

\begin{lemma}[Lemma 5 in \citet{ge2019step}]\label{lemma:C-crude-bound}
Suppose Assumptions \ref{assump:second-moment}, \ref{assump:fourth-moment-R} and \ref{assump:noise} hold.
Consider \eqref{eq:C-defi}.
Suppose $\gamma_0 < 1/R^2$.
Then for every $t$ we have
\[\CB_t \le \frac{\gamma_0 \sigma^2}{1 - \gamma_0 R^2} \IB.\]
\end{lemma}
\begin{proof}
The original proof has appeared in \citet{ge2019step} and \citet{jain2017markov}. We present a proof here for completeness.
We proceed with induction. For $t=0$ we have $\CB_0 = 0 \preceq \frac{\gamma_0 \sigma^2}{1-\gamma_0 R^2} \IB$. We then assume that 
$\CB_{t-1} \preceq \frac{\gamma_0 \sigma^2}{1-\gamma_0 R^2} \IB$, and exam $\CB_t$ based on \eqref{eq:C-defi}:
\begin{align}
    \CB_t 
    &= (\Ical - \gamma_t {\Tcal_t})\circ \CB_{t-1} + \gamma_t^2 \SigmaB \notag \\
    &= (\Ical - \gamma_t \HB \otimes\IB - \gamma_t \IB \otimes \HB )\circ \CB_{t-1} + \gamma_t^2 \Mcal \circ \CB_{t-1} + \gamma_t^2 \SigmaB \notag \\
    &\preceq \frac{\gamma_0 \sigma^2}{1-\gamma_0 R^2}\cdot (\IB - 2 \gamma_t \HB) + \gamma_t^2 \cdot \frac{\gamma_0 \sigma^2}{1-\gamma_0 R^2}\cdot R^2 \HB + \gamma_t^2 \sigma^2 \HB \notag \\
    &=   \frac{\gamma_0 \sigma^2}{1-\gamma_0 R^2} \cdot\IB - (2\gamma_t\gamma_0 - \gamma_t^2)\cdot \frac{\sigma^2}{1-\gamma_0 R^2} \HB \notag \\
    &\preceq \frac{\gamma_0 \sigma^2}{1-\gamma_0 R^2} \cdot\IB. \notag
\end{align}
This completes the induction.
\end{proof}

\begin{theorem}[A variance bound]\label{thm:HC-decay-lr-upper-bound}
Suppose Assumptions \ref{assump:second-moment}, \ref{assump:fourth-moment-R} and \ref{assump:noise} hold.
Consider \eqref{eq:C-defi}.
Let $K = (N-s) / \log (N-s)$.
Suppose $s\ge 0$, $K \ge 1$ and $\gamma_0 < 1/R^2$.  We have
\begin{equation*}
    \abracket{\HB, \CB_N} \le \frac{8\sigma^2}{1-\gamma_0 R^2}\Bigg(\frac{k^*}{K} + \gamma_0 \sum_{k^* < i \le k^\dagger}\lambda_i + \gamma_0^2 (s+K) \sum_{i>k^\dagger}\lambda_i^2\Bigg),
\end{equation*}
where $k^*$ and $k^\dagger$ can be arbitrary.
\end{theorem}
\begin{proof}
From \eqref{eq:C-defi} we have
\begin{align}
    \CB_t 
    &\preceq (\Ical - \gamma_t \widetilde{\Tcal}_t)\circ \CB_{t-1} + \gamma_t^2 \Mcal \circ \CB_{t-1} + \gamma_t^2 \sigma^2 \HB \notag \\
    &\preceq (\Ical - \gamma_t \widetilde{\Tcal}_t)\circ \CB_{t-1} + \gamma_t^2 \cdot \frac{\gamma_0 \sigma^2}{1-\gamma_0 R^2}\cdot R^2 \HB + \gamma_t^2 \sigma^2 \HB \qquad (\text{use Lemma \ref{lemma:C-crude-bound}}) \notag \\
    &=  (\Ical - \gamma_t \widetilde{\Tcal}_t)\circ \CB_{t-1} +  \frac{\gamma_t^2 \sigma^2}{1-\gamma_0 R^2} \HB. \notag
\end{align}
Solving the recursion yields 
\begin{align}
    \CB_N 
    &\preceq \frac{\sigma^2}{1-\gamma_0 R^2} \sum_{t=1}^N \gamma_t^2 \prod_{i=t+1}^N  (\Ical-\gamma_i \widetilde{\Tcal}_i)\circ \HB \notag \\
    &= \frac{\sigma^2}{1-\gamma_0 R^2} \sum_{t=1}^N \gamma_t^2 \prod_{i=t+1}^N  (\IB-\gamma_i \HB)^2 \HB \notag \\
    &\preceq \frac{\sigma^2}{1-\gamma_0 R^2} \underbrace{\sum_{t=1}^N \gamma_t^2 \prod_{i=t+1}^N  (\IB-\gamma_i \HB) \HB}_{(*)}. \label{eq:C-upper-bound-expanded}
\end{align}
Now recalling \eqref{eq:geometry-tail-decay-lr}, we have 
\begin{align}
    (*) 
    &= \gamma_0^2 \sum_{i=1}^{s+K} \Big( \IB - \gamma_0 \HB \Big)^{s+K - i} \prod_{j=1}^{L-1} \Big(\IB-\frac{\gamma_0}{2^j}\HB \Big)^{K} \HB \notag \\
    &\qquad + \sum_{\ell=1}^{L - 1} \Big(\frac{\gamma_0}{2^\ell}\Big)^2 \sum_{i=1}^{K} \Big( \IB - \frac{\gamma_0}{2^\ell}\HB \Big)^{K - i} \prod_{j=\ell+1}^{L-1} \Big(\IB-\frac{\gamma_0}{2^j}\HB \Big)^{K} \HB \notag \\
    &= \gamma_0 \bigg( \IB - \Big( \IB - \gamma_0 \HB \Big)^{s+K} \bigg) \prod_{j=1}^{L-1} \Big(\IB-\frac{\gamma_0}{2^j}\HB \Big)^{K} \notag \\
    &\qquad +\sum_{\ell=1}^{L - 1} \frac{\gamma_0}{2^\ell} \left( \IB - \Big( \IB - \frac{\gamma_0}{2^\ell}\HB \Big)^{K} \right) \prod_{j=\ell+1}^{L-1} \Big(\IB-\frac{\gamma_0}{2^j}\HB \Big)^{K}, \label{eq:C-upper-bound-expanded-part}
\end{align}
where we understand $\prod_{j=L}^{L-1}(\cdot) = 1$.
Define a scalar function
\[
f(x) := x\cdot \Big(1-\big(1-x\big)^{s+K}\Big) \cdot \prod_{j=1}^{L-1} \big(1-\frac{x}{2^j}\big)^K + \sum_{\ell=1}^{L-1} \frac{x}{2^\ell} \cdot \Big(1-\big(1-\frac{x}{2^{\ell}}\big)^K\Big) \cdot\prod_{j=\ell+1}^{L-1} \big(1-\frac{x}{2^{j}}\big)^K, 
\]
then applying $f(\cdot)$ to $\gamma_0 \HB$ in each diagonal entry, and using Lemma \ref{lemma:scalar-func-upper-bound}, we have
\[
f(\gamma_0 \HB) \preceq  \frac{8}{K} \IB_{0:k^*} + 2 \gamma_0 \HB_{k^*:k^\dagger}+  {2\gamma_0^2 (s+K)} \HB^2_{k^\dagger:\infty}, 
\]
for arbitrary $k^*$ and $k^\dagger$.
Now using \eqref{eq:C-upper-bound-expanded} and \eqref{eq:C-upper-bound-expanded-part}, we obtain
\[
\CB_{N} \preceq \frac{\sigma^2}{1-\gamma_0 R^2} \cdot f(\gamma_0 \HB)\cdot  \HB^\inv
\preceq \frac{8\sigma^2}{1-\gamma_0 R^2} \left( \frac{1}{K} \HB^\inv_{0:k^*} + \gamma_0 \IB_{k^*:k^\dagger} + \gamma_0^2 (s+K) \HB_{k^*:\infty} \right),
\]
and consequently, 
\[
\abracket{\HB,\CB_{N} } \le \frac{8\sigma^2}{1-\gamma_0 R^2}\left(\frac{k^*}{K} + \gamma_0 \sum_{k^* < i \le k^\dagger}\lambda_i + \gamma_0^2 (s+K) \sum_{i>k^\dagger}\lambda_i^2\right),
\]
where $k^*$ and $k^\dagger$ can be arbitrary.

\end{proof}

\begin{lemma}\label{lemma:scalar-func-upper-bound}
Suppose $s\ge 0, K \ge 1$ and $x\in(0,1]$.
For the scalar function 
\[
f(x) := x\cdot \Big(1-\big(1-x\big)^{s+K}\Big) \cdot \prod_{j=1}^{L-1} \big(1-\frac{x}{2^j}\big)^K + \sum_{\ell=1}^{L-1} \frac{x}{2^\ell} \cdot \Big(1-\big(1-\frac{x}{2^{\ell}}\big)^K\Big) \cdot\prod_{j=\ell+1}^{L-1} \big(1-\frac{x}{2^{j}}\big)^K,
\]
we have
\[ f(x) \le \min\cbracket{2(s+K)x^2,\ 2x,\ \frac{8}{K}}. \]
\end{lemma}
\begin{proof}
We show each upper bound separately.
\begin{itemize}
    \item 
For $x\in(0,1]$, we have
\( (1-x)^{s+K}\ge 1- (s+K)x\) and \( (1-x)^K \ge 1- Kx\), 
which lead to
\begin{equation*}
f(x) \le x\cdot (s+K)x\cdot 1 + \sum_{\ell=1}^{L-1} \frac{x}{2^\ell} \cdot \frac{K x}{2^\ell}\cdot 1
\le 2(s+K) x^2. 
\end{equation*}

\item Clearly, for $x\in(0,1]$ we have:
\(
    f(x) \le x \cdot 1 \cdot 1 + \sum_{\ell=1}^{L-1} \frac{x}{2^\ell} \cdot 1 \cdot 1
    \le 2x.\)

\item For $x\in (0, 2/K)$, by the previous bound we have 
\(
f(x) \le 2x \le 4/K.
\)

As for $x\in [2 / K, 1]$, there is an 
\[\ell^* := \lfloor \log (Kx) \rfloor - 1  \in [0,\ L-1 ), \] 
such that 
    \[ {2^{\ell^* + 1}}/{K} \le x < {2^{\ell^*+2}}/{K}.\]
    by which and the definition of $f(x)$ we obtain:
    \begin{align}
        f(x) &\le x \cdot 1 \cdot \prod_{j=1}^{L-1} \big(1-\frac{x}{2^j}\big)^K + \sum_{\ell=1}^{L-1} \frac{x}{2^\ell} \cdot 1 \cdot \prod_{j=\ell+1}^{L-1} \big(1-\frac{x}{2^{j}}\big)^K \notag \\
        &= \sum_{\ell=0}^{\ell^* } \frac{x}{2^\ell}  \cdot \prod_{j=\ell+1}^{L-1} \big(1-\frac{x}{2^{j}}\big)^K + \sum_{\ell=\ell^* + 1}^{L-1} \frac{x}{2^\ell} \cdot \prod_{j=\ell+1}^{L-1} \big(1-\frac{x}{2^{j}}\big)^K \notag\\
        &\le \sum_{\ell=0}^{\ell^* } \frac{x}{2^\ell}  \cdot \big(1-\frac{x}{2^{\ell+1}}\big)^K + \sum_{\ell=\ell^*  +1}^{L-1} \frac{x}{2^\ell} \cdot 1 \notag\\
        &\le \sum_{\ell=0}^{\ell^*} \frac{2^{\ell^*-\ell+2}}{K}  \cdot \Big(1-\frac{2^{\ell^*-\ell}}{K}\Big)^K + \sum_{\ell=\ell^* + 1}^{L-1} \frac{2^{\ell^*-\ell+2}}{K} \cdot 1 \notag\\
        &\le \frac{4}{K} \cdot \sum_{\ell=0}^{\ell^*} 2^{\ell^*-\ell}\cdot e^{-2^{\ell^*-\ell}} + \frac{4}{K} \notag\\
        &\le \frac{4}{K}\cdot 1  +\frac{4}{K} 
        = \frac{8}{K}.  \notag
    \end{align}
    In sum we have $f(x) \le 8/K$ holds for every $x \in (0,1]$.
\end{itemize}

\end{proof}

\subsection{Preparation: Bias Upper Bound in a Single Phase}\label{sec:bias-bound-single-phase}
In this section we consider running bias iterates with constant stepsize $\gamma$ for $n$ steps. We note this process corresponds to SGD in one phase with constant stepsize. For simplicity we denote the initial bias iterate as $\BB_0$.
Then the bias iterates are updated according to
\begin{equation}
    \BB_t = (\Ical - \gamma \Tcal)\circ \BB_{t-1},\quad
    t=1,2,\dots,n. \label{eq:bias-iter-const-stepsize}
\end{equation}
For simplicity, let us define 
\begin{gather}
    \widehat{\HB}_{t} := \frac{1}{\gamma t} \IB_{0:k^*} + \HB_{k^*:\infty},\quad t\ge 1,
\end{gather}
where $k^* \ge 0$ could be any integer.

From \eqref{eq:bias-iter-const-stepsize} we have 
\begin{align}
    \BB_t &\preceq (\Ical - \gamma \widetilde{\Tcal})\circ \BB_{t-1} + \gamma^2 \Mcal \circ \BB_{t-1} \notag\\
    &\preceq (\Ical - \gamma \widetilde{\Tcal})^t \circ \BB_{0} + \gamma^2\sum_{i=0}^{t-1} (\Ical - \gamma \widetilde{\Tcal})^{t-1-i} \circ \Mcal \circ \BB_{i} \notag\\
    &\preceq (\Ical - \gamma \widetilde{\Tcal})^t \circ \BB_{0} + \alpha \gamma^2\sum_{i=0}^{t-1} (\Ical - \gamma \widetilde{\Tcal})^{t-1-i} \circ \HB \cdot \abracket{\HB, \BB_i} \notag \\
    &= (\Ical - \gamma \widetilde{\Tcal})^t \circ \BB_{0} + \alpha \gamma^2\sum_{i=0}^{t-1} (\IB - \gamma \HB)^{2(t-1-i)} \HB \cdot \abracket{\HB, \BB_i}\label{eq:bias-iter-expanded},
\end{align}
where the inequality also holds for $t=0$ with the understanding that $\sum_{i=0}^{-1}\cdot = 0$.

The following lemma provides a crude upper bound on $\abracket{\HB,\BB_n}$.
\begin{lemma}\label{lemma:HB-upper-bound}
Suppose Assumptions \ref{assump:second-moment} and \ref{assump:fourth-moment} hold.
Consider \eqref{eq:bias-iter-const-stepsize}.
Suppose $n\ge 1$ and $\gamma < 1/(2\alpha\tr(\HB)\log n)$. We have
    \begin{align*}
     \abracket{\HB, \BB_n}
    \le \frac{2}{1-2\alpha \gamma \tr (\HB) \log n } \cdot  \abracket{\frac{1}{\gamma n}\IB_{0:k^*} + \HB_{k^*:\infty}, \BB_0  },
    \end{align*}
    where $k^*$ can be arbitrary.
\end{lemma}
\begin{proof}
Notice $(1-x)^t \le 1/((t+1)x)$ for $x\in (0,1)$, then
\( (\IB - \gamma \HB)^{2t} \HB
\preceq \frac{1}{\gamma (t+1)} \IB.
\)
Inserting this into \eqref{eq:bias-iter-expanded} and setting $t=n$, we obtain
\begin{align}
    \BB_t \preceq (\Ical - \gamma \widetilde{\Tcal})^t \circ \BB_{0} + \alpha\gamma \sum_{i=0}^{t-1} \frac{\abracket{\HB, \BB_t} }{t-i}\cdot \IB,\quad t\ge 1, \label{eq:B-expanded-first-order-bound}
\end{align}
and thus
\begin{align}
    \abracket{\HB, \BB_t} 
    &\le \abracket{(\Ical - \gamma \widetilde{\Tcal})^t \circ\HB, \BB_0} + \alpha\gamma \tr(\HB) \sum_{i=0}^{t-1} \frac{\abracket{\HB, \BB_t} }{t-i} \notag\\
    &= \abracket{(\IB-\gamma\HB)^{2t}\HB, \BB_0} + \alpha\gamma \tr(\HB) \sum_{i=0}^{t-1} \frac{\abracket{\HB, \BB_t} }{t-i},\quad n\ge 1. \label{eq:HB-expanded-first-order-bound}
\end{align}
Recursively applying \eqref{eq:HB-expanded-first-order-bound} to each $\abracket{\HB, \BB_t}$, we obtain 
\begin{align*}
    \sum_{t=0}^{n-1} \frac{\abracket{\HB, \BB_t} }{n-t}
    &\le \bigg\langle{ \sum_{t=0}^{n-1} \frac{(\IB-\gamma\HB)^{2t}\HB}{n-t}, \BB_0 }\bigg\rangle + \alpha \gamma \tr (\HB) \sum_{t=0}^{n-1} \sum_{i=0}^{t-1}\frac{\abracket{\HB, \BB_i} }{(n-t)(t-i)} \\
    &= \bigg\langle{ \sum_{t=0}^{n-1} \frac{(\IB-\gamma\HB)^{2t}\HB}{n-t}, \BB_0 }\bigg\rangle + \alpha \gamma \tr (\HB) \sum_{i=0}^{n-2} \frac{\abracket{\HB, \BB_i} }{n-i} \sum_{t=i+1}^{n-1}\bracket{\frac{1}{n-t} + \frac{1}{t-i}} \\
    &\le \bigg\langle{ \sum_{t=0}^{n-1} \frac{(\IB-\gamma\HB)^{2t}\HB}{n-t}, \BB_0 }\bigg\rangle + 2\alpha \gamma \tr (\HB) \log n \cdot \sum_{i=0}^{n-1} \frac{\abracket{\HB, \BB_i} }{n-i},
\end{align*}
which implies that for $\gamma < 1/(2\alpha\tr(\HB)\log n)$, we have 
\begin{align}
    \sum_{t=0}^{n-1} \frac{\abracket{\HB, \BB_t} }{n-t}
    &\le \frac{1}{1-2\alpha \gamma \tr (\HB) \log n } \cdot  \bigg\langle{ \underbrace{\sum_{t=0}^{n-1} \frac{(\IB-\gamma\HB)^{2t}\HB}{n-t}}_{(*)},\ \BB_0 \bigg\rangle}. \label{eq:HB-hybrid-sum}
\end{align}
We would like to acknowledge \citet{varre2021last}, from where we learn the trick to reach \eqref{eq:HB-hybrid-sum}.
Furthermore, we can bound $(*)$ as follows:
\begin{align}
    (*)
    &= \sum_{t=0}^{n/2-1} \frac{(\IB-\gamma\HB)^{2t}\HB}{n-t} + \sum_{t=n/2}^{n-1} \frac{(\IB-\gamma\HB)^{2t}\HB}{n-t} \notag \\
    &\le \frac{2}{n} \sum_{t=0}^{n/2-1} {(\IB-\gamma\HB)^{2t}\HB} + (\IB-\gamma\HB)^{n}\HB \sum_{t=n/2}^{n-1} \frac{1}{n-t} \notag \\
    &\le 2\cdot \frac{\IB - (\IB-\gamma\HB)^{n}}{\gamma n} + \log n \cdot (\IB-\gamma\HB)^{n}\HB \notag \\
    &\le 2\log n \cdot \bracket{ \frac{\IB - (\IB-\gamma\HB)^{n}}{\gamma n} +  (\IB-\gamma\HB)^{n}\HB }.\label{eq:H-hybrid-sum}
\end{align}
Finally, inserting \eqref{eq:HB-hybrid-sum} and \eqref{eq:H-hybrid-sum} into \eqref{eq:HB-expanded-first-order-bound}, we obtain
\begin{align*}
    \abracket{\HB, \BB_n} 
    &\le \abracket{(\IB-\gamma\HB)^{2n}\HB, \BB_0}  + \frac{2\alpha \gamma \tr(\HB)\log n}{1-2\alpha \gamma \tr (\HB) \log n }  \bigg\langle{ \frac{\IB - (\IB-\gamma\HB)^{n}}{\gamma n} +  (\IB-\gamma\HB)^{n}\HB, \BB_0 }\bigg\rangle  \\
    &\le \frac{1}{1-2\alpha \gamma \tr (\HB) \log n } \cdot  \abracket{ \frac{\IB - (\IB-\gamma\HB)^n}{\gamma n} +  (\IB-\gamma\HB)^{n}\HB, \BB_0  }  \\
    &\le \frac{2}{1-2\alpha \gamma \tr (\HB) \log n } \cdot  \abracket{\frac{1}{\gamma n}\IB_{0:k^*} + \HB_{k^*:\infty}, \BB_0  },
\end{align*}
where the last inequality is because
\begin{equation}\label{eq:H-split}
\frac{\IB - (\IB-\gamma\HB)^n}{\gamma n} \preceq \frac{1}{\gamma n}\IB_{0:k^*} + \HB_{k^*:\infty},\qquad
(\IB-\gamma\HB)^{n}\HB \preceq \frac{1}{\gamma n}\IB_{0:k^*} + \HB_{k^*:\infty}.
\end{equation}
\end{proof}
The following lemma provides an upper bound for $\BB_n$.
\begin{lemma}\label{lemma:B-upper-bound}
Suppose Assumptions \ref{assump:second-moment} and \ref{assump:fourth-moment} hold.
Consider \eqref{eq:bias-iter-const-stepsize}.
Suppose $n \ge 2$ and $\gamma < 1/(2\alpha\tr(\HB)\log n)$. We have
\[
 \BB_n 
    \preceq (\IB - \gamma \HB)^n \cdot \BB_{0}\cdot (\IB - \gamma \HB)^n + \frac{3\alpha \gamma^2 n\log n}{1-2\alpha\gamma\tr(\HB)\log n}  \cdot \big\langle{\widehat{\HB}_{n}, \BB_0}\big\rangle \cdot \widehat{\HB}_{n}.
\]
\end{lemma}
\begin{proof}
We bring Lemma \ref{lemma:HB-upper-bound} into \eqref{eq:bias-iter-expanded} to obtain 
\begin{align}
 \BB_n 
 &\preceq (\Ical - \gamma \widetilde{\Tcal})^n \circ \BB_{0} + \alpha \gamma^2 (\IB - \gamma \HB)^{2(n-1)} \HB \cdot \abracket{\HB, \BB_0} \notag \\
 &\quad + \alpha \gamma^2\sum_{t=1}^{n-1} (\IB - \gamma \HB)^{2(n-1-t)} \HB \cdot \abracket{\HB, \BB_t} \notag \\
    &\preceq (\Ical - \gamma \widetilde{\Tcal})^n \circ \BB_{0} + \alpha \gamma^2 \underbrace{(\IB - \gamma \HB)^{2(n-1)} \HB \cdot \abracket{\HB, \BB_0}}_{(*)} \notag \\
    &\qquad + \frac{2\alpha \gamma^2}{1-2\alpha\gamma\tr(\HB)\log n}\cdot \underbrace{\sum_{t=1}^{n-1} (\IB - \gamma \HB)^{2(n-1-t)} \HB \cdot \big\langle{\widehat{\HB}_{t}, \BB_0}\big\rangle}_{(**)}, \label{eq:B-upper-bound-0}
\end{align}
where 
\[
\widehat{\HB}_{t} := \frac{1}{\gamma t} \IB_{0:k^*} + \HB_{k^*:\infty}, \quad t \ge 1.
\]
For term $(**)$, we bound it by
\begin{align*}
(**)
&= \sum_{t=1}^{n/2-1} (\IB - \gamma \HB)^{2(n-1-t)} \HB \cdot \big\langle{\widehat{\HB}_{t}, \BB_0}\big\rangle  + \sum_{t=n/2}^{n-1} (\IB - \gamma \HB)^{2(n-1-t)} \HB \cdot \big\langle{\widehat{\HB}_{t}, \BB_0}\big\rangle \\
&\preceq \sum_{t=1}^{n/2-1} (\IB - \gamma \HB)^{n} \HB \cdot \big\langle{\widehat{\HB}_{t}, \BB_0}\big\rangle  + \sum_{t=n/2}^{n-1} (\IB - \gamma \HB)^{n-1-t} \HB \cdot \big\langle{\widehat{\HB}_{n/2}, \BB_0}\big\rangle \\
&=   (\IB - \gamma \HB)^{n} \HB \cdot \Big\langle{\sum_{t=1}^{n/2-1}\widehat{\HB}_{t}, \BB_0}\Big\rangle  +  \frac{  \IB - (\IB - \gamma \HB)^{n/2}}{\gamma} \cdot \big\langle{\widehat{\HB}_{n/2}, \BB_0}\big\rangle \\
&\preceq (\IB - \gamma \HB)^{n} \HB \cdot \big\langle{n (\log n-1) \cdot \widehat{\HB}_{n}, \BB_0}\big\rangle + 2 \cdot \frac{ \IB - (\IB - \gamma \HB)^n}{\gamma} \cdot \big\langle{\widehat{\HB}_{n}, \BB_0}\big\rangle \\
&\preceq n(\log n-1) \cdot\widehat{\HB}_{n}\cdot  \big\langle{\widehat{\HB}_{n}, \BB_0}\big\rangle + 2 \cdot\widehat{\HB}_{n}\cdot  \big\langle{\widehat{\HB}_{n}, \BB_0}\big\rangle\qquad (\text{use \eqref{eq:H-split}}) \notag\\
&= (n\log n-n+2) \cdot \big\langle{\widehat{\HB}_{n}, \BB_0}\big\rangle \cdot \widehat{\HB}_{n}. 
\end{align*}
In order to bound $(*)$, notice that for $n \ge 2$,
\[ (\IB - \gamma \HB)^{2(n-1)} \HB \preceq \frac{1}{2\gamma (n-1)} \IB_{0:k^*} + \HB_{k^*:\infty} \preceq \widehat{\HB}_n,\]
then we have
\begin{align*}
    (*) \preceq \widehat{\HB}_n \cdot \abracket{\HB, \BB_0} \preceq n \cdot \widehat{\HB}_n \cdot \big\langle{\widehat{\HB}_n, \BB_0}\big\rangle,
\end{align*}
where the last inequality is because $\gamma < 1/ \tr(\HB)$ implies $\lambda_1, \dots, \lambda_{k^*} < 1/ \gamma$ for every $k^*$.

Finally, bring the bounds on $(*)$ and $(**)$ into \eqref{eq:B-upper-bound-0}, we obtain 
\begin{align*}
    \BB_n &\preceq (\Ical - \gamma \widetilde{\Tcal})^n \circ \BB_0 + \alpha \gamma^2 \cdot n \cdot \widehat{\HB}_n \cdot \big\langle{\widehat{\HB}_n, \BB_0}\big\rangle \\
    &\quad + \frac{2\alpha \gamma^2}{1-2\alpha\gamma\tr(\HB)\log n}\cdot (n\log n-n+2) \cdot \big\langle{\widehat{\HB}_{n}, \BB_0}\big\rangle \cdot \widehat{\HB}_{n} \\
    &\preceq (\Ical - \gamma \widetilde{\Tcal})^n \circ \BB_0 + \frac{ 3 \alpha \gamma^2n\log n}{1-2\alpha\gamma\tr(\HB)\log n} \cdot \big\langle{\widehat{\HB}_{n}, \BB_0}\big\rangle \cdot \widehat{\HB}_{n}.
\end{align*}
Applying the definition of $\widetilde{\Tcal}$ completes the proof.

\end{proof}

\subsection{Bias Upper Bound}\label{append-sec:multi-phase}

Let us denote the bias iterate at the end of each stepsize-decaying phase by
\begin{equation}\label{eq:B-iter-phase}
\BB^{(\ell)} :=
\begin{cases}
\BB_0, & \ell = 0; \\
\BB_{s + K * \ell }, & \ell = 1,2,\dots, L.
\end{cases}
\end{equation}
According to \eqref{eq:geometry-tail-decay-lr} and the above definition, we can interpret the SGD iterates \eqref{eq:B-defi} as follows:
in phase $\ell=1$, SGD is initialized from $\BB_0$ and runs for $s+K$ steps with constant stepsize $\gamma^{(1)} := \gamma$, and output $\BB^{(1)}$; 
in phase $\ell \ge 2$, SGD is initialized from $\BB^{(\ell-1)}$ and runs for $K$ steps with constant stepsize 
\[
\gamma^{(\ell-1)} := \frac{\gamma_0}{2^{\ell-1}},
\]
and output $\BB^{(\ell)}$; the final output is $\BB^{(L)} = \BB_N$.

We now build an upper bound for bias error based on results obtained in Section \ref{sec:bias-bound-single-phase}.

\begin{lemma}\label{lemma:HB-ell-recursion}
Suppose Assumptions \ref{assump:second-moment} and \ref{assump:fourth-moment} hold.
Consider \eqref{eq:B-iter-phase} and \eqref{eq:B-defi}.
Suppose $\gamma_0 < 1/(3\alpha\tr(\HB)\log (s+K))$. We have 
\begin{align*}
\text{for}\ \ \ell  = 1, \ \  \big\langle\HB, \BB^{(1)} \big\rangle
&\le 6\cdot \Big\langle \frac{1}{\gamma_0 (s+K)} \IB_{0:k^*} + \HB_{k^*:\infty},\ \BB_0 \Big\rangle; \\
\text{for}\ \ \ell  \ge 2, \ \  \big\langle\HB, \BB^{(\ell)} \big\rangle
&\le 6\cdot \Big\langle{\frac{1}{\gamma^{(\ell-1)}K}\IB_{0:k^*} + \HB_{k^*:\infty},\ \BB^{(\ell-1)} }\Big\rangle,
\end{align*}
where $k^*$ can be arbitrary.
\end{lemma}
\begin{proof}
For $\big\langle\HB, \BB^{(1)} \big\rangle$, we apply Lemma \ref{lemma:HB-upper-bound} with $\gamma \to \gamma_0$ and $n \to s+K$, and use the condition that $\alpha \gamma\tr(\HB) \log (s+K) \le 1/3$.

For $\big\langle\HB, \BB^{(\ell)} \big\rangle$ with $\ell \ge 2$, we apply Lemma \ref{lemma:HB-upper-bound} with $\gamma \to \gamma^{(\ell-1)}$, $n \to K$ and $\BB_0 \to \BB^{(\ell-1)}$, and use the condition that $\alpha \gamma\tr(\HB) \log (K) \le 1/3$.

\end{proof}

\begin{lemma}\label{lemma:B-ell-recursion}
Suppose Assumptions \ref{assump:second-moment} and \ref{assump:fourth-moment} hold.
Consider \eqref{eq:B-iter-phase} and \eqref{eq:B-defi}.
Suppose $\gamma_0 < 1/(3\alpha\tr(\HB)\log (s+K) )$. We have 
\begin{align*}
\text{for} &\ \ \ell  = 1, \ \ 
\BB^{(1)} 
\preceq (\IB - \gamma_0 \HB)^{s+K}\cdot \BB_0 \cdot (\IB - \gamma_0 \HB)^{s+K} \ + \\
&  9 \alpha \gamma_0^2 (s+K) \log (s+K)  \cdot \Big\langle{\frac{1}{\gamma_0 (s+K)}\IB_{0:k^*} + \HB_{k^*:\infty},\ \BB_0 }\Big\rangle \cdot \Big({\frac{1}{\gamma_0 (s+K)}\IB_{0:k^\dagger} + \HB_{k^\dagger:\infty}}\Big), \\
\text{for} &\ \ \ell  \ge 2, \ \ 
\BB^{(\ell)} 
\preceq (\IB - \gamma^{(\ell-1)}\HB)^K \cdot \BB^{(\ell-1)}\cdot (\IB - \gamma^{(\ell-1)}\HB)^K \ + \\
& 5 \alpha (\gamma^{(\ell-1)})^2 K \log K  \cdot \Big\langle{\frac{1}{\gamma^{(\ell-1)}K}\IB_{0:k^*} + \HB_{k^*:\infty},\ \BB^{(\ell-1)} }\Big\rangle \cdot \Big({\frac{1}{\gamma^{(\ell-1)}K}\IB_{0:k^\dagger} + \HB_{k^\dagger:\infty}}\Big),
\end{align*}
where $k^*$ and $k^\dagger$ can be arbitrary.
\end{lemma}
\begin{proof}
For $\BB^{(1)} $, we apply Lemma \ref{lemma:B-upper-bound} with $\gamma \to \gamma_0$ and $n \to s+K$, and use the condition that $\alpha \gamma_0 \tr(\HB) \log (s+K) \le 1/3$.

For $\BB^{(\ell)}$ with $\ell \ge 2$, we apply Lemma \ref{lemma:B-upper-bound} with $\gamma \to \gamma^{(\ell-1)}$, $n \to K$ and $\BB_0 \to \BB^{(\ell-1)}$, and use the condition that $\alpha \gamma^{(\ell-1)}\tr(\HB) \log (K) \le \alpha \gamma_0 \tr(\HB) \log (s+K) / 2 \le 1/6$.
\end{proof}

\begin{lemma}\label{lemma:HB-ell-reduction}
Suppose Assumptions \ref{assump:second-moment} and \ref{assump:fourth-moment} hold.
Consider \eqref{eq:B-iter-phase} and \eqref{eq:B-defi}.
Suppose $\gamma_0 < 1/(3\alpha\tr(\HB)\log (s+K))$. We have
\[
\big\langle{\HB, \BB_N }\big\rangle 
= \big\langle{\HB, \BB^{(L)} }\big\rangle 
\le  e \cdot \big\langle{ \HB, \BB^{(2)} }\big\rangle.
\]
\end{lemma}
\begin{proof}
Let $\ell \ge 2$.
In Lemma \ref{lemma:B-ell-recursion} choosing $k^*=0$ and $k^\dagger = \infty$ 
we obtain
\begin{align*}
\BB^{(\ell)} 
&\preceq(\IB - \gamma^{(\ell-1)}\HB)^K \cdot \BB^{(\ell-1)}\cdot (\IB - \gamma^{(\ell-1)}\HB)^K   + 5 \alpha \gamma^{(\ell-1)} \log K  \cdot \big\langle{ \HB, \BB^{(\ell-1)} }\big\rangle \cdot \IB \\
&\preceq \BB^{(\ell-1)} + 5 \alpha \gamma^{(\ell-1)} \log K  \cdot \big\langle{ \HB, \BB^{(\ell-1)} }\big\rangle \cdot \IB,
\end{align*}
which implies that
\begin{align*}
    \big\langle{\HB, \BB^{(\ell)} }\big\rangle
    &\le \big(1+ 5 \alpha \gamma^{(\ell-1)} \tr(\HB) \log K\big)\cdot \big\langle{ \HB, \BB^{(\ell-1)} }\big\rangle.
\end{align*}
The above inequality provides us with a recursion about the bias iterates that would not blow up:
\begin{align*}
    \big\langle{\HB, \BB^{(L)} }\big\rangle
    &\le \prod_{\ell=3}^{L} \big(1+ 5 \alpha \gamma^{(\ell-1)} \tr(\HB)\log K \big)\cdot \big\langle{ \HB, \BB^{(2)} }\big\rangle\\
    &\le e^{\sum_{\ell=3}^{L} 5 \alpha \gamma^{(\ell-1) } \tr(\HB)\log K} \cdot \big\langle{ \HB, \BB^{(2)} }\big\rangle  \\
    &\le e \cdot \big\langle{ \HB, \BB^{(2)} }\big\rangle,
\end{align*}
where the last inequality is because $\sum_{\ell=3}^L\gamma^{(\ell-1)}\le \gamma_0/2$ and $\alpha\gamma_0\tr(\HB)\log K < 1/3$.
\end{proof}

\begin{lemma}\label{lemma:HB-2-upper-bound}
Suppose Assumptions \ref{assump:second-moment} and \ref{assump:fourth-moment} hold.
Consider \eqref{eq:B-iter-phase} and \eqref{eq:B-defi}.
Suppose $\gamma_0 < 1/(3\alpha\tr(\HB)\log (s+K))$. We have 
\begin{align*}
    & \big\langle{\HB, \BB^{(2)}} \big\rangle 
    \le 12\cdot\Big\langle{\frac{1}{\gamma_0 K} \IB_{0:k^*} + \HB_{k^*:\infty}, (\IB-\gamma \HB)^{2(s+K)} \BB_0}\Big\rangle\ +  \\
    & 108 \alpha \log (s+K) \cdot \Big\langle{\frac{1}{\gamma_0 (s+K)} \IB_{0:k^\dagger} + \HB_{k^\dagger:\infty}, \BB_0}\Big\rangle \cdot \Big(\frac{k^*}{K}  + \gamma_0 \sum_{k^*\le i < k^\dagger}\lambda_i + \gamma_0^2 (s+K)\sum_{i>k^\dagger}\lambda_i^2 \Big),
\end{align*}
where $k^*$ and $k^\dagger$ can be arbitrary.
\end{lemma}
\begin{proof}
According to Lemma \ref{lemma:HB-ell-recursion}, we have 
\[
\big\langle{\HB, \BB^{(2)}}\big\rangle
\le 6\cdot \Big\langle{ \frac{1}{\gamma^{(1)} K} \IB_{0:k^*} + \HB_{k^*:\infty}, \BB^{(1)} }\Big\rangle
\le 12\cdot \Big\langle{ \frac{1}{\gamma_0 K} \IB_{0:k^*} + \HB_{k^*:\infty}, \BB^{(1)} }\Big\rangle.
\]
On the other hand, in Lemma \ref{lemma:B-ell-recursion} choosing $k^* = k^\dagger$, we have 
\begin{align*}
    & \BB^{(1)} 
    \le (\IB - \gamma_0 \HB )^{s+K}\cdot \BB_0 \cdot (\IB - \gamma_0 \HB )^{s+K} \ + \\
    & \ 9 \alpha \gamma_0^2 (s+K)\log (s+K)  \cdot \Big\langle{ \frac{1}{\gamma_0 (s+K)} \IB_{0:k^\dagger} + \HB_{k^\dagger:\infty}, \BB_0}\Big\rangle \cdot  \Big({\frac{1}{\gamma_0 (s+K)} \IB_{0:k^\dagger} + \HB_{k^\dagger:\infty}}\Big).
\end{align*}
Combining these two inequalities yields:
\begin{equation*}
\begin{aligned}
    \big\langle{\HB, \BB^{(2)}}\big\rangle
    & \le 12\cdot \Big\langle{\frac{1}{\gamma_0 K} \IB_{0:k^*} + \HB_{k^*:\infty}, (\IB-\gamma_0 \HB)^{2(s+K)} \BB_0}\Big\rangle \ + \\
    & \qquad 108 \alpha \gamma_0^2 (s+K)\log (s+K)\cdot \Big\langle{\frac{1}{\gamma_0 (s+K)} \IB_{0:k^\dagger} + \HB_{k^\dagger:\infty}, \BB_0}\Big\rangle  \ \times \\ 
    & \qquad\qquad\qquad  \underbrace{\Big\langle\frac{1}{\gamma_0 K} \IB_{0:k^*} + \HB_{k^*:\infty},\ \frac{1}{\gamma_0 (s+K)} \IB_{0:k^\dagger} + \HB_{k^\dagger:\infty}  \Big\rangle}_{(*)}.
\end{aligned}
\end{equation*}
The proof is completed by noting that
\begin{equation*}
        (*) \le  \frac{k^*}{\gamma_0^2 K(s+K)} + \frac{1}{\gamma_0 (s+K)}  \sum_{k^*<i\le k^\dagger} \lambda_i + \sum_{i>k^\dagger}\lambda_i^2. 
\end{equation*}
where $k^\dagger \ge k^*$ and $k^*$ and $k^\dagger$ are otherwise arbitrary.
\end{proof}

\begin{theorem}[A bias upper bound]\label{thm:HB-decay-lr-upper-bound}
Suppose Assumptions \ref{assump:second-moment} and \ref{assump:fourth-moment} hold.
Consider \eqref{eq:B-defi}.
Suppose $\gamma_0 < 1/(3\alpha\tr(\HB)\log (s+K))$. We have 
\begin{align*}
    & \abracket{\HB, \BB_N}
    \le 12e \cdot\Big\langle{\frac{1}{\gamma_0 K} \IB_{0:k^*} + \HB_{k^*:\infty}, (\IB-\gamma_0 \HB)^{2(s+K)} \BB_0}\Big\rangle \ + \\
    & 108 e \alpha \log (s+K) \cdot \Big\langle{\frac{1}{\gamma_0 (s+K)} \IB_{0:k^\dagger} + \HB_{k^\dagger:\infty}, \BB_0}\Big\rangle \cdot \Big(\frac{k^*}{K}  + \gamma_0 \sum_{k^*\le i < k^\dagger}\lambda_i + \gamma_0^2 (s+K)\sum_{i>k^\dagger}\lambda_i^2 \Big),
\end{align*}
where $k^*$ and $k^\dagger$ can be arbitrary.
\end{theorem}
\begin{proof}
This is by Lemmas \ref{lemma:HB-ell-reduction} and \ref{lemma:HB-2-upper-bound}.
\end{proof}

\subsection{Proof of Theorem \ref{thm:tail-decay-upper-bound}}
\begin{proof}[Proof of Theorem \ref{thm:tail-decay-upper-bound}]
This is by combining Lemma \ref{lemma:bias-var-decomp}, Theorems \ref{thm:HC-decay-lr-upper-bound} and \ref{thm:HB-decay-lr-upper-bound}, and set $R^2 = \alpha \tr(\HB)$.
\end{proof}

\subsection{Proof of Corollary \ref{thm:examples}}
\begin{proof}[Proof of Corollary \ref{thm:examples}]
For all these examples one can verify that $\tr(\HB) \eqsim 1$. Therefore $\gamma_0 \eqsim 1/\log N$.

According to the optimal choice of $k^*$ and $k^\dagger$, we can verify that
\begin{align*}
    {\frac{\big\| (\IB-\gamma_0 \HB)^{s+K} (\wB_0 - \wB^*) \big\|^2_{\IB_{0:k^*}} }{\gamma_0 K} +  \big\| (\IB-\gamma_0 \HB)^{s+K} (\wB_0 - \wB^*) \big\|^2_{\HB_{k^*:\infty}} } \lesssim \frac{\big\|  (\wB_0 - \wB^*) \big\|^2_{2} }{\gamma_0 K} 
    \lesssim \frac{\log^2 N}{N},
\end{align*}
and that 
\begin{align*}
    \log (s+K) \cdot \bigg( \frac{\|\wB_0 - \wB^* \|^2_{\IB_{0:k^\dagger}}}{\gamma_0 (s+K)}  + \|\wB_0 - \wB^* \|^2_{\HB_{k^\dagger:\infty}} \bigg) 
    \lesssim \log N \cdot \frac{\big\|  (\wB_0 - \wB^*) \big\|^2_{2}}{\gamma_0 (s+K)}
    \lesssim \frac{ \log^2 N}{N},
\end{align*}
therefore in Theorem \ref{thm:tail-decay-upper-bound} we have 
\begin{align*}
    \risk 
    &\le \biasErr + \varErr \\
    &\lesssim \frac{\log^2 N}{N} + \frac{\log^2 N}{N} \cdot (*) + (*) \\
    &\lesssim \max\Big\{ \frac{\log^2 N}{N}, (*) \Big\},
\end{align*}
where
\begin{align*}
     (*) &= \frac{k^*}{K} + \gamma_0 \sum_{k^* < i \le k^\dagger}\lambda_i + \gamma_0^2 (s+K) \sum_{i>k^\dagger}\lambda_i^2 \\
     &\eqsim \frac{k^* \log N}{N} + \frac{1}{\log N} \cdot \sum_{k^* < i \le k^\dagger}\lambda_i + \frac{N}{ \log^2 N} \cdot \sum_{i>k^\dagger}\lambda_i^2.
\end{align*}
We next exam the order of ${\log^2 N} / {N}$ vs. $(*)$.
\begin{enumerate}
    \item By definitions we have 
    \[k^* \eqsim \bracket{\frac{N}{\log^2 N}}^{\frac{1}{1+r}},\quad
    k^\dagger \eqsim \bracket{\frac{N}{\log N}}^{\frac{1}{1+r}},\]
    therefore we have 
    \begin{align*}
    (*) &\eqsim  \bracket{\frac{N}{\log^2 N}}^{\frac{1}{1+r}} \cdot \frac{\log N}{N} + \frac{1}{\log N} \cdot \bracket{\frac{N}{\log^2 N}}^{\frac{-r}{1+r}} + \frac{N}{ \log^2 N} \cdot  \bracket{\frac{N}{\log N}}^{\frac{-1-2r}{1+r}} \\
    & \eqsim (\log N)^{\frac{r-1}{1+r}} \cdot N^{\frac{-r}{1+r}} + (\log N)^{\frac{-1}{1+r}} \cdot N^{\frac{-r}{1+r}} 
    \eqsim (\log N)^{\frac{r-1}{1+r}} \cdot N^{\frac{-r}{1+r}}.
    \end{align*}
    This implies that 
    \(\risk \lesssim (\log N)^{\frac{r-1}{1+r}} \cdot N^{\frac{-r}{1+r}}.\)
    
    \item By definitions we have 
    \[k^* \eqsim N \cdot (\log N)^{-2-r},\quad 
    k^\dagger \eqsim N \cdot (\log N)^{-1-r}, \]
    therefore we have 
    \begin{align*}
        (*) &\eqsim (\log N)^{-1-r} + \frac{1}{\log N} \cdot (\log k^*)^{1-r} + \frac{N}{\log^2 N} \cdot \bracket{ (k^\dagger)^{-1} \cdot (\log k^\dagger)^{-2r} }  \\
        &\eqsim (\log N)^{-1-r} + (\log N)^{-r} + (\log N)^{-1-r}
        \eqsim (\log N)^{-r}.
    \end{align*}
    This implies that \(\risk \lesssim (\log N)^{-r}.\)

    \item By definitions we have 
    \[k^* \eqsim  \log N,\quad 
    k^\dagger \eqsim \log N, \]
    therefore we have 
    \begin{align*}
        (*) \eqsim \frac{\log^2 N}{N} + \frac{1}{\log N} \cdot 2^{ - k^*} + \frac{N}{\log^2 N} \cdot 2^{-2 k^\dagger} \eqsim \frac{\log^2 N}{N}.
    \end{align*}
    This implies that \(\risk \lesssim \log^2 N / N.\)
\end{enumerate}
\end{proof}

\section{Proof of Lower Bound}\label{append-sec:lower-bound}
\subsection{Variance Lower Bound}
\begin{theorem}[A variance lower bound]\label{thm:C-lower-bound}
Suppose Assumptions \ref{assump:second-moment} and \ref{assump:well-specified-noise} hold.
Consider \eqref{eq:C-defi}.
Let $K = (N-s) / \log (N-s)$.
Suppose $s\ge 0$, $K\ge 10$ and $\gamma_0 < 1/\lambda_1$. We have 
\[
\abracket{\HB,\CB_{N} } \ge \frac{\sigma^2}{400}\left(\frac{k^*}{K} + \gamma_0 \sum_{k^* < i \le k^\dagger} \lambda_i + \gamma_0^2 (s+K) \sum_{i>k^\dagger}\lambda_i^2\right),
\]
where $k^*:= \max\cbracket{k: \lambda_k \ge 1/(\gamma_0 K)  }$ and $k^\dagger := \max\cbracket{k: \lambda_k \ge 1/(\gamma_0 (s+K))  }$.
\end{theorem}
\begin{proof}
Notice that
\begin{align}
    \CB_t 
    &= (\Ical - \gamma_t \widetilde{\Tcal}_t)\circ \CB_{t-1} + \gamma_t^2 (\Mcal-\widetilde{\Mcal}) \circ \CB_{t-1} + \gamma_t^2 \sigma^2 \HB \notag \\
    &\succeq (\Ical - \gamma_t \widetilde{\Tcal}_t)\circ \CB_{t-1} + \gamma_t^2 \sigma^2 \HB. \notag
\end{align}
Solving the recursion we obtain
\begin{align}
    \CB_N 
    &\succeq {\sigma^2} \sum_{t=1}^N \gamma_t^2 \prod_{i=t+1}^N  (\Ical-\gamma_i \widetilde{\Tcal}_i)\circ \HB \notag \\
    &= {\sigma^2} \sum_{t=1}^N \gamma_t^2 \prod_{i=t+1}^N  (\IB-\gamma_i \HB)^2 \HB \notag \\
    &\succeq {\sigma^2} \underbrace{\sum_{t=1}^N \gamma_t^2 \prod_{i=t+1}^N  (\IB-2\gamma_i \HB) \HB }_{(*)}. \label{eq:C-lower-bound-expanded}
\end{align}
Now recalling \eqref{eq:geometry-tail-decay-lr}, we have
\begin{align}
    (*) 
    &=\gamma_0^2 \sum_{i=1}^{s+K} \Big( \IB - 2\gamma_0 \HB \Big)^{s+K - i} \prod_{j=1}^{L-1} \Big(\IB-\frac{\gamma_0}{2^{j-1}}\HB \Big)^{K} \HB \notag \\
    &\qquad + \sum_{\ell=1}^{L - 1} \Big(\frac{\gamma_0}{2^{\ell}}\Big)^2 \sum_{i=1}^{K} \Big( \IB - \frac{\gamma_0}{2^{\ell-1}}\HB \Big)^{K - i} \prod_{j=\ell+1}^{L-1} \Big(\IB-\frac{\gamma_0}{2^{j-1}}\HB \Big)^{K} \HB \notag\\
    &= \frac{\gamma_0}{2} \left( \IB - \Big( \IB - 2\gamma_0\HB \Big)^{s+K} \right) \Bigg(\prod_{j=1}^{L-1} \Big(\IB-\frac{\gamma_0}{2^{j-1}}\HB \Big)\Bigg)^{K} \notag\\ 
    &\qquad +\sum_{\ell=1}^{L - 1} \frac{\gamma_0}{2^{\ell+1}} \left( \IB - \Big( \IB - \frac{\gamma_0}{2^{\ell-1}}\HB \Big)^{K} \right) \Bigg(\prod_{j=\ell+1}^{L-1} \Big(\IB-\frac{\gamma_0}{2^{j-1}}\HB \Big)\Bigg)^{K}  \notag \\
    &\ge \frac{\gamma_0}{2} \left( \IB - \Big( \IB - 2\gamma_0\HB \Big)^{s+K} \right) \Big( \IB-2\gamma_0 \HB \Big)^{K} \notag\\ 
    &\qquad +\sum_{\ell=1}^{L - 1} \frac{\gamma_0 }{2^{\ell+1}} \left( \IB - \Big( \IB - \frac{\gamma_0 }{2^{\ell-1}}\HB \Big)^{K} \right)  \Big(\IB-\frac{\gamma_0 }{2^{\ell-1}}\HB \Big)^{K}, \label{eq:C-lower-bound-expanded-part}
\end{align}
where we understand $\prod_{j=L}^{L-1}(\cdot) = 1$, and the last inequality is because for every $\ell \ge 0$,
\[
\prod_{j=\ell+1}^{L-1} \Big(\IB-\frac{\gamma_0 }{2^{j-1}}\HB \Big)
\ge \IB - \sum_{j=\ell+1}^{L-1} \frac{\gamma_0 }{2^{j-1}}\HB
\ge \IB - \frac{\gamma_0 }{2^{\ell-1}} \HB,
\]
where we understand $\sum_{j=L}^{L-1}(\cdot) = 0$.
Define a scalar function
\[
f(x) := \frac{x}{2} \cdot \Big(1-\big(1-2x\big)^{s+K}\Big) \cdot \big(1-2x\big)^K + \sum_{\ell=1}^{L-1} \frac{x}{2^{\ell+1}} \cdot \Big(1-\big(1-\frac{x}{2^{\ell-1}}\big)^K\Big) \cdot \big(1-\frac{x}{2^{\ell-1}}\big)^K, 
\]
and apply it to $\gamma_0 \HB$ entry-wisely, then according to Lemma \ref{lemma:scalar-func-lower-bound}, we have
\[
f(\gamma_0 \HB) \succeq  \frac{1}{400K} \IB_{0:k^*} +  \frac{\gamma_0 }{40} \HB_{k^*:k^\dagger} + \frac{\gamma_0^2 (s+K)}{40} \HB^2_{k^\dagger:\infty}, 
\]
where $k^*:= \max\cbracket{k: \lambda_k \ge 1/(\gamma_0 K)  }$ and $k^\dagger := \max\cbracket{k: \lambda_k \ge 1/(\gamma_0 (s+K))  }$.
Now using \eqref{eq:C-lower-bound-expanded} and \eqref{eq:C-lower-bound-expanded-part} we obtain
\[
\CB_{N} \succeq \sigma^2 \cdot f(\gamma_0 \HB)\cdot  \HB^\inv
\succeq \frac{\sigma^2}{400} \left( \frac{1}{K} \HB^\inv_{0:k^*} + \gamma_0 \IB_{k^*:k^\dagger} + \gamma_0^2 (s+K) \HB_{k^\dagger:\infty}  \right),
\]
and as a consequence, 
\[
\abracket{\HB,\CB_{N} } \ge \frac{\sigma^2}{400}\left(\frac{k^*}{K} + \gamma_0 \sum_{k^* < i \le k^\dagger} \lambda_i + \gamma_0^2 (s+K) \sum_{i>k^\dagger}\lambda_i^2\right),
\]
where $k^*:= \max\cbracket{k: \lambda_k \ge 1/(\gamma_0 K)  }$ and $k^\dagger := \max\cbracket{k: \lambda_k \ge 1/(\gamma_0 (s+K))  }$.
\end{proof}

\begin{lemma}\label{lemma:scalar-func-lower-bound}
Suppose $s\ge 0$, $K \ge 10$ and $x\in(0,1]$.
For the scalar function 
\[
f(x) := \frac{x}{2} \cdot \Big(1-\big(1-2x\big)^{s+K}\Big) \cdot \big(1-2x\big)^K + \sum_{\ell=1}^{L-1} \frac{x}{2^{\ell+1}} \cdot \Big(1-\big(1-\frac{x}{2^{\ell-1}}\big)^K\Big) \cdot \big(1-\frac{x}{2^{\ell-1}}\big)^K, 
\]
we have
\[
f(x) \ge 
\begin{cases}
{(s+K) x^2}/{40}, & 0 < x < {1}/(s+K) ;\\
{x}/{40}, & {1}/( s+ K) \le x < {1}/{K};\\
{1}/(400K), & {1}/{K} \le x \le 1.
\end{cases}
\]
\end{lemma}
\begin{proof}
We prove each part of the lower bound separately.
\begin{itemize}
    \item 
    For $x\in (0, 1/(s+K))$ and $K \ge 10$, we have
\(
(1-2x)^{s+K} \le  (1-x)^{s+K} \le 1-(s+K)x/{2}\) and 
\( (1-2x)^K \ge (1-2/(s+K))^K\ge (1-2/K)^K\ge (1-2/10)^{10}\ge \frac{1}{10}, \)
which yield
\begin{equation*}
f(x) \ge \frac{x}{2}\cdot \frac{(s+K)x}{2}\cdot \frac{1}{10} = \frac{(s+K) x^2}{40}.
\end{equation*}

\item For $x\in [1/(s+K), 1/K)$ and $K \ge 10$, we have 
\((1-2x)^{s+K} \le (1-2/(s+K))^{s+K} \le 1/e^2\)
and 
\( (1-2x)^K \ge (1-2/K)^K\ge (1-2/10)^{10} \ge \frac{1}{10}, \)
which yield
\begin{equation*}
f(x) \ge \frac{x}{2}\cdot \big(1-\frac{1}{e^2}\big) \cdot \frac{1}{10} \ge  \frac{x}{40}.
\end{equation*}

\item For $x\in [1/K, 1]$, there is an $\ell^* := \lfloor \log (K x) \rfloor \in [0, L)$, such that 
    \( {2^{\ell^*}} / {K} \le x < {2^{\ell^*+1}} / {K},\)
which yields
\begin{align}
f(x) 
&\ge  \sum_{\ell=0}^{L-1} \frac{x}{2^{\ell+1}} \cdot \Big(1-\big(1-\frac{x}{2^{\ell-1}}\big)^K\Big) \cdot \big(1-\frac{x}{2^{\ell-1}}\big)^K  \qquad(\text{since $s+K \ge K$}) \notag \\
&\ge \frac{x}{2^{\ell^*+1}} \cdot \Big(1-\big(1-\frac{x}{2^{\ell^*-1}}\big)^K\Big) \cdot \big(1-\frac{x}{2^{\ell^*-1}}\big)^K \notag \\
&\ge \frac{1}{2K} \cdot \Big( 1 - \big(1-\frac{2}{K}\big)^K\Big)\cdot \big(1-\frac{4}{K}\big)^K \notag \\
&\ge \frac{1}{2K} \cdot \Big( 1 - \frac{1}{e^2}\Big) \cdot \Big(1-\frac{4}{10}\Big)^{10} 
\ge \frac{1}{400K}. \qquad(\text{since $K \ge 10$}) \notag 
\end{align}
\end{itemize}
\end{proof}

\subsection{Bias Lower Bound}
We now build a lower bound for the bias error.
\begin{theorem}\label{thm:B-lower-bound}
Suppose Assumptions \ref{assump:second-moment}, \ref{assump:fourth-moment} and \ref{assump:well-specified-noise} hold.
Consider \eqref{eq:B-defi}.
Let $K = (N-s) / \log (N-s)$.
Suppose $s\ge 0$, $K\ge 10$ and $\gamma_0 < 1/\lambda_1$. We have 
\begin{align*}
    \abracket{\HB, \BB_N} 
    &\ge \big\langle  \HB, (\IB- \gamma_0 \HB)^{2(s+2K)} \BB_0 \big\rangle \\
    & \quad + \frac{\beta}{1200}\cdot \langle\HB_{k^\dagger:\infty}, \BB_0 \rangle \cdot \left( \frac{k^*}{K} + \gamma_0 \sum_{k^*< i \le k^\dagger} + \gamma_0^2 (s+K) \sum_{i > k^\dagger}\lambda_i^2  \right),
\end{align*}
where $k^*:= \max\cbracket{k: \lambda_k \ge 1/(\gamma_0 K)  }$ and $k^\dagger := \max\cbracket{k: \lambda_k \ge 1/(\gamma_0 (s+K))  }$.
\end{theorem}
\begin{proof}
Starting from \eqref{eq:B-defi}, we have 
\begin{align}
    \BB_n &= (\Ical - \gamma_n \widetilde{\Tcal}_n)\circ \BB_{n-1} + \gamma_n^2 (\Mcal-\widetilde{\Mcal}) \circ \BB_{n-1} \notag \\
    &\succeq (\Ical - \gamma_n \widetilde{\Tcal}_n ) \circ \BB_{n-1} + \beta \gamma_n^2 \cdot \HB \cdot\abracket{\HB, \BB_{n-1}} \label{eq:B-lower-bound-expanded} \\
    &\succeq (\Ical - \gamma_n \widetilde{\Tcal}_n ) \circ \BB_{n-1}, \notag
\end{align}
recursively solving this, we obtain a crude lower bound on $\BB_n$:
\begin{align*}
    \BB_n \succeq \prod_{t=1}^n (\Ical - \gamma_t \widetilde{\Tcal}_t) \circ \BB_{0} \succeq  \prod_{t=1}^N (\Ical - \gamma_t \widetilde{\Tcal}_t) \circ \BB_{0}, \quad \text{for}\ n=1,\dots,N.
\end{align*}
This gives us a crude lower bound on $\abracket{\HB, \BB_n}$ for $n=1,\dots,N$:
\begin{align*}
    \abracket{\HB, \BB_n}  
    \ge \Big\langle  \prod_{t=1}^N (\Ical - \gamma_t \widetilde{\Tcal}_t) \circ \HB, \ \BB_0 \Big\rangle  = \Big\langle {\prod_{t=1}^N (\IB- \gamma_t \HB)^2 \HB, \ \BB_0} \Big\rangle.
\end{align*}
Bring this into \eqref{eq:B-lower-bound-expanded}, we have
\begin{align*}
    \BB_n &\succeq (\Ical - \gamma_n \widetilde{\Tcal}_n ) \circ \BB_{n-1} + \beta \gamma_n^2 \cdot \HB \cdot \Big\langle {\prod_{t=1}^N (\IB- \gamma_t \HB)^2 \HB, \ \BB_0} \Big\rangle, \quad \text{for}\ n=1,\dots,N,
\end{align*}
recursively solving which yields:
\begin{equation}\label{eq:B-fine-lower-bound-expanded}
    \BB_N \succeq \prod_{t=1}^n (\Ical - \gamma_t \widetilde{\Tcal}_t) \circ \BB_{0} + \beta \Big\langle {\prod_{t=1}^N (\IB- \gamma_t \HB)^2 \HB, \ \BB_0} \Big\rangle \cdot \underbrace{\sum_{t=1}^N \gamma_t^2 \prod_{i=t+1}^N (\Ical - \gamma_i \widetilde{\Tcal}_i)\circ \HB}_{(*)}.
\end{equation}
Noting that here the term $(*)$ in \eqref{eq:B-fine-lower-bound-expanded} is exactly the term $(*)$ appeared in \eqref{eq:C-lower-bound-expanded} in Theorem \ref{thm:C-lower-bound}, therefore by repeating the analysis in Theorem \ref{thm:C-lower-bound} we know that
\begin{equation*}
    (*) \succeq  \frac{1}{400} \left( \frac{1}{K} \HB^\inv_{0:k^*} + \gamma_0 \IB_{k^*:k^\dagger} + \gamma_0^2 (s+K) \HB_{k^\dagger:\infty}  \right),
\end{equation*}
where $k^*:= \max\cbracket{k: \lambda_k \ge 1/(\gamma_0 K)  }$ and $k^\dagger := \max\cbracket{k: \lambda_k \ge 1/(\gamma_0 (s+K))  }$.
As a consequence we have
\begin{equation}\label{eq:term-star-lower-bound}
    \abracket{\HB, (*)} \ge \frac{1}{400} \left( \frac{k^*}{K} + \gamma_0 \sum_{k^*< i \le k^\dagger} + \gamma_0^2 (s+K) \sum_{i > k^\dagger}\lambda_i^2  \right).
\end{equation}
Back to \eqref{eq:B-fine-lower-bound-expanded}, taking inner product with $\HB$ yields
\begin{align}
    \abracket{\HB, \BB_N} 
    &\ge \Big\langle \prod_{t=1}^N (\Ical - \gamma_t \widetilde{\Tcal}_t) \circ \HB, \ \BB_0 \Big\rangle + \beta\cdot \Big\langle {\prod_{t=1}^N (\IB- \gamma_t \HB)^2 \HB, \ \BB_0} \Big\rangle \cdot  \abracket{\HB, (*)}  \notag \\
    &= \Big\langle {\underbrace{\prod_{t=1}^N (\IB- \gamma_t \HB)^2 \HB}_{(**)}, \ \BB_0} \Big\rangle + \beta  \cdot \Big\langle \underbrace{\prod_{t=1}^N (\IB- \gamma_t \HB)^2 \HB}_{(**)}, \ \BB_0 \Big\rangle \cdot \abracket{\HB, (*)}. \label{eq:HB-fine-lower-bound}
\end{align}
We next bound $(**)$. Recall \eqref{eq:geometry-tail-decay-lr}, we have
\begin{align}
    (**) 
    &= (\IB- \gamma_0 \HB)^{2(s+K)} \cdot \prod_{\ell=1}^{L-1} \big(\IB- \frac{\gamma_0}{2^\ell} \HB \big)^{2K} \cdot \HB \notag \\
    &\succeq (\IB- \gamma_0 \HB)^{2(s+K)} \cdot  \big(\IB- \sum_{\ell=1}^{L-1}\frac{\gamma_0}{2^\ell} \HB \big)^{2K} \cdot \HB \notag \\
    &\succeq (\IB- \gamma_0 \HB)^{2(s+K)} \cdot  \big(\IB- \gamma_0 \HB \big)^{2K} \cdot \HB \notag \\
    & = (\IB- \gamma_0 \HB)^{2(s+2K)} \HB, \label{eq:term-star-star-lower-bound}
\end{align}
Noticing that for $K\ge 10$ and $x \in (0, 1/(s+K))$,
\[
(1-x)^{2(s+2K)} \ge (1-x)^{s+K} \ge \Big(1 - \frac{1}{s+K}\Big)^{s+K} \ge \big(1-\frac{1}{10}\big)^{10} \ge \frac{1}{3},
\]
we can further lower bound $(**)$ with $k^\dagger := \max\cbracket{k: \lambda_k \ge 1/(\gamma_0 (s+K))  }$:
\begin{equation}\label{eq:term-star-star-lower-bound-split}
    (**) \succeq (\IB- \gamma_0 \HB)^{2(s+2K)} \HB 
    \succeq \frac{1}{3} \HB_{k^\dagger:\infty}
\end{equation}
Bringing \eqref{eq:term-star-lower-bound}, \eqref{eq:term-star-star-lower-bound} and \eqref{eq:term-star-star-lower-bound-split} into \eqref{eq:HB-fine-lower-bound} completes the proof:
\begin{align*}
    \abracket{\HB, \BB_N} 
    &\ge \big\langle  \HB, (\IB- \gamma_0 \HB)^{2(s+2K)} \BB_0 \big\rangle \\
    & \quad + \frac{\beta}{1200}\cdot \langle\HB_{k^\dagger:\infty}, \BB_0 \rangle \cdot \left( \frac{k^*}{K} + \gamma_0 \sum_{k^*< i \le k^\dagger} + \gamma_0^2 (s+K) \sum_{i > k^\dagger}\lambda_i^2  \right),
\end{align*}
where $k^*:= \max\cbracket{k: \lambda_k \ge 1/(\gamma_0 K)  }$ and $k^\dagger := \max\cbracket{k: \lambda_k \ge 1/(\gamma_0 (s+K))  }$.
\end{proof}

\subsection{Proof of Theorem \ref{thm:tail-decay-lower-bound}}
\begin{proof}[Proof of Theorem \ref{thm:tail-decay-lower-bound}]
This is by combining Lemma \ref{lemma:bias-var-decomp-lowerbonud}, Theorems \ref{thm:C-lower-bound} and \ref{thm:B-lower-bound}.
\end{proof}

\section{Proof for Polynomially Decaying Stepsize}\label{append-sec:poly}
Recall that the polynomially decaying stepsize satisfies the following rule:
\begin{equation}\label{eq:stepsize-decay2}
    \gamma_t = 
    \begin{cases}
    \gamma_0, & 1 \le t \le s; \\
    {\gamma_0}/{(t-s)^a}, & s < t \le N.
    \end{cases}
\end{equation}

\subsection{Proof of the Lower Bound of Variance Error}

\begin{lemma}\label{lemma:lowerbound_var_poly}
Suppose $\gamma_0< 1/(4\lambda_1)$ and apply polynomially decaying stepsize, then it holds that
\item \textbf{Case 1: $0\le a<1$.} Let $k^* = \max\{k:\gamma_0 \lambda_k\ge (1-a)/[2(N-s)^{a-1}]\}$ and $k^\dagger = \max\{k:\gamma_0 \lambda_k\ge 1/(2s)\}$, we have
    \begin{align*}
    \langle\HB, \CB_N\rangle \ge\sigma^2 \cdot\bigg(\sum_{i\le k^*} \frac{(1-a)\cdot\gamma_0\lambda_i}{N^a}\vee\frac{(1-a)^2a\log(N)}{16 eN}+\sum_{k^*+1\le i\le k^\dagger} \frac{\gamma_0\lambda_i}{4e^2} + \sum_{i\ge k^\dagger+1}\frac{s\gamma^2\lambda_i^2}{2e^2}\bigg).
    \end{align*}
    
    \item \textbf{Case 2: $a=1$.}  Let $k^* = \max\{k:\gamma_0\lambda_k\ge 1/(2+2\log(N-s-1))\}$ and $k^\dagger = \max\{k:\gamma_0\lambda_k\ge 1/(2s)\}$, we have
    \begin{align*}
    \langle\HB, \CB_N\rangle \ge \sigma^2\cdot\bigg(\sum_{i\le k^*} \frac{\gamma^2\lambda_i^2}{N^{4\gamma_0\lambda_i}}+\sum_{k^*+1\le i\le k^\dagger} \frac{\gamma_0\lambda_i}{4e^2} + \sum_{i\ge k^\dagger+1}\frac{s\gamma^2\lambda_i^2}{2e^2}\bigg).
    \end{align*}
    
\end{lemma}
\begin{proof}
Consider $\CB_t$ defined in \eqref{eq:C-defi}, we have
\begin{align}
    \CB_N
    &= (\Ical - \gamma_N \widetilde{\Tcal})\circ \CB_{N-1} + \gamma_N^2 (\Mcal-\widetilde{\Mcal}) \circ \CB_{N-1} + \gamma_N^2 \sigma^2 \HB \notag \\
    &\succeq (\Ical - \gamma_N \widetilde{\Tcal})\circ \CB_{N-1} + \gamma_N^2 \sigma^2 \HB\notag\\
    &\succeq {\sigma^2} \sum_{t=1}^N \gamma_t^2 \prod_{i=t+1}^N  (\Ical-\gamma_i \widetilde{\Tcal})\circ \HB,
\end{align}
where in the last inequality we use the fact that $\CB_0 = \boldsymbol{0}$.
Then using the fact that $\HB$ is a PSD matrix, it holds that
\begin{align*}
\langle\HB,\CB_N\rangle&\ge \sigma^2\cdot \sum_{t=1}^N \gamma_t^2\cdot \bigg\langle\prod_{i=t+1}^N  (\Ical-\gamma_i \widetilde{\Tcal})\circ \HB, \HB\bigg\rangle\notag\\
& = \sigma^2\cdot\sum_{t=1}^N \gamma_t^2\cdot \bigg\langle \prod_{i=t+1}^N(\IB-\gamma_i\HB)^2\HB,\HB\bigg\rangle\notag\\
& = \sigma^2\sum_{j}\sum_{t=1}^N\gamma_t^2\cdot \prod_{i=t+1}^N(1-\gamma_i\lambda_j)^2\lambda_j^2,
\end{align*}
where the second equality follows from the definition of $\widetilde{\Tcal}$ and the fact that $\widetilde{\Tcal}\circ\AB$ is commute to $\HB$ for any $\AB$ that is commute to $\HB$. Then it suffices to consider the following scalar function:
\begin{align}\label{eq:def_fx}
f(x) :&= \sum_{t=1}^N\gamma_t^2 \prod_{i=t+1}^N  (1-\gamma_i x)^2 x^2 \notag\\
& = \underbrace{\gamma^2\sum_{t=1}^s\prod_{i=t+1}^N  (1-\gamma_i x)^2 x^2}_{:=f_1(x)} + \underbrace{\sum_{t=s+1}^N\gamma_t^2\prod_{i=t+1}^N  (1-\gamma_i x)^2 x^2}_{:=f_2(x)}
\end{align}
where we explicitly decompose the function $f(x)$ into the summation of two functions $f_1(x)$ and $f_2(x)$ according to the length of iterations that use stepsize $\gamma$.

\noindent\textbf{Lower bound of $f_1(x)$.} We first provide a lower bound of $f_1(x)$. Note that $f_1(x)$ can be rewritten as
\begin{align*}
f_1(x)& = \gamma^2\cdot \sum_{t=1}^s (1-\gamma x)^{2(s-t)}\cdot\prod_{i=s+1}^N  (1-\gamma_i x)^2 x^2\notag\\
& = \gamma^2\cdot\prod_{i=s+1}^N  (1-\gamma_i x)^2 x^2\cdot \sum_{t=0}^{s-1} (1-\gamma x)^{2t}.
\end{align*}
Then note that for any $i\ge s+1$, we have $\gamma_i = \gamma/(i-s)^a$. Applying the fact that $(1-\gamma x)^2\ge (1-2\gamma x)$ for any $\gamma x\le 1/2$, it follows that
\begin{align*}
f_1(x)& \ge \gamma^2\cdot\prod_{i=1}^{N-s}  \bigg(1-\frac{\gamma x}{i^a}\bigg)^2 x^2\cdot \sum_{t=0}^{s-1} (1-2\gamma x)^{t}\notag\\
& = \frac{\gamma x}{2}\cdot \big[1 - (1-2\gamma x)^{s}]\cdot \prod_{i=1}^{N-s}  \bigg(1-\frac{\gamma x}{i^a}\bigg)^2.
\end{align*}
Moreover, note that $\gamma x\le 1/2$ implies that $(1-\gamma x/i^a)\ge e^{-2\gamma x/i^a}$ for any $i\ge 1$, we further have
\begin{align}\label{eq:lowerbound_f1_temp1}
f_1(x)& \ge\frac{\gamma x}{2}\cdot \big[1 - (1-2\gamma x)^{s}]\cdot e^{-4\gamma x\cdot\sum_{i=1}^{N-s} i^{-a}}.
\end{align}
Note that
\begin{align}\label{eq:bound_sum_poly}
\sum_{i=1}^{N-s} i^{-a} = 1 + \sum_{i=2}^{N-s} i^{-a} \le 1 + \int_{1}^{N-s-1} z^{-a}\mathrm{d} z = \begin{cases}
1+\frac{(N-s-1)^{1-a}-1}{1-a}, & 0\le a<1; \\
   1+\log(N-s-1), &a=1.
\end{cases}
\end{align}
Therefore, plugging \eqref{eq:bound_sum_poly} into \eqref{eq:lowerbound_f1_temp1}, it holds that
\begin{align*}
f_1(x)\ge 
\begin{cases}
\frac{\gamma x}{2}\cdot\big[1 - (1-2\gamma x)^{s}]\cdot e^{-4\gamma x\cdot (N-s)^{1-a}/(1-a)} , & 0\le a < 1\\
\frac{\gamma x}{2}\cdot\big[1 - (1-2\gamma x)^{s}]\cdot e^{-4\gamma x\cdot [1+\log(N-s)]}, &a=1.
\end{cases}
\end{align*}
\begin{itemize}
    \item \textbf{Case of $0\le a<1$.}
For the case of $0\le a<1$,  assume $s=\Omega((N-s)^{1-a})$, let $k^* = \max\{k:\gamma_0 \lambda_k\ge (1-a)/[2(N-s)^{1-a}]\}$ and $k^\dagger = \max\{k:\gamma_0 \lambda_k\ge 1/(2s)\}$, where it can be verified that $k^*\le k^\dagger$). Then note that for any $k\le k^\dagger$, we have
\begin{align*}
1-(1-2\gamma x)^s\ge 1-(1-1/s)^{s}\ge 1/2
\end{align*}
and for any $k\ge k^\dagger +1$,
\begin{align*}
1-(1-2\gamma x)^s\ge 1 - e^{-2s\gamma x}\ge 1-(1-s\gamma x)\ge s\gamma x
\end{align*}
where the second inequality holds since $e^{-x}\le 1-x/2$ for any $x\in[0, 1]$. Besides, we also have for any $k\ge k^*+1$,
\begin{align*}
e^{-4\gamma x\cdot (N-s)^{1-a}/(1-a)}\ge e^{-2}.
\end{align*}
Besides, note that the $g(x) = xe^{- cx}$ first increases and then decreases as $x$ increases,  then for any for any $x\in \big[(1-a)/[2(N-s)^{1-a}], a(1-a)\log(N)/[4(N-s)^{1-a}]\big]$, we have
\begin{align*}
f_1(x)&\ge \min\big\{f_1\big((1-a)/[2(N-s)^{1-a}]\big), f_1\big(a(1-a)\log(N)/[4(N-s)^{1-a}]\big)\big\}\notag\\
&=\min\bigg\{\frac{1-a}{4e^2(N-s)^{1-a}},\frac{a(1-a)\log(N)}{8(N-s)^{1-a}}\cdot e^{-a\log(N)}\bigg\}\notag\\
& = \frac{a(1-a)\log(N)}{8N}.
\end{align*}
Therefore, we can further define $k' = \max\{k: \gamma_0\lambda k\ge a(1-a)\log(N)/[4(N-s)^{a-1}]\}$ such that for any $k'< k\le k^*$, it holds that
\begin{align*}
f_1(\lambda_k)\ge \frac{\gamma_0\lambda_k}{4}\cdot e^{-a\log(N)} = \frac{a(1-a)\log(N)}{8N}
\end{align*}

Combining these bounds we can obtain that
\begin{align}\label{eq:lb_case1_f1}
\sum_i f_1(\lambda_i) \ge \sum_{i>k'}f_1(\lambda_i) \ge  \sum_{k'< i\le k^*}\frac{a(1-a)\log(N)}{8N}+\sum_{k^*< i\le k^\dagger}\frac{\gamma_0\lambda_i}{4e^2}+ \sum_{i\ge k^\dagger+1}\frac{s\gamma^2\lambda_i^2}{2e^2}.
\end{align}
\item \textbf{Case of $a=1$.} Similarly, when $a=1$, we can redefine $k^*$ as $k^* = \max\{k:\gamma_0 \lambda_k\ge 1/(2+2\log(N-s-1))\}$ and then similarly, it holds that
\begin{align}\label{eq:lb_case2_f1}
\sum_i f_1(\lambda_i) \ge \sum_{k\ge k^* +1}f_1(\lambda_i)\ge \sum_{k^* < i\le k^\dagger} \frac{\gamma_0\lambda_i}{4e^2} + \sum_{i\ge k^\dagger+1}\frac{s\gamma^2\lambda_i^2}{2e^2}.
\end{align}

\end{itemize}

\noindent\textbf{Lower bound of $f_2(x)$.} Plugging the formula of the polynomially decaying stepsize, we have
\begin{align*}
f_2(x) &= \sum_{t=s+1}^N\gamma_t^2\prod_{i=t+1}^N  (1-\gamma_i x)^2 x^2 \notag\\
&=\sum_{t=s+1}^N\frac{\gamma^2}{(t-s)^{2a}}\cdot \prod_{i=t+1}^N\bigg( 1 - \frac{\gamma x}{(i-s)^{a}}\bigg)^2 x^2\notag\\
& = \sum_{t=1}^{N-s}\frac{\gamma^2}{t^{2a}}\cdot \prod_{i=t+1}^{N-s}\bigg( 1 - \frac{\gamma x}{i^{a}}\bigg)^2 x^2
\end{align*}
Similarly, for any $\gamma x\le 1/2$, we have
\begin{align*}
\bigg( 1 - \frac{\gamma x}{i^{a}}\bigg)^2 \ge 1 - \frac{2\gamma x}{i^{a}}\ge e^{-4\gamma x/i^a}.
\end{align*}
Then it follows that
\begin{align}\label{eq:lowerbound_f2_temp2}
f_2(x)\ge \sum_{t=1}^{N-s}\frac{\gamma^2 x^2}{t^{2a}}\cdot e^{-4\gamma x\cdot \sum_{i={t+1}}^{N-s}i^{-a}}.
\end{align}
\begin{itemize}
\item \textbf{Case of $0\le a<1$} 

We first consider the case of $0\le a <1$, where two cases will be studied separately: (1) $\gamma x\le (1-a)/[2(N-s)^{a-1}]  $ and (2) $\gamma x> (1-a)/[2(N-s)^{a-1}] $. 
For the first case, it is clear that
\begin{align*}
4\gamma x\cdot\sum_{i=t+1}^{N-s} i^{-a}\le4\gamma x\cdot\sum_{i=1}^{N-s}i^{-a}\le \frac{2(1-a)}{(N-s)^{1-a}}\cdot \frac{(N-s)^a}{(1-a)}=2,
\end{align*}
which implies that 
\begin{align}\label{eq:f2_case1_case1}
f_2(x) \ge e^{-2}\cdot \gamma^2 x^2\cdot \sum_{t=1}^{N-s} t^{-2a} \ge \frac{[1+(N-s)^{1-2a}]\gamma^2x^2}{2e^2} .
\end{align}
Then we can move to the case 
\begin{align}\label{eq:condition_gamma_x_2}
4\gamma x\ge \frac{2(1-a)}{(N-s)^{1-a}}.
\end{align}
Let $t^*$ be the  index satisfying 
\begin{align}\label{eq:def_t*}
\sum_{i=t^*+1}^{N-s} i^{-a} \le \frac{1}{4\gamma x}\le\sum_{i=t^*}^{N-s} i^{-a} .
\end{align}
Note that we have 
\begin{align*}
\sum_{i=t^*+1}^{N-s} i^{-a}&\ge \int_{t^*+1}^{N-s} z^{-a} \mathrm d z= \frac{(N-s)^{1-a}-(t^*+1)^{1-a}}{1-a}\notag\\
\sum_{i=t^*}^{N-s} i^{-a}&\le \int_{t^*-1}^{N-s} z^{-a} \mathrm d z= \frac{(N-s)^{1-a}-(t^*-1)^{1-a}}{1-a}.
\end{align*}
Plugging the above inequality into \eqref{eq:def_t*} gives
\begin{align*}
\frac{(N-s)^{1-a}-(t^*+1)^{1-a}}{1-a}\le \frac{1}{4\gamma x}\le \frac{(N-s)^{1-a}-(t^*-1)^{1-a}}{1-a},
\end{align*}
which implies that
\begin{align*}
t^*\in\bigg[\bigg((N-s)^{1-a}-\frac{1-a}{4\gamma x}\bigg)^{\frac{1}{1-a}}-1, \bigg[(N-s)^{1-a}-\frac{1-a}{4\gamma x}\bigg]^{\frac{1}{1-a}}+1\bigg].
\end{align*}
Note that 
\begin{align*}
\bigg[(N-s)^{1-a}-\frac{1-a}{4\gamma x}\bigg]^{\frac{1}{1-a}} & = (N-s)\cdot \bigg[1 - \frac{1-a}{4\gamma x(N-s)^{1-a}}\bigg]^{\frac{1}{1-a}}    \notag\\
&\le (N-s) - \frac{(1-a)\cdot (N-s)^a}{4\gamma x},
\end{align*}
where the inequality follows from \eqref{eq:condition_gamma_x_2} and the fact that $1/(1-a)\ge 1$. Therefore, it holds that
\begin{align}\label{eq:upperbound_t*}
t^*\le (N-s) - \frac{(1-a)\cdot (N-s)^a}{4\gamma x}+1.
\end{align}
Therefore, applying the above inequality to \eqref{eq:lowerbound_f2_temp2} gives
\begin{align}\label{eq:f2_case1_case2}
f_2(x )&\ge  \sum_{t=1}^{N-s}\frac{\gamma^2 x^2}{t^{2a}}\cdot e^{-4\gamma x\cdot \sum_{i={t+1}}^{N-s}i^{-a}}.\notag\\
&\ge \sum_{t=t^*}^{N-s} \frac{\gamma^2 x^2}{t^{2a}}\cdot e^{-4\gamma x\cdot\sum_{i=t+1}^N i^{-a}}\notag\\
&\overset{(i)}{\ge} \sum_{t=t^*}^{N-s} \frac{\gamma^2x^2}{N^{2a}}\cdot e^{-4\gamma  x\cdot\sum_{i=t^*+1}^{N-s} i^{-a}}\notag\\
&\overset{(ii)}{\ge} (N-s-t^*+1)\cdot \frac{\gamma^2x^2}{(N-s)^{2a}}\cdot e^{-1}\notag\\
&\overset{(iii)}{\ge}\frac{(1-a)\cdot \gamma x}{4e\cdot N^{a}},
\end{align}
where the $(i)$ holds since $t\in[t^*, N-s]$, $(ii)$ follows from the fact that $\sum_{i=t^*+1}^{N-s}i^{-a}\le 1$, and $(iii)$ follows from \eqref{eq:upperbound_t*}.
Then combining \eqref{eq:f2_case1_case1} and \ref{eq:f2_case1_case2} and set $k^* := \max\{k:\gamma_0\lambda_k\ge (1-a)/(2(N-s)^{1-a})$, we can get
\begin{align}\label{eq:lb_case1_f2}
\sum_{i}f_2(\lambda_i) \ge \sum_{i\le k^*}\frac{(1-a)\cdot \gamma_0\lambda_i}{4e\cdot N^{a}} +  \sum_{i\ge k^*+1}\frac{[1+(N-s)^{1-2a}]\gamma^2\lambda_i^2}{2e^2}
\end{align}

\item \textbf{Case of $a=1$.}
Then considering the case of $a=1$, where it holds that
\begin{align*}
\sum_{i=t+1}^{N-s}i^{-a}\le \int_{t}^{N-s}z^{-a}\mathrm d z = \log(N-s)-\log(t).
\end{align*}
Then the following holds according to \eqref{eq:lowerbound_f2_temp2},
\begin{align*}
f_2(x)&\ge \sum_{t=1}^{N-s}\frac{\gamma^2 x^2}{t^2}\cdot e^{-4\gamma x\cdot \sum_{i={t+1}}^{N-s}i^{-a}}\notag\\
&\ge \sum_{t=1}^{N-s}\frac{\gamma^2 x^2}{t^2}\cdot e^{-4\gamma x\cdot \big[\log(N-s)-\log(t)\big]}\notag\\
&\ge \sum_{t=1}^{N-s}\frac{\gamma^2 x^2}{t^2}\cdot\bigg(\frac{t}{N-s}\bigg)^{4\gamma x}.
\end{align*}
Then note that for any $4\gamma x< 1$,
\begin{align*}
\sum_{t=1}^{N-s}\frac{\gamma^2x^2}{t^2}\cdot \bigg(\frac{t}{N-s}\bigg)^{4\gamma x} = \frac{\gamma^2x^2}{(N-s)^{4\gamma x}}\cdot\sum_{t=1}^{N-s}\frac{1}{t^{2-4\gamma x}}\ge  \frac{\gamma^2x^2}{(N-s)^{4\gamma x}},
\end{align*}
which implies that
\begin{align}\label{eq:lb_case2_f2}
\sum_{i}f_2(\lambda_i) \ge  \sum_{i}\frac{\gamma_0^2\lambda_i^2}{(N-s)^{4\gamma_0\lambda_i}}.
\end{align}

\end{itemize}

Now we can combine the derived lower bounds for $f_1(x)$ and $f_2(x)$ in \eqref{eq:lb_case1_f1}, \eqref{eq:lb_case2_f1} \eqref{eq:lb_case1_f2}, and \eqref{eq:lb_case2_f2}, and obtain
\begin{itemize}
    \item \textbf{Case of $0\le a<1$.} Let $k^* = \max\{k:\gamma_0 \lambda_k\ge (1-a)/[2(N-s)^{a-1}]\}$ and $k^\dagger = \max\{k:\gamma_0 \lambda_k\ge 1/(2s)\}$, we have
    \begin{align*}
    \langle\HB, \CB_N\rangle &\ge \sigma^2\cdot\sum_i[f_1(\lambda_i)+f_2(\lambda_i)]\notag\\
    &\ge\sigma^2 \cdot\bigg(\sum_{i\le k^*} \frac{(1-a)\cdot\gamma_0\lambda_i}{N^a}\vee \frac{(1-a)^2a\log(N)}{16 eN}+\sum_{k^*+1\le i\le k^\dagger} \frac{\gamma_0\lambda_i}{4e^2} + \sum_{i\ge k^\dagger+1}\frac{s\gamma_0^2\lambda_i^2}{2e^2}\bigg),
    \end{align*}
    where we use the fact that 
    \begin{align*}
    \frac{(1-a)\gamma_0\lambda_i}{N^a}\ge \frac{(1-a)^2a\log(N)}{16 eN}
    \end{align*}
    for all $i\le k'$ (please refer to \eqref{eq:lb_case1_f1} for the definition of $k'$).

    \item \textbf{Case of $a=1$.}  Let $k^* = \max\{k:\gamma_0 \lambda_k\ge 1/(2+2\log(N-s-1))\}$ and $k^\dagger = \max\{k:\gamma_0 \lambda_k\ge 1/(2s)\}$, we have
    \begin{align*}
    \langle\HB, \CB_N\rangle &\ge \sigma^2\cdot\sum_i[f_1(\lambda_i)+f_2(\lambda_i)]\notag\\
    &\ge \sigma^2\cdot\bigg(\sum_{i\le k^*}\frac{\gamma_0^2\lambda_i^2}{N^{4\gamma_0 \lambda_i}}+\sum_{k^*+1\le i\le k^\dagger} \frac{\gamma_0 \lambda_i}{4e^2} + \sum_{i\ge k^\dagger+1}\frac{s\gamma_0^2\lambda_i^2}{2e^2}\bigg)  .
    \end{align*}
\end{itemize}
\end{proof}

\subsection{Proof of the Lower Bound of Bias Error}
\begin{lemma}\label{lemma:lowerbound_bias_poly}
If applying polynomially decaying stepsize with $\gamma< 1/(4\lambda_1)$, then it holds that
\begin{itemize}
    \item \textbf{Case 1: $0\le a <1$.} Let $k^* = \max\{k:\gamma_0\lambda_k\ge (1-a)/[2(N-s)^{a-1}]\}$ and $k^\dagger = \max\{k:\gamma_0\lambda_k\ge 1/(2s)\}$, we have
    \begin{align*}
    \langle\HB, \BB_N\rangle &\ge \big\| (\IB-\gamma \HB)^{s+2N^{1-a}/(1-a)}\cdot (\wB_0-\wB^*)\big\|_{\HB}^2 \notag\\
    &\qquad +  e^{-4}\beta \langle\HB_{k^\dagger:\infty}, \BB_0\rangle \cdot\bigg(\sum_{i\le k^*} \frac{(1-a)\cdot\gamma_0\lambda_i}{N^a}+\sum_{k^*+1\le i\le k^\dagger} \frac{\gamma_0\lambda_i}{4e^2} + \sum_{i\ge k^\dagger+1}\frac{s\gamma_0^2\lambda_i^2}{2e^2}\bigg).
    \end{align*}
\item \textbf{Case 2: $a =1$.}  Let $k^* = \max\{k:\gamma_0\lambda_k\ge 1/(2+2\log(N-s-1))\}$ and $k^\dagger = \max\{k:\gamma_0\lambda_k\ge 1/(2s)\}$, we have  
 \begin{align*}
\langle\HB, \BB_N\rangle &\ge \big\| (\IB-\gamma \HB)^{s+2\log(N)}\cdot (\wB_0-\wB^*)\big\|_{\HB}^2 \notag\\
&\qquad +  e^{-4}\beta \langle\HB_{k^\dagger:\infty}, \BB_0\rangle \cdot\bigg(\sum_{i\le k^*} \frac{\gamma_0^2\lambda_i^2}{N^{4\gamma_0\lambda_i}}+\sum_{k^*+1\le i\le k^\dagger} \frac{\gamma_0\lambda_i}{4e^2} + \sum_{i\ge k^\dagger+1}\frac{s\gamma_0^2\lambda_i^2}{2e^2}\bigg).
\end{align*}
\end{itemize}
\end{lemma}
\begin{proof}
The proof of Lemma \ref{lemma:lowerbound_bias_poly} follows a similar idea of the proof of Thoerem \ref{thm:B-lower-bound}. In particular, by \eqref{eq:B-fine-lower-bound-expanded}, we have
\begin{equation}\label{eq:B-fine-lower-bound-expanded-poly}
    \BB_N \succeq \prod_{t=1}^n (\Ical - \gamma_t \widetilde{\Tcal}_t) \circ \BB_{0} + \beta \Big\langle {\prod_{t=1}^N (\IB- \gamma_t \HB)^2 \HB, \ \BB_0} \Big\rangle \cdot \sum_{t=1}^N \gamma_t^2 \prod_{i=t+1}^N (\Ical - \gamma_i \widetilde{\Tcal}_i)\circ \HB.
\end{equation}
Then we focus on the scalar function $g(x) = \prod_{t=1}^N (1-\gamma_t x)^2x$ for all $x\le 1/(2\gamma)$.  Specifically, using the inequality $(1-\gamma_tx)^2\ge e^{-4\gamma_tx}$, we have
\begin{align*}
g(x)\ge e^{-4x \cdot \sum_{t=1}^N\gamma_t}\ge e^{-8s\gamma x},
\end{align*}
where we use the assumption that $s\gamma\ge \sum_{t=s+1}^{s}\gamma_t$. Then let $k^\dagger:=\max\{k: \gamma_0\lambda_k\ge 1/(2s)\}$, we have
\begin{align*}
\prod_{t=1}^N(\IB-\gamma\HB)^2\HB\succeq e^{-4}\cdot\HB_{k^\dagger:\infty}.
\end{align*}
Plugging the above inequality into \eqref{eq:B-fine-lower-bound-expanded-poly} and multiplying by $\HB$ on both sides, we have
\begin{align*}
\langle\HB, \BB_N\rangle \succeq \Big\langle\HB,\prod_{t=1}^N (\Ical - \gamma_t \widetilde{\Tcal}_t) \circ \BB_{0}\Big\rangle + e^{-4}\beta \langle\HB_{k^\dagger:\infty}, \BB_0\rangle \cdot \underbrace{\Big\langle\HB,\sum_{t=1}^N \gamma_t^2 \prod_{i=t+1}^N (\Ical - \gamma_i \widetilde{\Tcal}_i)\circ \HB\Big\rangle}_{(*)}.
\end{align*}
Regarding $(*)$, we can define the function $f(x)$ as did \eqref{eq:def_fx}, it is clear that $(*) = \sum_if(\lambda_i)$ so that the results of Lemma \ref{lemma:lowerbound_bias_poly} can be directly applied. Moreover, note that
\begin{align*}
\prod_{t=1}^N(1-\gamma_t x)^2\ge (1-\gamma x)^{2s}\cdot e^{-4\gamma x\sum_{t=1}^{N-s}t^{-a}} \ge 
\begin{cases}
(1-\gamma x)^{2s}\cdot e^{-4\gamma x N^{1-a}/(1-a)},& 0\le a < 1\\
(1-\gamma x)^{2s}\cdot e^{-4\gamma x\log(N)},& a=1.
\end{cases}
\end{align*}
Then combining the above results and applying the fact that $\BB_0 = (\wB_0-\wB^*)(\wB_0-\wB^*)^\top$, we have
\begin{itemize}
    \item \textbf{Case 1: $0\le a <1$.} Let $k^* = \max\{k:\gamma_0\lambda_k\ge (1-a)/[2(N-s)^{a-1}]\}$ and $k^\dagger = \max\{k:\gamma_0\lambda_k\ge 1/(2s)\}$, we have
    \begin{align*}
    \langle\HB, \BB_N\rangle &\ge \big\| (\IB-\gamma \HB)^{s+2N^{1-a}/(1-a)}\cdot (\wB_0-\wB^*)\big\|_{\HB}^2 \notag\\
    &\qquad +  e^{-4}\beta \langle\HB_{k^\dagger:\infty}, \BB_0\rangle \cdot\bigg(\sum_{i\le k^*} \frac{(1-a)\cdot\gamma_0\lambda_i}{N^a}+\sum_{k^*+1\le i\le k^\dagger} \frac{\gamma_0\lambda_i}{4e^2} + \sum_{i\ge k^\dagger+1}\frac{s\gamma_0^2\lambda_i^2}{2e^2}\bigg).
    \end{align*}
\item \textbf{Case 2: $a =1$.}  Let $k^* = \max\{k:\gamma_0\lambda_k\ge 1/(2+2\log(N-s-1))\}$ and $k^\dagger = \max\{k:\gamma_0\lambda_k\ge 1/(2s)\}$, we have  
 \begin{align*}
\langle\HB, \BB_N\rangle &\ge \big\| (\IB-\gamma \HB)^{s+2\log(N)}\cdot (\wB_0-\wB^*)\big\|_{\HB}^2 \notag\\
&\qquad +  e^{-4}\beta \langle\HB_{k^\dagger:\infty}, \BB_0\rangle \cdot\bigg(\sum_{i\le k^*} \frac{\gamma_0^2\lambda_i^2}{N^{4\gamma_0\lambda_i}}+\sum_{k^*+1\le i\le k^\dagger} \frac{\gamma_0\lambda_i}{4e^2} + \sum_{i\ge k^\dagger+1}\frac{s\gamma_0^2\lambda_i^2}{2e^2}\bigg).
\end{align*}
\end{itemize}
This completes the proof.

 \end{proof}

\subsection{Proof of Theorem \ref{thm:tail_decay_poly}}
Here we state the full version of Theorem \ref{thm:tail_decay_poly} and provide its proof.

\begin{theorem}[A lower bound for poly-decaying stepsizes]\label{thm:tail_decay_poly_full_ver}
Consider last iterate SGD with stepsize scheme \eqref{eq:poly-tail-decay-lr}.
Suppose Assumptions \ref{assump:second-moment}, \ref{assump:fourth-moment}\ref{item:fourth-moement-lower} and \ref{assump:well-specified-noise} hold. Suppose $\gamma_0< 1/(4\lambda_1)$ and $s\gamma_0\ge \sum_{t=s+1}^N\gamma_t$, then for any constant $a\in[0, 1]$,
\[
\Ebb [ L(\wB_N) - L(\wB^*) ] = \half \biasErr + \half \varErr.
\]
Moreover:
\begin{itemize}[leftmargin=*]
\item If $ 0\le a <1$, then 
    \begin{align*}
    \biasErr \ge \big\| (\IB-\gamma_0 \HB)^{s+\frac{2N^{1-a}}{1-a}}\cdot (\wB_0-\wB^*)\big\|_{\HB}^2  +  \frac{(1-a)^{2}\beta}{e^4}\cdot \|\wB_0 - \wB^* \|^2_{\HB_{k^\dagger:\infty}} \cdot\frac{\DIM}{N},
    \end{align*}
    and
    \begin{align*}
      \varErr \ge  (1-a)^2 \sigma^2\cdot \frac{\DIM}{N}.  
    \end{align*}
    Here $k^* := \max\{k:\gamma_0 \lambda_k\ge (1-a)/( 2(N-s)^{1-a} )\}$, $k^\dagger := \max\{k:\gamma_0 \lambda_k\ge 1/(2s)\}$, and the \emph{effective dimension} is defined by
    \begin{align*}
    \DIM & := \sum_{i\le k^*} \max\{ N^{1-a}\gamma_0\lambda_i ,\ \frac{a\log(N)}{16 e} \}  + \sum_{k^*< i\le k^\dagger} \frac{N \gamma_0 \lambda_i}{4e^2} + \sum_{i> k^\dagger}\frac{s N \gamma_0^2\lambda_i^2}{2e^2}.
    \end{align*}
    
\item If $ a =1$, then 
 \begin{align*}
\biasErr &\ge \big\| (\IB-\gamma_0 \HB)^{s+2\log(N)}\cdot (\wB_0-\wB^*)\big\|_{\HB}^2 +  \frac{\beta}{e^4} \cdot \|\wB_0 - \wB^* \|^2_{\HB_{k^\dagger:\infty}}  \cdot\frac{\DIM}{N},
\end{align*}
and
\begin{align*}
\varErr &\ge \sigma^2\cdot \frac{\DIM}{N}.  
\end{align*}
Here $k^* := \max\{k:\gamma_0 \lambda_k\ge 1/(2+2\log(N-s-1))\}$ $k^\dagger := \max\{k:\gamma_0 \lambda_k\ge 1/(2s)\}$, and the \emph{effective dimension} is defined by
\begin{align*}
    \DIM := \sum_{i\le k^*} N^{1 - 4\gamma_0 \lambda_i} \gamma_0^2\lambda_i^2+\sum_{k^*< i\le k^\dagger} \frac{N \gamma_0 \lambda_i}{4e^2} + \sum_{i> k^\dagger}\frac{sN\gamma_0^2\lambda_i^2}{2e^2}.
\end{align*}
\end{itemize}

\end{theorem}
\begin{proof}
The proof is a simple combination of Lemmas \ref{lemma:bias-var-decomp-lowerbonud}, \ref{lemma:lowerbound_bias_poly}, and \ref{lemma:lowerbound_var_poly}.
\end{proof}

\subsection{Proof of Theorem \ref{thm:comparison}}
\begin{proof}
The proof will be focusing showing that the upper bound for geometrically decaying stepsize (Theorem \ref{thm:tail-decay-upper-bound}) is smaller than the lower bound for polynomially decaying stepsize (Theorem \ref{thm:tail_decay_poly})  up to some constant factors.

First, note that $\IB_{0:k^*}/\gamma_0K\preceq \HB_{0:k^*}$, if setting $k^* = \max\{k: \gamma_0\lambda_k\ge 1/K\}$, then the first term of the bias error upper bound in Theorem \ref{thm:tail_decay_poly} can be further relaxed as
\begin{align*}
&\frac{\big\| (\IB-\gamma_0 \HB)^{s+K} (\wB_0 - \wB^*) \big\|^2_{\IB_{0:k^*}} }{\gamma_0 K} +  \big\| (\IB-\gamma_0 \HB)^{s+K} (\wB_0 - \wB^*) \big\|^2_{\HB_{k^*:\infty}} \notag\\
&\le \big\| (\IB-\gamma_0 \HB)^{s+K} (\wB_0 - \wB^*) \big\|_{\HB_{0:k^*}}^2 + \big\| (\IB-\gamma_0 \HB)^{s+K} (\wB_0 - \wB^*) \big\|_{\HB_{k^*:\infty}}^2\notag\\
& = \|(\IB-\gamma\HB)^{s+K} (\wB-\wB^*)\|_{\HB}^2.
\end{align*}
Moreover, note that we have set $s=N/2$, it is clear that $s\gamma \ge\sum_{t=s+1}\gamma_t $ for polynomially decaying stepsize so that Theorem \ref{thm:tail_decay_poly} holds. Then it can be shown that the length of each phase in SGD with geometrically decaying stepsize is $K=(N-s)/\log(N-s)=\Theta(N/\log(N))=\omega(N^{1-a}\vee \log(N))$. Therefore, we have
\begin{align*}
\|(\IB-\gamma\HB)^{s+K} (\wB-\wB^*)\|_{\HB}^2\le\begin{cases}
\|(\IB-\gamma\HB)^{s+\frac{2N^{1-a}}{1-a}}(\wB-\wB^*)\|_{\HB}^2, &\text{for any constant } a\in[0, 1);\\
\|(\IB-\gamma\HB)^{s+2\log(N)}(\wB-\wB^*)\|_\HB^2, &a=1.
\end{cases}
\end{align*}

Note that the second term in the bias error bound has a quite similar form as the variance bound. Then we will consider the variance error and the results can be directly applied to the second term in the bias error bound. By looking at the upper bound in Theorems \ref{thm:tail-decay-upper-bound} and \ref{thm:tail_decay_poly}, we can compare the variance error along different dimensions separately.

\noindent\textbf{Case 1: $a\in[0, 1)$} 
For the case $a\in[0,1)$, we define $k^*_1 = \max\{k:\gamma_0\lambda_k\ge (1-a)/[2(N-s)^{1-a}]\} = \max\{k:\gamma_0\lambda_k\ge \Theta(1/N^{1-a})\}$, $k^*_2 = \max\{k:\gamma_0\lambda_k\ge 1/K\} =\max\{k:\gamma_0\lambda_k\ge \Theta(\log(N)/N)\} $,  and $k^\dagger = \max\{k:\gamma_0\lambda_k\ge \Theta(1/N)\}$, we have
\begin{align*}
\varErr_{\mathrm{exp}}\lesssim \sigma^2\cdot\bigg(\sum_{i\le k_1^*}\frac{\log(N)}{N} + \sum_{k_1^*< i\le k_2^*}\frac{\log(N)}{N} + \sum_{k_2^*< i\le k^\dagger}\gamma_0\lambda_i + \sum_{i> k^\dagger}N\gamma_0^2\lambda_i^2\bigg) 
\end{align*}
and 
\begin{align*}
\varErr_{\mathrm{poly}}\gtrsim \sigma^2\cdot\bigg(\sum_{i\le k_1^*}\frac{\gamma_0\lambda_i}{N^a}\vee \frac{\log(N)}{N} + \sum_{k_1^*< i\le k_2^*}\gamma_0\lambda_i + \sum_{k_2^*< i\le k^\dagger}\gamma_0\lambda_i + \sum_{i> k^\dagger}N\gamma_0^2\lambda_i^2\bigg). 
\end{align*}
Then it suffices to consider the  case $k_1^*<i\le k_2^*$. In particular, according to the definition of $k_1^*$ and $k_2^*$, it is clear that for any $k_1^*<i\le k_2^*$,
\begin{align*}
\gamma_0\lambda_i\ge \frac{1}{K} =\Theta\big(\log(N)/N\big).
\end{align*}
This implies that $\varErr_{\mathrm{exp}}\lesssim \varErr_{\mathrm{poly}}$. Then we can go back to the second term of the bias error bounds, which have a similar formula of the variance error bound. Applying the definition $R(N) = (\|\wB-\wB^*\|_{\IB_{0:k^\dagger}}^2/(\gamma_0 N)+\|\wB-\wB^*\|_{\HB_{k^\dagger:\infty}})/\sigma^2$, we can conclude that 
\[
\Ebb [ L(\wB_N^{\mathrm{geo}}) - L(\wB^*) ] \le C\cdot [1 + \log(N)\cdot R(N)]\cdot \Ebb [ L(\wB_N^{\mathrm{poly}}) - L(\wB^*) ]
\]

\noindent\textbf{Case 2: $a=1$.} Similarly, we can now define $k_1^*:=\max\{k:\gamma_0\lambda_k\ge 1/(2+2\log(N-s-1))\}=\max\{k:\gamma_0\lambda_k\ge\Theta(1/\log(N))\}$ and get
\begin{align*}
\varErr_{\mathrm{poly}}\gtrsim \sigma^2\cdot\bigg(\sum_{i\le k_1^*} \frac{\gamma_0^2\lambda_i^2}{N^{4\gamma_0 \lambda_i}}+\sum_{k_1^*< i\le k_2^*} \gamma_0 \lambda_i+\sum_{k_2^*< i\le k^\dagger} \gamma_0 \lambda_i + \sum_{i\ge k^\dagger+1}N\gamma_0^2\lambda_i^2\bigg).
\end{align*}
Note that we have assumed $4\gamma_0\lambda_i<1$, then we have for all $i\le k_1^*$
\begin{align*}
\frac{\gamma_0^2\lambda_i^2}{N^{4\gamma_0\lambda_i}}\ge N^{-4\gamma_0\lambda_i}/\log^2(N)=\Omega(\log(N)/N).
\end{align*}
Additionally, for any $k_1^*<i\le k_2^*$, we also have
\begin{align*}
\gamma_0\lambda_i\ge \frac{1}{K} =\Theta\big(\log(N)/N\big).
\end{align*}
Then combining the bias error and variance bounds, we can also conclude that 
\[
\Ebb [ L(\wB_N^{\mathrm{geo}}) - L(\wB^*) ] \le C\cdot[1 + \log(N)\cdot R(N)]\cdot \Ebb [ L(\wB_N^{\mathrm{poly}}) - L(\wB^*) ]
\]
where $R(N) = (\|\wB-\wB^*\|_{\IB_{0:k^\dagger}}^2/(\gamma_0 N)+\|\wB-\wB^*\|_{\HB_{k^\dagger:\infty}})/\sigma^2$. This completes the proof.
\end{proof}

\end{document}